%% file: main.tex
\documentclass[12pt,a4paper]{article}
\input{preamble}

\let\tilde\widetilde
\let\epsilon\varepsilon

\long\def\black#1{{\color{black}#1}}

\begin{document}
\maketitle
%\tableofcontents
%%%%%%%%%%%%%%%%%%%%%%%%%%%%%%%%%%%%%%%%%%%%%%
\begin{abstract}
This paper concentrates on the approximation power of deep feed-forward neural networks in terms of width and depth. It is proved by construction that ReLU networks with width $\mathcal{O}\big(\max\{d\lfloor N^{1/d}\rfloor,\, N+2\}\big)$
and depth $\mathcal{O}(L)$  can approximate a H\"older continuous function on $[0,1]^d$ with an  approximation rate $\mathcal{O}\big(\lambda\sqrt{d} (N^2L^2\ln N)^{-\alpha/d}\big)$, where $\alpha\in (0,1]$ and $\lambda>0$ are H\"older order and constant, respectively. Such a rate is optimal up to a constant in terms of width and depth separately, while existing results are only nearly optimal without the logarithmic factor in the approximation rate. %All prefactors inside $\mathcal{O}(\cdot)$ have explicit formulas. (Maybe this is not important to mention in the abstract, we say it in the content.)
More generally,  for an arbitrary continuous function $f$ on $[0,1]^d$, the approximation rate becomes $\mathcal{O}\big(\,\sqrt{d}\,\omega_f\big( (N^2L^2\ln N)^{-1/d}\big)\,\big)$, where $\omega_f(\cdot)$ is the modulus of continuity. 
We also extend our analysis to any continuous function $f$ on a bounded set. Particularly, if ReLU networks with depth $31$ and width $\mathcal{O}(N)$ are used to approximate one-dimensional Lipschitz continuous functions  on $[0,1]$ with a Lipschitz constant $\lambda>0$, the approximation rate in terms of the total number of parameters, $W=\mathcal{O}(N^2)$, becomes $\mathcal{O}(\tfrac{\lambda}{W\ln W})$, which has not been discovered in the literature for fixed-depth ReLU networks.
\end{abstract}

\section{Introduction}
\label{sec:intro}
Over the past few decades,
the expressiveness of neural networks has been widely studied from many points of view, e.g., in terms of combinatorics \cite{NIPS2014_5422}, topology \cite{6697897}, Vapnik-Chervonenkis (VC) dimension \cite{Bartlett98almostlinear,Sakurai,pmlr-v65-harvey17a}, fat-shattering dimension \cite{Kearns,Anthony:2009}, information theory \cite{PETERSEN2018296}, classical approximation theory \cite{Cybenko1989ApproximationBS,HORNIK1989359,barron1993,yarotsky18a,shijun1,shijun:thesis,shijun2,shijun3,shijun4,shijun1,shijun5,2019arXiv191210382L,SIEGEL2020313,e2020banach}, optimization \cite{NIPS2016_6112,DBLP:journals/corr/NguyenH17,opt,xu2020,JMLR:v20:17-526}. 
The error analysis of neural networks consists of three parts: the approximation error, the optimization error, and the generalization error. This paper focuses on the approximation error for ReLU networks. 

The approximation errors of feed-forward neural networks with various activation functions  have been studied for different types of functions, e.g., smooth functions \cite{NIPS2017_7203,DBLP:journals/corr/LiangS16,yarotsky2017,DBLP:journals/corr/abs-1807-00297,shijun3}, piecewise smooth functions \cite{PETERSEN2018296}, band-limited functions \cite{bandlimit}, continuous functions \cite{yarotsky18a,shijun2,shijun4,shijun5}. %However, to the best of our knowledge, most existing  theories \cite{NIPS2017_7203,bandlimit,yarotsky2017,DBLP:journals/corr/LiangS16,Hadrien,suzuki2018adaptivity,PETERSEN2018296,yarotsky18a,DBLP:journals/corr/abs-1807-00297,Daubechies2019} can only provide approximation errors with implicit formulas in the sense that the prefactor is characterized by a big O notation. 
%For example,  an approximation rate $\calO\big(\omega_f(L^{-2/d})\big)$ is estimated in \cite{yarotsky18a} via a narrow ReLU network with $L$ layers for sufficiently large $L$, where $\omega_f(\cdot)$ is the modulus of continuity. 
\black{
In the early works of approximation theory for neural networks, the universal approximation theorem  \cite{Cybenko1989ApproximationBS,HORNIK1991251,HORNIK1989359} without approximation rates showed that there exists a sufficiently large neural network approximating a target function in a certain function space within any given error $\varepsilon>0$. In particular, it is shown in \cite{NEURIPS2018_03bfc1d4} that the ReLU-activated residual neural network with one-neuron hidden layers is  a universal approximator. The universal approximation property  for general residual neural networks was proved in \cite{2019arXiv191210382L} via a dynamical system  approach.

An asymptotic analysis of the approximation rate in terms of depth is provided in \cite{yarotsky18a,2019arXiv190609477Y}
for ReLU networks. To be exact, the nearly optimal approximation rates of ReLU networks with width $\calO(d)$ and depth $\calO(L)$ for functions in $C([0,1]^d)$ and the unit ball of $C^s([0,1]^d)$  are $\calO(\omega_f(L^{-2/d}))$ and $\calO((L/\ln L)^{-2s/d})$, respectively.
These two papers  provide the  approximation rate in terms of depth asymptotically for fixed-width networks.  A different approach is used in  \cite{shijun2,shijun3} to obtain a quantitative characterization of the approximation rate in terms of width, depth, and smoothness order  for continuous and smooth functions.
}%%%%%%%%%%%%%%%%%%%  blue

Particularly, it was shown in \cite{shijun2} that a ReLU network with width $C_1(d)\cdot N$ and depth $C_2(d)\cdot  L$ can attain an approximation error $C_3(d)\cdot\omega_f(N^{-2/d}L^{-2/d})$ to approximate a continuous function $f$ on $[0,1]^d$, where $C_1(d)$, $C_2(d)$, and $C_3(d)$ are three constants in $d$ with explicit formulas to specify their values, and $\omega_f(\cdot)$ is the modulus of continuity of $f\in C([0,1]^d)$ defined via
\begin{equation*}
\omega_f(r)\coloneqq \sup\big\{|f(\bm{x})-f(\bm{y})|:\bm{x},\bm{y}\in [0,1]^d,\ \|\bm{x}-\bm{y}\|_2\le r\big\},\quad \tn{for any $r\ge 0$}.
\end{equation*}
Such an approximation error is optimal in terms of $N$ and $L$ up to a logarithmic term
and the corresponding optimal approximation theory is still unavailable. To address this problem, we provide a constructive proof in this paper to show that ReLU networks of width $\calO(N)$ and depth $\calO(L)$ can approximate an arbitrary continuous function $f$ on $[0,1]^d$ with an optimal approximation error $\calO\left(\sqrt{d}\,\omega_f\big((N^2L^2\ln N)^{-1/d}\big)\right)$ in terms of $N$ and $L$. 
%Such a rate is  based on new analysis techniques merely based on the structure of networks and a modified bit extraction technique inspired by  \cite{Bartlett98almostlinear}, instead of designing networks to approximate traditional approximation basis like polynomials and splines as in the existing literature \cite{pmlr-v70-safran17a,DBLP:journals/corr/LiangS16,DBLP:journals/corr/RolnickT17,yarotsky2017,yarotsky18a,Hadrien,Boris,NIPS2017_7203,PETERSEN2018296,Johannes,suzuki2018adaptivity,MO}. 
As shown by our main result, Theorem~\ref{thm:main} below, the approximation rate obtained here admits explicit formulas to specify its prefactors when $\omega_f(\cdot)$ is known. 

\begin{theorem}
	\label{thm:main}
	Given a continuous function $f\in C([0,1]^d)$, for any $N\in \N^+$, $L\in \N^+$, and $p\in [1,\infty]$, 
	there exists a function $\phi$ implemented by a ReLU network with width $C_1\max\big\{d\lfloor N^{1/d}\rfloor,\, N+2\big\}$
	and depth $11L+C_2$
	such that 
	\begin{equation*}
	\|f-\phi\|_{L^p([0,1]^d)}\le 131\sqrt{d}\,\omega_f\Big(\big(N^2L^2 \log_3(N+2)\big)^{-1/d}\Big),
	\end{equation*}
	where $C_1=16$ and $C_2=18$ if $p\in [1,\infty)$;  $C_1=3^{d+3}$ and $C_2=18+2d$ if $p=\infty$.
\end{theorem}

Note that $3^{d+3}\max\big\{d\lfloor N^{1/d}\rfloor,\, N+2\big\}\le 3^{d+3}\max\big\{dN,\, 3N\big\}\le 3^{d+4}dN$. Given any $\tildeN,\tildeL\in\N^+$ with $\tildeN\ge 3^{d+4}d$ and $\tildeL\ge 29+2d$, there exist $N,L\in\N^+$ such that
\begin{equation*}
    3^{d+4}dN\le \tildeN< 3^{d+4}d(N+1) \tn{\quad  and\quad} 
    11L+18+2d\le \tildeL< 11(L+1)+18+2d.
\end{equation*}
It follows that
\begin{equation*}
    N\ge \frac{N+1}{3} > \frac{\tildeN}{3^{d+5}d}\tn{\quad  and\quad}  
    L\ge \frac{L+1}{2} > \frac12 \cdot \frac{\tildeL-18-2d}{11}=\frac{\tildeL-18-2d}{22}.
\end{equation*}
Then we have an immediate corollary of Theorem~\ref{thm:main}.

\begin{corollary}
	\label{coro:tildeNL}
	Given a continuous function $f\in C([0,1]^d)$, for any $\tildeN\in \N^+$ and $\tildeL\in\N^+$ with $\tildeN\ge 3^{d+4}d$ and $\tildeL\ge 29+2d$, 
	there exists a function $\phi$ implemented by a ReLU network with width $\tildeN$
	and depth $\tildeL$
	such that 
	\begin{equation*}
	\|f-\phi\|_{L^\infty([0,1]^d)}\le 131\sqrt{d}\,\omega_f\bigg(
	\Big( (\tfrac{\tildeN}{3^{d+5}d} )^2
	    (\tfrac{\tildeL-18-2d}{22})^2 \log_3(\tfrac{\tildeN}{3^{d+5}d}+2)    
	    \Big)^{-1/d}\bigg).
	\end{equation*}
\end{corollary}

As a special case of Theorem~\ref{thm:main} for explicit error characterization, let us take H\"older continuous functions as an example.
Let $\holder{\lambda}{\alpha}$ denote the space of  H\"older continuous functions on $[0,1]^d$ of order $\alpha\in (0,1]$ with  a H\"older constant $\lambda>0$. We have an immediate corollary of Theorem~\ref{thm:main} as follows.
\begin{corollary}
	\label{coro:main}
	Given a H\"older continuous function $f\in\holder{\lambda}{\alpha}$, for any $N\in \N^+$, $L\in \N^+$, and $p\in [1,\infty]$, 
	there exists a function $\phi$ implemented by a ReLU network with width $C_1\max\big\{d\lfloor N^{1/d}\rfloor,\, N+2\big\}$
	and depth $11L+C_2$
	such that 
	\begin{equation*}
		\|f-\phi\|_{L^p([0,1]^d)}\le 131\lambda \sqrt{d} \big(N^2L^2 \log_3(N+2)\big)^{-\alpha/d},
	\end{equation*}
	where $C_1=16$ and $C_2=18$ if $p\in [1,\infty)$;  $C_1=3^{d+3}$ and $C_2=18+2d$ if $p=\infty$.
\end{corollary}
%An immediate question following the constructive approximation is how much we can improve the approximation rate.
%In fact, the approximation rate  for  $\holder{\lambda}{\alpha}$ is optimal as we shall see later in Section~\ref{sec:optimality}. 

\vspace{8pt}
To better illustrate the importance of our theory, we summarize our key contributions as follows.
\begin{enumerate}[(1)]
	\item Upper bound: We provide a quantitative and non-asymptotic approximation rate $131\sqrt{d}\,\omega_f\Big(\big(N^2L^2 \log_3(N+2)\big)^{-1/d}\Big)$ in terms of width $\calO(N)$ and depth $\calO(L)$ for any $f\in  C([0,1]^d)$ in Theorem~\ref{thm:main}. 
	\begin{enumerate}[(\theenumi.1)]
		\item  This approximation error analysis can be extended to $f\in C(E)$ for any $E\subseteq [-R,R]^d$ with $R>0$ as we shall see later in Theorem~\ref{thm:main:E}.
		\item In the case of one-dimensional Lipschitz continuous functions on $[0,1]$ with a Lipschitz constant $\lambda>0$, the approximation rate in Theorem~\ref{thm:main} becomes  $\mathcal{O}(\tfrac{\lambda}{W\ln W})$
		for ReLU networks with $31$ hidden layers and $\mathcal{O}(W)$ parameters via setting $L=1$ and $W=\calO(N^2)$ therein. To the best of our knowledge, the approximation rate $\mathcal{O}(\tfrac{\lambda}{W\ln W})$ is better than existing known results using fixed-depth ReLU networks to approximate Lipschitz continuous functions on $[0,1]$.
	%	\item The approximation rate characterized by the width and depth in this paper is more generic and useful than the one in terms of the number of nonzero parameters denoted as $W$ in the literature. First, our rate in terms of the width and depth works for arbitrary width and depth up to constants, while the one in terms of $W$ is only valid for special network architectures since there are too many networks of different architectures with the same number of parameters.
	%	 Second, our rate in terms of width and depth can deduce a rate in terms of $W$ (e.g., set $N=\calO(1)$ and $W=\calO(L)$),  while it is not true the other way around. 
	\end{enumerate}
	
	\item Lower bound: Through the VC-dimension bounds of ReLU networks given in \cite{pmlr-v65-harvey17a}, we show, in Section~\ref{sec:optimality}, that the approximation rate $131\lambda \sqrt{d} \big(N^2L^2 \log_3(N+2)\big)^{-\alpha/d}$ in terms of width $\calO(N)$ and depth $\calO(L)$ for $\holder{\lambda}{\alpha}$ is optimal as follows. 
		\begin{enumerate}[(\theenumi.1)]
		\item When the width is fixed, both the approximation upper and lower bounds take the form of $CL^{-2\alpha/d}$ for a positive constant $C$.
		\item When the depth is fixed, both the approximation upper and lower bounds take the form of $C (N^2\ln N)^{-\alpha/d}$ for a positive constant $C$.
	\end{enumerate}
\end{enumerate}

\begin{figure}[!htp]        
     \centering
      \includegraphics[width=0.78\textwidth]{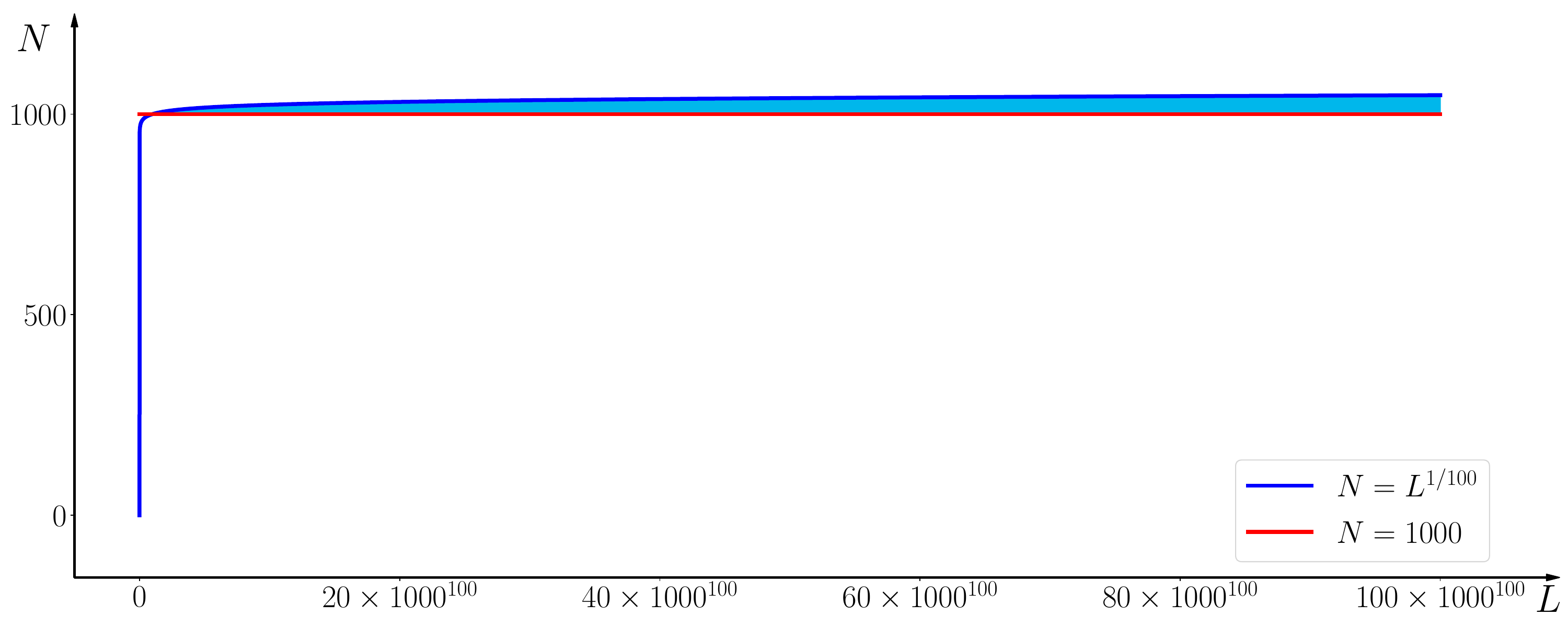}
     \caption{Our rate is optimal in terms of width $\calO(N)$ and depth $\calO(L)$  simultaneously except for the region marked in cyan characterized by $\{(N,L)\in \N^2: C_1\le N\le L^{C_2}\}$, where $C_i=C_i(\alpha,d)$ for $i=1,2$ are two positive constants. This figure is an example for  $C_1=1000$ and $C_2=1/100$. }
     \label{fig:opt:NLplane}
\end{figure}
% In the approximation theory, it is mathematically beautiful and technical to match the lower  and upper bounds of the approximation error. Generally, it is relatively simple to prove a  \textbf{nearly} optimal approximation error, which has a gap between the lower and upper bounds, e.g., a gap of a logarithmic term as shown in our previous paper \cite{shijun2}. 
% We bridge the gap in this paper to provide a better understanding for the approximation of ReLU networks.

\black{We would like to point out that if  $N$ and $L$ vary simultaneously,
% width $\calO(N)$ and depth $\calO(L)$ are both changeable,
the rate is optimal in the $N$-$L$ plane except for a small region as shown in Figure~\ref{fig:opt:NLplane}. See Section~\ref{sec:optimality} for a detailed discussion. The earlier result in \cite{shijun2} provides a {nearly} optimal approximation error that has a gap (a logarithmic term) between the lower and upper bounds.  It is technically challenging to match  the upper bound  with the  lower bound. Compared to the nearly optimal rate $19\lambda \sqrt{d}N^{-2\alpha/d}L^{-2\alpha/d}$ for H\"older continuous functions in $\holder{\lambda}{\alpha}$ in \cite{shijun2}, this paper achieves the optimal rate $131\lambda \sqrt{d} \big(N^2L^2 \log_3(N+2)\big)^{-\alpha/d}$ using  more technical and sophisticated construction. For example, a novel bit extraction technique different to that in \cite{Bartlett98almostlinear} is proposed, and new  ReLU networks are constructed to approximate step functions more efficiently than those in \cite{shijun2}. The optimal result obtained in this paper could also be extended to other functions spaces, leading to better understanding of deep network approximation. } 

We have obtained the optimal approximation rate for (H\"older) continuous functions approximated by ReLU networks. There are two possible directions to improve the approximation rate or reduce the effect of the curse of dimensionality. The first one is to consider proper target function spaces, e.g., Barron
spaces \cite{barron1993,Weinan2019,e2020representation,siegel2021optimal},  band-limited functions \cite{doi:10.1002/mma.5575,bandlimit},  smooth functions \cite{2019arXiv190609477Y,shijun3}, and analytic functions \cite{DBLP:journals/corr/abs-1807-00297}.
The other direction is to consider neural networks with other activation functions.
For example, the results of  \cite{2019arXiv190609477Y} imply that $(\sin,\tn{ReLU})$-activated networks with $W$ parameters  can achieve an asymptotic approximation error $\calO(2^{-c_d\sqrt{W}})$ for Lipschitz continuous functions defined on $[0,1]^d$, where $c_d$ is an unknown constant depending on $d$.
 Floor-ReLU networks with width $\calO(N)$ and depth $\calO(L)$ are constructed in \cite{shijun4} to admit an approximation rate $\omega_f(\sqrt{d}N^{-\sqrt{L}})+2\omega_f(\sqrt{d})N^{-\sqrt{L}}$ for any continuous function $f\in C([0,1]^d)$.
It is shown in \cite{shijun5} that three-hidden-layer networks with $\calO(W)$ parameters using the floor function ($\lfloor x \rfloor)$, the exponential function ($2^x$), and the step function ($\one_{x\ge 0}$) as activation functions can approximate Lipschitz functions defined on $[0,1]^d$ with an exponentially
small error $\calO(\sqrt{d}2^{-W})$. 
% In another example,
% The $(\sin,\tn{ReLU})$-activated neural networks with $W$ parameters are constructed in \cite{2019arXiv190609477Y} to approximate Lipschitz functions   within an  approximation error $\calO(e^{-c_d\sqrt{W}})$  for sufficiently large $W$, where $c_d$ is a constant depending on $d$. 
% These three papers \cite{shijun4,shijun5,2019arXiv190609477Y} provide roughly exponential approximation rates via introducing new activation functions. 
% Replacing $2^x$ and $\one_{x\ge 0}$ in \cite{shijun5} 
% by a non-polynomial analytic function,
% a recent paper \cite{Yarotsky2021ElementarySA}
% further improves the approximation rate of \cite{shijun5}. 
By the use of more sophisticated activation functions instead of those used in \cite{shijun4,shijun5,2019arXiv190609477Y}, a recent paper \cite{Yarotsky2021ElementarySA} shows that 
there exists a network of size depending on $d$ implicitly, achieving an arbitrary approximation error for any continuous function in $C([0,1]^d)$.
A key ingredient  of the approaches mentioned above is to use  more than one  activation functions to design neural  network architectures.  
% Precisely, an arbitrary error is attained for any continuous function in $C([0,1]^d)$ approximated by a network, with its size only depending on $d$ (not  specify the dependence on $d$), using one of the following activation function sets: 1)
% $\sin$ and $\arcsin$ (replacing ReLU by $\arcsin$) 
% 2) the floor function  ($\lfloor \cdot\rfloor$) and a non-polynomial analytic function (replacing ReLU or $2^x$\&$\one_{x\ge 0}$  by a non-polynomial analytic function).

 The  error analysis of deep learning  is to estimate  approximation, generalization,  and  optimization errors. Here, we give a brief discussion, the interested reader can find more details in \cite{shijun3,shijun4}.
Let $\phi(\bm{x};\bm{\theta})$ denote a function computed by  a  network parameterized with $\bm{\theta}$. 
Given a target function $f$, the final goal is to find the expected risk minimizer
\begin{equation*}
	\bm{\theta}_{\mathcal{D}}\coloneqq \argmin_{\bm{\theta}} R_{\mathcal{D}}(\bm{\theta}),\quad \tn{where}\   R_{\mathcal{D}}(\bm{\theta})\coloneqq \mathbb{E}_{\bm{x}\sim U(\mathcal{X})} \left[\ell( \phi(\bm{x};\bm{\theta}),f(\bm{x}))\right],
\end{equation*}
with a loss function $\ell(\cdot,\cdot)$ and an unknown data {distribution} $U(\mathcal{X})$.

In practice, for given {samples} $\{( \bm{x}_i,f(\bm{x}_i))\}_{i=1}^n$,  the goal of supervised learning
is to identify  the empirical risk minimizer
\begin{equation*}\label{eqn:emloss}
	\bm{\theta}_{\mathcal{S}}\coloneqq\argmin_{\bm{\theta}}R_{\mathcal{S}}(\bm{\theta}),\quad \tn{where}\ R_{\mathcal{S}}(\bm{\theta}):=
	\frac{1}{n}\sum_{i=1}^n \ell\big( \phi(\bm{x}_i;\bm{\theta}),f(\bm{x}_i)\big).
\end{equation*}
In fact, one could only get a numerical minimizer $\bm{\theta}_{\mathcal{N}}$ via  a numerical optimization method. The discrepancy between the target function $f$ and the learned function $\phi(\bm{x};\bm{\theta}_{\mathcal{N}})$ is measured by $R_{\mathcal{D}}(\bm{\theta}_{\mathcal{N}}) $, which is bounded by
\begin{equation*}	
		\mathResize[0.86]{
		R_{\mathcal{D}}(\bm{\theta}_{\mathcal{N}})    
		\le \underbrace{R_{\mathcal{D}}(\bm{\theta}_{\mathcal{D}})}_{\tn{\color{blue}Approximation error}} \ +\  \underbrace{[R_{\mathcal{S}}(\bm{\theta}_{\mathcal{N}})-R_{\mathcal{S}}(\bm{\theta}_{\mathcal{S}})]}_{\tn{\color{blue}Optimization error}}\  + \  \underbrace{[R_{\mathcal{D}}(\bm{\theta}_{\mathcal{N}})-R_{\mathcal{S}}(\bm{\theta}_{\mathcal{N}})]
			+[R_{\mathcal{S}}(\bm{\theta}_{\mathcal{D}})-R_{\mathcal{D}}(\bm{\theta}_{\mathcal{D}})]}_{\tn{\color{blue}Generalization error }}.  
			}%%%%%%%%%%%%%%% mathResize
\end{equation*}
 This paper deals with the approximation error of ReLU networks for continuous functions and  gives an upper bound of $R_{\mathcal{D}}(\bm{\theta}_{\mathcal{D}})$ which is optimal up to a constant.  %This leads to the exact estimate  of the approximation error  $R_{\mathcal{D}}(\bm{\theta}_{\mathcal{D}})$ up to a constant for continuous functions.  
 Note that the approximation error analysis given here is  independent of  data samples and  deep learning algorithms.   However,  the analysis of optimization and generalization errors  do depend on data samples, deep learning algorithms, models, etc. 
 \black{ For example, refer to \cite{neyshabur2018the,Weinan2019,Weinan2019APE,2020arXiv200913500E,NIPS2016_6112,DBLP:journals/corr/NguyenH17,opt,xu2020,JMLR:v20:17-526} for a further understanding of the generalization and optimization errors.
}%%%%%%%%%%% blue

The rest of this paper is organized as follows. In Section~\ref{sec:analysis},  we prove Theorem~\ref{thm:main} by assuming Theorem~\ref{thm:mainGap} is true, show the optimality of Theorem~\ref{thm:main}, and extend our analysis to continuous functions defined on any bounded set. Next, Theorem~\ref{thm:mainGap} is proved in Section~\ref{sec:proof:mainGap} based on Propositions~\ref{prop:stepFunc} and \ref{prop:pointFitting}, the proofs of which can be found in Section~\ref{sec:proof:props}.
 Finally, Section~\ref{sec:conclusion} concludes this paper with a short discussion.

\section{Theoretical analysis}
\label{sec:analysis}
In this section, we first prove Theorem~\ref{thm:main} and discuss its optimality. Next, we extend our analysis to general continuous functions defined on any bounded set. Notations throughout this paper are summarized in Section~\ref{sec:notation}. 

\subsection{Notations}
\label{sec:notation}

    Let us summarize all basic notations used in this paper as follows.
\begin{itemize}
   \item Let $\R$, $\Q$, and $\Z$ denote the set of real numbers, rational numbers, and integers, respectively.
    
    \item Let $\N$ and $\N^+$ denote the set of natural numbers and positive natural numbers, respectively.  That is,
    $\N^+=\{1,2,3,\cdots\}$ and $\N=\N^+\bigcup\{0\}$.

     \item Matrices are denoted by bold uppercase letters. For instance,  $\bm{A}\in\mathbb{R}^{m\times n}$ is a real matrix of size $m\times n$, and $\bm{A}^T$ denotes the transpose of $\bm{A}$.  %Correspondingly, $\bm{A}(i,j)$ is the $(i,j)$-th entry of $\bm{A}$; $\bm{A}(:,j)$ is the $j$-th column of $\bm{A}$; $\bm{A}(i,:)$ is the $i$-th row of $\bm{A}$. 
     Vectors are denoted as bold lowercase letters. For example, $\bm{v}=[v_1,\cdots,v_d]^T=\left[\def\arraystretch{0.748}\begin{array}{c}
          v_1  \\
          \vdots \\
          v_d
     \end{array}\right]\in \R^d$ is a column vector %consisting of numbers $\{v_i\}_i$
     with $\bm{v}(i)=v_i$ being the $i$-th element. Besides, ``['' and ``]''  are used to  partition matrices (vectors) into blocks, e.g., $\bmA=\left[\begin{smallmatrix}\bmA_{11}&\bmA_{12}\\ \bmA_{21}&\bmA_{22}\end{smallmatrix}\right]$.
     
      \item For any $p\in [1,\infty)$, the $p$-norm (or $\ell^p$-norm) of a vector $\bmx=[x_1,x_2,\cdots,x_d]^T\in\R^d$ is defined by 
    \begin{equation*}
        \|\bmx\|_p\coloneqq \big(|x_1|^p+|x_2|^p+\cdots+|x_d|^p\big)^{1/p}.
    \end{equation*}

     %\item Assume $\mathcal{N}$ is a subset of $\R$, then $\max \mathcal{N}$, $\min\mathcal{N}$, $\sup \mathcal{N}$, and $\inf \mathcal{N}$ mean the maximum, the minimum, the supremum, and the infimum of $\mathcal{N}$, respectively.
     
     \item For any $x\in \R$, let $\lfloor x\rfloor:=\max \{n: n\le x,\ n\in \Z\}$ and $\lceil x\rceil:=\min \{n: n\ge x,\ n\in \Z\}$.
     
     \item Assume $\bm{n}\in \N^d$, then $f(\bm{n})=\mathcal{O}(g(\bm{n}))$ means that there exists positive $C$ independent of $\bm{n}$, $f$, and $g$ such that $ f(\bm{n})\le Cg(\bm{n})$ when all entries of $\bm{n}$ go to $+\infty$.

     \item For any $\theta\in[0,1)$, suppose its binary representation is $\theta=\sum_{\ell=1}^{\infty}\theta_\ell2^{-\ell}$ with $\theta_\ell\in \{0,1\}$, we introduce a special notation $\bin 0.\theta_1\theta_2\cdots \theta_L$ to denote the $L$-term binary representation of $\theta$, i.e., $\bin 0.\theta_1\theta_2\cdots \theta_L\coloneqq\sum_{\ell=1}^{L}\theta_\ell2^{-\ell}$. 

     \item  Let $\mu(\cdot)$ denote the Lebesgue measure. 
     
     \item Let $\one_{S}$ be the characteristic function on a set $S$, i.e., $\one_{S}$ is equal to $1$ on $S$ and $0$ outside $S$.
     
     \item Let $|S|$ denote the size of a set $S$, i.e., the number of all elements in $S$.
     
     \item The set difference of two sets $A$ and $B$ is denoted by $A\backslash B:=\{x:x\in A,\ x\notin B\}$. 
     
     \item 
     Given any $K\in \N^+$ and $\delta\in (0, \tfrac{1}{K})$, define a trifling region  $\Omega([0,1]^d,K,\delta)$ of $[0,1]^d$ as 
     \begin{equation}
     \label{eq:triflingRegionDef}
     \Omega([0,1]^d,K,\delta)\coloneqq\bigcup_{j=1}^{d} \bigg\{\bmx=[x_1,x_2,\cdots,x_d]^T\black{\in [0,1]^d}: x_j\in \bigcup_{k=1}^{K-1}(\tfrac{k}{K}-\delta,\tfrac{k}{K})\bigg\}.
     \end{equation}
     In particular, $\Omega([0,1]^d,K,\delta)=\emptyset$ if $K=1$. See Figure~\ref{fig:region} for two examples of trifling regions.
     
     \begin{figure}[!htp]        
     	\centering
     \begin{minipage}{0.805\textwidth}
     	\centering
     	\begin{subfigure}[b]{0.33\textwidth}
     		\centering            
     		\includegraphics[width=0.999\textwidth]{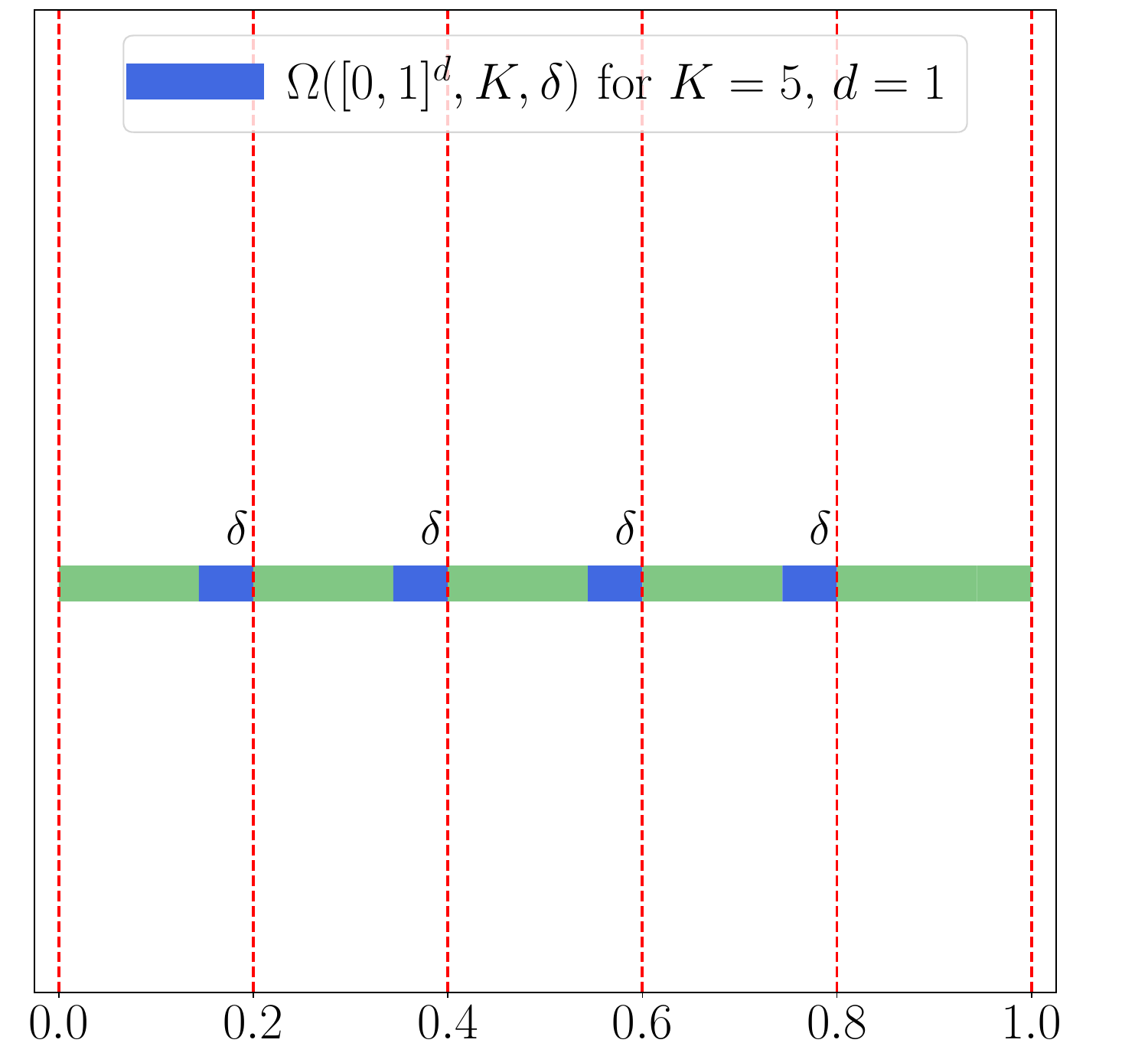}
     		\subcaption{}
     	\end{subfigure}
     \begin{minipage}{0.07\textwidth}
     	\,
     \end{minipage}
     	\begin{subfigure}[b]{0.33\textwidth}
     		\centering            \includegraphics[width=0.999\textwidth]{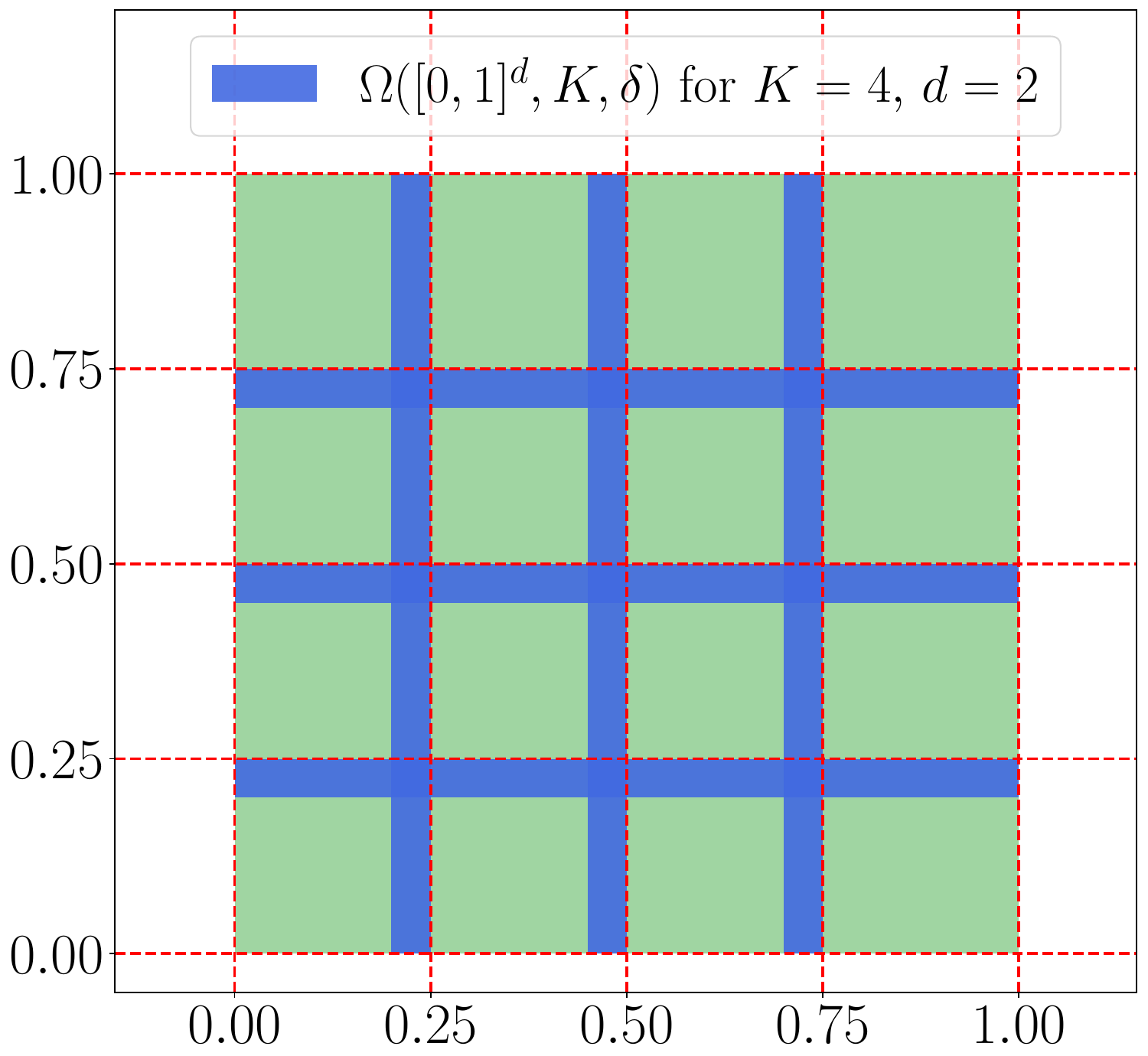}
     		\subcaption{}
     	\end{subfigure}
 	\end{minipage}
     	\caption{Two examples of trifling regions. (a)  $K=5,d=1$. (b) $K=4,d=2$.}
     	\label{fig:region}
     \end{figure}
 
     \item Let $\holder{\lambda}{\alpha}$ denote the space of  H{\"o}lder continuous functions on $[0,1]^d$ of order $\alpha\in (0,1]$ with  a H\"older constant $\lambda>0$. 
     
     \item For a continuous piecewise linear function $f(x)$, the $x$ values where the slope changes are typically called \textbf{breakpoints}. 
     
     \item Let $\cpl(\R,n)$ denote the space that consists of all continuous piecewise linear functions with at most $n$ breakpoints on $\R$.

     \item Let $\sigma:\R\to \R$ denote the rectified linear unit (ReLU), i.e. $\sigma(x)=\max\{0,x\}$. With a slight abuse of notation, we define $\sigma:\R^d\to \R^d$ as $\sigma(\bmx)=\left[\begin{array}{c}
     	\max\{0,x_1\}  \\
     	\vdots \\
     	\max\{0,x_d\}
     \end{array}\right]$ for any $\bmx=[x_1,\cdots,x_d]^T\in \R^d$.

     \item We will use $\NNF$ to denote a function implemented by a ReLU network for short and use Python-type notations to specify a class of functions implemented by ReLU networks with several conditions, e.g., $\NNF(\tn{c}_1;\ \tn{c}_2;\ \cdots;\ \tn{c}_m)$ is a set of functions implemented by  ReLU networks satisfying $m$ conditions given by $\{\tn{c}_i\}_{1\leq i\leq m}$, each of which may specify the number of inputs ($\NNinput$), the number of outputs ($\NNoutput$), 
    %  the total number of neurons in all hidden layers ($\NNneuron$),
     the number of hidden layers ($\NNdepth$), the total  number of parameters ($\NNparameter$), and the width in each hidden layer ($\NNwidthvec$), the maximum width of all hidden layers ($\NNwidth$), etc. For example, if $\phi\in \NNF(\NNinput=2\NNspace \NNwidthvec=[100,100]\NNspace\NNoutput=1)$,  then $\phi$ is a function satisfying
     \begin{itemize}
         \item $\phi$ maps from $\R^2$ to $\R$.
         \item $\phi$ can be implemented by a ReLU network with two hidden layers and the number of neurons in each hidden layer is $100$.
     \end{itemize}
 
%     \item $[n]^L$ is short for $[n,n,\cdots,n]\in \N^{L}$. 
%     For example, \[\NNF(\NNinput=d\NNspace\NNwidthvec=[100,100])=\NNF(\NNinput=d\NNspace\NNwidthvec=[100]^2).\]

     \item For any function $\phi\in \NNF(\NNinput=d\NNspace\NNwidthvec=[N_1,N_2,\cdots,N_L]\NNspace\NNoutput=1)$, if we set $N_0=d$ and $N_{L+1}=1$, then the architecture of the network implementing $\phi$ can be briefly described as follows:
    \begin{equation*}
    \begin{aligned}
    \bm{x}=\widetilde{\bm{h}}_0 
    \myto{2.2}^{\bm{W}_0,\ \bm{b}_0}_{\calL_0} \bm{h}_1
    \myto{1.25}^{\sigma} \tilde{\bm{h}}_1 \ \cdots\ \myto{2.7}^{\bm{W}_{L-1},\ \bm{b}_{L-1}}_{\calL_{L-1}} \bm{h}_L
    \myto{1.25}^{\sigma} \tilde{\bm{h}}_L
    \myto{2.2}^{\bm{W}_{L},\ \bm{b}_{L}}_{\calL_L} \bm{h}_{L+1}=\phi(\bm{x}),
    \end{aligned}
    \end{equation*}
    where $\bm{W}_i\in \R^{N_{i+1}\times N_{i}}$ and $\bm{b}_i\in \R^{N_{i+1}}$ are the weight matrix and the bias vector in the $i$-th affine linear transformation $\calL_i$, respectively, i.e., 
    \[\bm{h}_{i+1} =\bm{W}_i\cdot \tilde{\bm{h}}_{i} + \bm{b}_i\eqqcolon \calL_i(\tilde{\bm{h}}_{i}),\quad \tn{for $i=0,1,\cdots,L$,}\]  
    and
    \[
       \tilde{\bm{h}}_i=\sigma(\bm{h}_i),\quad \tn{for $i=1,2,\cdots,L$.}
    \]
    In particular, $\phi$ can be represented in a form of function compositions as follows.
    \begin{equation*}
        \phi =\calL_L\circ\sigma\circ\calL_{L-1}\circ \sigma\circ \ \cdots \  \circ \sigma\circ\calL_1\circ\sigma\circ\calL_0,
    \end{equation*}
    which has been illustrated in Figure~\ref{fig:ReLUeg}.
    \begin{figure}[!htp]        
     	\centering
     		\centering            \includegraphics[width=0.7\textwidth]{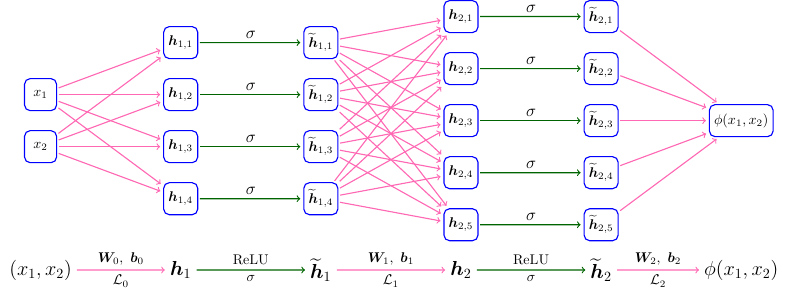}
     	\caption{An example of a ReLU network with width $5$ and depth $2$. }
     	\label{fig:ReLUeg}
     \end{figure}

     \item The expression ``a network with width $N$ and depth $L$'' means
     \begin{itemize}
         \item The maximum width of this network for all \textbf{hidden} layers is no more than $N$.
         \item The number of \textbf{hidden} layers of this network is no more than $L$.
     \end{itemize} 
\end{itemize}

\subsection{Proof of Theorem~\ref{thm:main}}
\label{sec:proof:main}

The key point is to construct piecewise constant functions to approximate continuous functions in the proof. However, it is impossible to construct a piecewise constant function implemented by a ReLU network due to the continuity of ReLU networks. Thus, we introduce the trifling region  $\Omega([0,1]^d,K,\delta)$, defined in Equation~\eqref{eq:triflingRegionDef}, and use ReLU networks to implement piecewise constant functions outside the trifling region.
To prove Theorem~\ref{thm:main}, we first introduce a weaker variant of Theorem~\ref{thm:main}, showing how to construct ReLU networks to pointwisely approximate continuous functions except for the trifling region. 
\begin{theorem}
	\label{thm:mainGap}
	Given  a function $f\in C([0,1]^d)$, for any $N\in \N^+$ and $L\in \N^+$,
	there exists a function $\phi$ implemented by a  ReLU network with width $\max\big\{8d\lfloor N^{1/d}\rfloor+3d,\, 16N+30\big\}$
	and depth $11L+18$
	such that $ \|\phi\|_{L^\infty(\R^d)}\le |f(\bmzero)|+ \omega_f(\sqrt{d})$ and 
	\begin{equation*}
	|f(\bmx)-\phi(\bmx)|\le 130\sqrt{d}\,\omega_f\Big(\big(N^2L^2 \log_3(N+2)\big)^{-1/d}\Big),\quad \tn{for any $\bmx\in [0,1]^d\backslash\Omega([0,1]^d,K,\delta)$},
	\end{equation*}
	where $K=\lfloor N^{1/d}\rfloor^2\lfloor L^{1/d}\rfloor^2 \big\lfloor \lfloor \log_3(N+2)\rfloor^{1/d}\big\rfloor$ and $\delta$ is an arbitrary number in $(0,\tfrac{1}{3K}]$.
\end{theorem}

With Theorem~\ref{thm:mainGap} that will be proved in Section~\ref{sec:proof:mainGap}, we can easily prove Theorem~\ref{thm:main} for the case $p\in[1,\infty)$. To attain the rate in $L^\infty$-norm, we need to control the approximation error in the trifling region. 
To this end, we introduce a theorem to deal with the approximation inside the trifling region $\Omega([0,1]^d,K,\delta)$.

\begin{theorem}[Theorem $3.7$ of \cite{shijun:thesis} or Theorem $2.1$ of \cite{shijun3}]
	\label{thm:Gap}
	Given any $\varepsilon>0$, $N,L,K\in \N^+$, and $\delta\in (0, \tfrac{1}{3K}]$,
	assume $f$ is a continuous function in $C([0,1]^d)$ and $\tildephi$ can be implemented by a ReLU network with width $N$ and depth $L$. If 
	\begin{equation*}
	|f(\bmx)-\tildephi(\bmx)|\le \varepsilon,\quad \tn{for any $\bmx\in [0,1]^d\backslash \Omega([0,1]^d,K,\delta)$,}
	\end{equation*}
	then there exists a function $\phi$  implemented by a new ReLU network with width $3^d(N+4)$ and depth $L+2d$ such that 
	\begin{equation*}
	|f(\bmx)-\phi(\bmx)|\le \varepsilon+d\cdot\omega_f(\delta),\quad \tn{for any $\bmx\in [0,1]^d$.}
	\end{equation*}
\end{theorem}
Now we are ready to prove Theorem~\ref{thm:main} by assuming Theorem~\ref{thm:mainGap} is true, which  will be proved later in Section~\ref{sec:proof:mainGap}.
\begin{proof}[Proof of Theorem~\ref{thm:main}]		
	We may assume $f$ is not a constant function since it is  a trivial case. Then $\omega_f(r)>0$ for any $r>0$. 
	Let us first consider the case $p\in [1,\infty)$. Set 
	$K=\lfloor N^{1/d}\rfloor^2\lfloor L^{1/d}\rfloor^2\big\lfloor \lfloor \log_3(N+2)\rfloor^{1/d}\big\rfloor$ and choose a small $\delta\in (0,\tfrac{1}{3K}]$ such that
	\begin{equation*}
	\begin{split}
	Kd\delta\big( \black{2}|f(\bmzero)|+\black{2}\omega_f(\sqrt{d})\big)^p
	&=\lfloor N^{1/d}\rfloor^2\lfloor L^{1/d}\rfloor^2 \big\lfloor \lfloor \log_3(N+2)\rfloor^{1/d}\big\rfloor 
	d\delta \big( 2|f(\bmzero)|+2\omega_f(\sqrt{d})\big)^p\\
	&\le \bigg(\omega_f\Big(\big(N^2L^2 \log_3(N+2)\big)^{-1/d}\Big)\bigg)^p. 
	\end{split}
	\end{equation*}
	By Theorem~\ref{thm:mainGap}, there exists a function $\phi$ implemented by a ReLU network with width \[\max\big\{8d\lfloor N^{1/d}\rfloor+3d,\, 16N+30\big\}\le 16\max \big\{d\lfloor N^{1/d}\rfloor,\, N+2\big\}\]
	and depth $11L+18$ such that \black{$\|\phi\|_{L^\infty(\R^d)}\le |f(\bmzero)|+\omega_f(\sqrt{d})$} and 
	\begin{equation*}
	|f(\bmx)-\phi(\bmx)|\le 130\sqrt{d}\,\omega_f\Big(\big(N^2L^2 \log_3(N+2)\big)^{-1/d}\Big),\quad \tn{for any $\bmx\in [0,1]^d\backslash\Omega([0,1]^d,K,\delta)$}.
	\end{equation*}
	It follows from $\mu(\Omega([0,1]^d,K,\delta))\le Kd\delta$ and \black{$\|f\|_{L^\infty([0,1]^d)}\le |f(\bmzero)|+\omega_f(\sqrt{d})$} that
	\begin{equation*}
	\begin{split}
	\|f-\phi\|_{L^p([0,1]^d)}^p
	&=\int_{\Omega([0,1]^d,K,\delta)}|f(\bmx)-\phi(\bmx)|^p\tn{d}\bmx+\int_{[0,1]^d\backslash\Omega([0,1]^d,K,\delta)}|f(\bmx)-\phi(\bmx)|^p\tn{d}\bmx\\
	&\le Kd\delta\big( \black{2}|f(\bmzero)|+\black{2}\omega_f(\sqrt{d})\big)^p+ \bigg(130\sqrt{d}\,\omega_f\Big(\big(N^2L^2 \log_3(N+2)\big)^{-1/d}\Big)\bigg)^p\\
	&\le \bigg(\omega_f\Big(\big(N^2L^2 \log_3(N+2)\big)^{-1/d}\Big)\bigg)^p
			+ \bigg(130\sqrt{d}\,\omega_f\Big(\big(N^2L^2 \log_3(N+2)\big)^{-1/d}\Big)\bigg)^p\\
	&\le \bigg(131\sqrt{d}\,\omega_f\Big(\big(N^2L^2 \log_3(N+2)\big)^{-1/d}\Big)\bigg)^p.
	\end{split}
	\end{equation*}
	Hence, $\|f-\phi\|_{L^p([0,1]^d)}\le 131\sqrt{d}\,\omega_f\Big(\big(N^2L^2 \log_3(N+2)\big)^{-1/d}\Big)$.
	
	Next, let us discuss the case $p=\infty$. Set $K=\lfloor N^{1/d}\rfloor^2\lfloor L^{1/d}\rfloor^2\big\lfloor \lfloor \log_3(N+2)\rfloor^{1/d}\big\rfloor$ and choose a small $\delta\in(0,\tfrac{1}{3K}]$ such that 
	\begin{equation*}
	d\cdot \omega_f(\delta)\le \omega_f\Big(\big(N^2L^2 \log_3(N+2)\big)^{-1/d}\Big).
	\end{equation*}	
	By Theorem~\ref{thm:mainGap}, there exists a function $\tildephi$ implemented   by a  ReLU network with width $\max\big\{8d\lfloor N^{1/d}\rfloor+3d,\, 16N+30\big\}$
	and depth $11L+18$ such that
	\begin{equation*}
	|f(\bmx)-\tildephi(\bmx)|\le 130\sqrt{d}\,\omega_f\Big(\big(N^2L^2 \log_3(N+2)\big)^{-1/d}\Big)\eqqcolon \varepsilon,
	\end{equation*}
\tn{for any $\bmx\in [0,1]^d\backslash\Omega([0,1]^d,K,\delta)$}.
	By Theorem~\ref{thm:Gap}, there exists a  function $\phi$ implemented by a ReLU network  with width 
	\begin{equation*}
	3^d\Big(\max\big\{8d\lfloor N^{1/d}\rfloor+3d,\, 16N+30\big\}+4\Big)\le 3^{d+3}\max\big\{d\lfloor N^{1/d}\rfloor,\, N+2\big\}
	\end{equation*}
	and depth $11L+18+2d$ such that 
	\begin{equation*}
	|f(\bmx)-\phi(\bmx)|\le \varepsilon+d\cdot \omega_f(\delta)\le  131\sqrt{d}\,\omega_f\Big(\big(N^2L^2 \log_3(N+2)\big)^{-1/d}\Big),\quad \tn{for any $\bmx\in [0,1]^d$}.
	\end{equation*}
	So we finish the proof.
\end{proof}

%%%%%%%%%%%%%%%%%%%%%%%%%%%%%%%%%%%%%%%
%%%%%%%%%%%%%%%%%%%%%%%%%%%%%%%%%%%%%%%%%%%
\subsection{Optimality}\label{sec:optimality}
%This section will show that the approximation rate in Theorem~\ref{thm:main} is nearly tight and there is no room to improve for the function class $\holder{\lambda}{\alpha}$. 
%Theorem~\ref{thm:lowInfty} below shows that the approximation rate  $\calO(\omega_f(N^{-(2/d+\rho)}L^{-(2/d+\rho)}))$ for any $\rho>0$ is unachievable, implying the approximation rate in Theorem~\ref{thm:main} is nearly tight for the function class $\holder{\lambda}{\alpha}$.

This section will show that the approximation rates in Theorem~\ref{thm:main} and Corollary~\ref{coro:main} are optimal and there is no room to improve for the function class $\holder{\lambda}{\alpha}$. 
% \black{Therefore, our results are also optimal in the whole continuous functions space.} 
\black{Therefore, the approximation rate for the whole continuous functions space in terms of width and depth in Theorem~\ref{thm:main} cannot be improved.}
A typical method to characterize the optimal approximation theory of neural networks is to study the connection
between the approximation error and Vapnik–Chervonenkis (VC) dimension \cite{yarotsky2017,yarotsky18a,shijun3,shijun:thesis,shijun2}. This method relies on the VC-dimension upper bound given in \cite{pmlr-v65-harvey17a}. In this paper, we adopt this method with several modifications to simplify the proof.
%To this end, 
%%To illustrate the optimality of Theorem~\ref{thm:main}, 
%we introduce a tool called VC-dimension and 
%Vapnik–Chervonenkis (VC) dimension

Let us first present the definitions of VC-dimension and related concepts. 
% In short, the VC-dimension of a function set is defined as the size of the largest set of points that this class of functions can shatter.
Let  $H$ be a class of functions mapping from a general domain $\mathcal{X}$ to $\{0,1\}$. 
We say $H$ shatters the set $\{\bmx_1,\bmx_2,\cdots,\bmx_m\}\subseteq \mathcal{X}$ if
\begin{equation*}
	\Big| \Big\{\big[h(\bmx_1),h(\bmx_2),\cdots,h(\bmx_m)\big]^T\in \{0,1\}^m: h\in H\Big\}\Big|=2^m,
\end{equation*}
where $|\cdot|$ denotes the size of a set.
This equation means, given any $\theta_i\in \{0,1\}$ for $i=1,2,\cdots,m$, there exists $h\in H$ such that
$h(\bmx_i)=\theta_i$ for all $i$. For a general  function set $\scrF$ mapping from $\mathcal{X}$ to $\R$, we say $\scrF$ shatters $\{\bmx_1,\bmx_2,\cdots,\bmx_m\}\subseteq \mathcal{X}$ if $\calT\circ \scrF$ does, where \begin{equation*}
		\calT(t)\coloneqq \left\{
		\begin{array}{ll}
	    1,&\  t\ge 0,\\ 0, &\ t< 0\phantom{,}
        \end{array}
		\right. 
		\quad \tn{and}\quad \calT\circ \scrF\coloneqq \{\calT\circ f: f\in \scrF\}.
\end{equation*}

For any $m\in \N^+$, we define the growth function of $H$ as 
\begin{equation*}
	\Pi_H(m)\coloneqq \max_{\bmx_1,\bmx_2,\cdots,\bmx_m\in \mathcal{X}} \Big| \Big\{\big[h(\bmx_1),h(\bmx_2),\cdots,h(\bmx_m)\big]^T\in \{0,1\}^m: h\in H\Big\}\Big|.
\end{equation*}

\begin{definition}[VC-dimension]
	Let $H$ be a class of functions from $\mathcal{X}$ to $\{0,1\}$.
	The VC-dimension of $H$, denoted by  $\vcd(H)$, is the size of the largest shattered set, namely, 
%	the largest $m$ such that $\Pi_H(m)=2^m$. By convention, $\vcd(H)=\infty$ if $\Pi_H(m)=2^m$ for all $m\in \N^+$. In fact, we can also define $\vcd(H)$ as
	\[\vcd(H)\coloneqq \sup \{m\in\N^+:\Pi_H(m)=2^m \}\] 
	if $\{m\in\N^+:\Pi_H(m)=2^m \}$ is not empty. In the case of $\{m\in\N^+:\Pi_H(m)=2^m \}=\emptyset$, we may define $\vcd(H)=0$.
	
% 	\footnote{In this definition, we assume $\{m\in\N^+:\Pi_H(m)=2^m \}$ is not empty since it is a trivial case. In the case of $\{m\in\N^+:\Pi_H(m)=2^m \}=\emptyset$, we may define $\vcd(H)=0$. }
	
	%Let $\scrF$ be a class of functions from $\mathcal{X}$ to $\R$. The VC-dimension of $\scrF$ is defined by $\vcd(\sgn(\scrF))$, where $\sgn(t)=\left\{{1,\, t>0, \atop 0, \, t\le 0\phantom{,}}\right.$ and $\sgn(\scrF)\coloneqq \{\sgn(f)=\sgn\circ f: f\in \scrF\}$.
	%For simplicity, we denote  $\vcd(\sgn(\scrF))$ by $\vcd(\scrF)$  for short.
	Let $\scrF$ be a class of functions from $\mathcal{X}$ to $\R$. The VC-dimension of $\scrF$, denoted by $\vcd(\scrF)$, is defined by $\vcd(\scrF)\coloneqq\vcd(\calT\circ\scrF)$, where
	\begin{equation*}
		\calT(t)\coloneqq \left\{
				\begin{array}{ll}
	    1,&\  t\ge 0,\\ 0, &\ t< 0\phantom{,}
        \end{array}
        \right.
		\quad \tn{and}\quad \calT\circ \scrF\coloneqq \{\calT\circ f: f\in \scrF\}.
	\end{equation*}
	%For simplicity, we denote  $\vcd(\calT\circ \scrF)$ by $\vcd(\scrF)$  for short. 
	In particular, the expression ``VC-dimension of a network (architecture)'' means the VC-dimension of the function set that consists of all functions implemented by this network (architecture).
\end{definition}

We remark that one may also define $\vcd(\scrF)$ as $\vcd(\scrF)\coloneqq\vcd(\widetilde{\calT}\circ\scrF)$, where
	\begin{equation*}
	\widetilde{\calT}(t)\coloneqq
	\left\{
		\begin{array}{ll}
	    1,&\  t> 0,\\ 0, &\ t\le 0\phantom{,}
        \end{array}
        \right.
	\quad \tn{and}\quad \widetilde{\calT}\circ \scrF\coloneqq \{\widetilde{\calT}\circ f: f\in \scrF\}.
\end{equation*}
 Note that function spaces generated by networks  are closed under linear transformation. Thus, these two definitions of VC-dimension are equivalent.

%
%establish a theorem, Theorem~\ref{thm:linkVcdRate},  to characterize the connection between VC-dimension and the approximation rate. We denote VC-dimension of a function set $\scrF$ by $\vcd(\scrF)$. 
%%Let us first establish a theorem to 
The theorem below, similar to Theorem~$4.17$ of \cite{shijun:thesis}, reveals the connection between VC-dimension and the approximation rate.
\begin{theorem}
	\label{thm:linkVcdRate}
	Assume $\scrF$ is a  set of functions mapping from $[0,1]^d$ to $\R$. For any $\varepsilon >0$, if  $\vcd(\scrF)\ge 1$ and
	\begin{equation}
	\label{eq:distLipScrF}
	\inf_{\phi\in \scrF}\|\phi-f\|_{L^\infty([0,1]^d)}\le \varepsilon,\quad \tn{for any $f\in \Lip([0,1]^d,\alpha,1)$,}
	\end{equation}
	then $\vcd(\scrF)\ge (9\varepsilon)^{-d/\alpha}$.
\end{theorem}

This theorem demonstrates the connection between VC-dimension of $\scrF$ and the approximation rate using elements of $\scrF$ to approximate functions in $\Lip([0,1]^d,\alpha,\lambda  )$. To be precise, the VC-dimension of $\scrF$ determines an approximation rate lower bound  $\vcd(\scrF)^{-\alpha/d}/9$, which is the best possible approximation rate. 
Denote the best approximation error of functions in $\Lip([0,1]^d,\alpha,1)$ approximated by ReLU networks with width $N$ and depth $L$ as 
\begin{equation*}
	\calE_{\alpha,d}(N,L) \coloneqq \sup_{f\in \Lip([0,1]^d,\alpha,1)} \bigg(\inf_{\phi\in \NNF(\NNwidth\le N;\,\NNdepth \le L)} \|\phi-f\|_{L^\infty([0,1]^d)}\bigg).
\end{equation*}
We have three remarks listed below.
\begin{enumerate}[(i)]
	\item A large VC-dimension cannot guarantee a good approximation rate. For example, it is easy to verify that
	\begin{equation*}
		\vcd\Big(\big\{f: f(x)=\cos(ax), \ a\in\R \big\}\Big)=\infty.
	\end{equation*}
However, functions in $\big\{f: f(x)=\cos(ax), \ a\in\R \big\}$ cannot approximate H\"older continuous functions well.
	\item A large VC-dimension is necessary for a good approximation rate, because the best possible approximation rate is controlled by an expression of VC-dimension, as shown in Theorem~\ref{thm:linkVcdRate}.
	\black{It is shown in Theorem~$6$ and $8$ of \cite{pmlr-v65-harvey17a} that the VC-dimension of ReLU networks has two types of upper bounds: $\calO(WL\ln W)$ and $\calO(WU)$. Here, $W$, $L$, and  $U$ are the numbers of parameters, layers, and neurons, respectively. If we let $N$ denote the maximum width of the network, then $W=\calO(N^2L)$ and $U=\calO(NL)$, implying that
	\begin{equation*}
	    WL\ln W=\calO\big(N^2L\cdot L\ln (N^2L)\big)=\calO\big(N^2L^2\ln (NL)\big)
	\end{equation*} and 
	\begin{equation*}
	    WU=\calO(N^2L\cdot NL)=\calO(N^3L^2).
	\end{equation*}
	It follows that  
	\begin{equation*}
		\vcd\big(\NNF(\NNwidth\le N\NNspace \NNdepth\le L)\big)\le \min\Big\{ \calO\big(N^2L^2\ln(NL)\big), \calO(N^3L^2)\Big\},
	\end{equation*}
	}%%%%%%%%%%%%%
	deducing
	\begin{equation}\label{eq:error:lb:ub}
		\mathResize[0.9]{
		\underbrace{C_1(\alpha,d) \Big(\min\{ N^2L^2\ln(NL), N^3L^2\} \Big)^{-\alpha/d} \le }_{ \tn{implied by Theorem~\ref{thm:linkVcdRate}}  }
		\,\calE_{\alpha,d}(N,L) \,
		\underbrace{\le  C_2(\alpha,d)\Big(N^2L^2\ln N\Big)^{-\alpha/d}}_{ \tn{implied by Corollaries~\ref{coro:tildeNL} and \ref{coro:main}}  },
	}%%%%%%%%%%%%%5
	\end{equation}	
where $C_1(\alpha,d)$ and $C_2(\alpha,d)$ are two positive constants determined by $s,d$, and $C_2(s,d)$ can be explicitly expressed.
\begin{itemize}
    \item 
When $L=L_0$ is fixed, Equation~\eqref{eq:error:lb:ub} implies
\begin{equation*}
	C_1(\alpha,d,L_0)(N^2\ln N)^{-\alpha/d}\,\le \,\calE_{\alpha,d}(N,L_0) \,\le\, C_2(\alpha,d,L_0) (N^2\ln N)^{-\alpha/d},
\end{equation*}
where $C_1(\alpha,d,L_0)$ and $C_2(\alpha,d,L_0)$ are two positive constants determined by $\alpha,d,L_0$.
\item 
When $N=N_0$ is fixed, Equation~\eqref{eq:error:lb:ub} implies
\begin{equation*}
	C_1(\alpha,d,N_0)L^{-2\alpha/d}\,\le \,\calE_{\alpha,d}(N_0,L) \,\le\, C_2(\alpha,d,N_0) L^{-2\alpha/d},
\end{equation*}
where $C_1(\alpha,d,N_0)$ and $C_2(\alpha,d,N_0)$ are two positive constants determined by $\alpha,d,N_0$.
\black{
\item 
It is easy to verify that Equation~\eqref{eq:error:lb:ub} is tight except for the following region 
\[\big\{(N,L)\in \N^2: C_3(\alpha,d)\le N\le L^{C_4(\alpha,d)}\big\},\]
$C_3=C_3(\alpha,d)$ and $C_4=C_4(\alpha,d)$ are two positive constants. See Figure~\ref{fig:opt:NLplane} for an illustration for the case $C_3=1000$ and $C_4=1/100$.}
\end{itemize}
\end{enumerate}

%For example, if $\vcd(\scrF)\le v$, then the possible optimal approximation rate is $v^{-\alpha/d}/9$.
%A typical application of this theorem is to prove the optimality of approximation rates when using ReLU networks to approximate functions in $\Lip([0,1]^d,\alpha,\lambda  )$. It is shown in \cite{pmlr-v65-harvey17a} that VC-dimension of ReLU networks with a fixed architecture with $W$ parameters and $L$ layers has an upper bound $\calO(WL\ln W)$. It follows that VC-dimension of ReLU networks with width $\calO(N)$ and depth $\calO(L)$ is bounded by $\calO(N^2L^2\ln(NL))$. Hence, the best possible approximation rate for $\Lip([0,1]^d,\alpha,\lambda  )$ is $\calO( (N^2L^2\ln(NL))^{-\alpha/d})$, which means our approximation rate $19\sqrt{d} N^{-2\alpha/d}L^{-2\alpha/d}$ is optimal up to a logarithm term. We call such optimality ``nearly optimal''. 

Finally, let us present the detailed proof of  Theorem~\ref{thm:linkVcdRate}.
\begin{proof}[Proof of Theorem~\ref{thm:linkVcdRate}]
	Recall that the VC-dimension of a function set is defined as the size of the largest set of points that this class of functions can shatter. So our goal is to find a subset of $\scrF$ to shatter $\calO(\varepsilon^{-d/\alpha})$ points in $[0,1]^d$, which can be divided into two steps.
	\begin{itemize}
		\item Construct $\{f_\chi:\chi\in \mathscr{B}\}\subseteq\Lip([0,1]^d,\alpha,1)$ that scatters $\calO(\varepsilon^{-d/\alpha})$ points, where $\scrB$ is a set defined later.
		\item Design $\phi_\chi\in \scrF$, for each $\chi\in \scrB$, based on $f_\chi$ and Equation~\eqref{eq:distLipScrF} such that	$\{\phi_\chi:\chi\in\scrF\}\subseteq\scrF$ also shatters $\calO(\varepsilon^{-d/\alpha})$ points.
	\end{itemize}  
	The details of these two steps can be found below.
	
	\mystep{1}{Construct $\{f_\chi:\chi\in \mathscr{B}\}\subseteq\Lip([0,1]^d,\alpha,1)$ that scatters $\calO(\varepsilon^{-d/\alpha})$ points.}
	
	We may assume $\varepsilon\le 2/9$ since the case $\varepsilon>2/9$ is trivial. In fact, $\varepsilon>2/9$ implies
	\[\vcd(\scrF)\ge 1\ge 1/2\ge 2^{-d/\alpha}> (9\varepsilon)^{-d/\alpha}.\]
	Let $K=\lfloor \left(9\varepsilon/2\right)^{-1/\alpha}\rfloor\in \N^+$  and divide $[0,1]^d$ into $K^d$ non-overlapping sub-cubes $\{Q_{\bmbeta}\}_{\bmbeta}$ as follows: 
	\[Q_{\bmbeta}\coloneqq \big\{{\bm{x}}=[x_1,x_2,\cdots,x_d]^T\in[0,1]^d: x_i\in [\tfrac{\beta_i}{K},\tfrac{\beta_i+1}{K}],\ i=1,2,\cdots,d\big\},\] for any index vector  ${\bmbeta}= [\beta_1,\beta_2,\cdots,\beta_d]^T\in \{0,1,\cdots,K-1\}^d$.
	
	Let $Q (\bm{x}_0,\eta)$ denote the closed cube with  center  ${\bm{x}}_0\in\R^d$ and sidelength $\eta>0$. 
	Define a function $\zeta_Q$ on $[0,1]^d$ corresponding to $Q=Q (\bm{x}_0,\eta)\subseteq [0,1]^d$  such that:
	\begin{itemize}
		\item $\zeta_{Q} (\bm{x}_0)= (\eta/2)^\alpha/2$;
		\item $\zeta_{Q} (\bm{x})=0$ for any $\bmx\notin Q \backslash\partial Q $, where $\partial Q $ is the boundary of $Q$;
		\item $\zeta_{Q }$ is linear on the line that connects ${\bm{x}}_0$ and ${\bm{x}}$ for any ${\bm{x}}\in \partial Q$.
	\end{itemize}

	Define
	\begin{equation*}
	\mathscr{B}\coloneqq \big\{\chi: \chi \tn{ is  a map from } \{0,1,\cdots,K-1\}^d \tn{   to } \{-1,1\}\big\}.
	\end{equation*}
	For each $\chi\in \mathscr{B}$, we define
	\begin{equation*}
	f_\chi (\bm{x})\coloneqq \sum_{{\bmbeta}\in \{0,1,\cdots,K-1\}^d} \chi ({\bmbeta})\zeta_{Q_{\bmbeta}} (\bm{x}),
	\end{equation*}
	where $\zeta_{Q_{\bmbeta}} (\bm{x})$ is the associated function introduced just above.  
% 	See two examples of $f_\chi$ in Figure~\ref{fig:fchiEgs}.
	It is easy to check that  $ \{f_\chi:\chi\in \mathscr{B}\}\subseteq \Lip([0,1]^d,\alpha,1)$ 
	%and $\{f_\chi:\chi\in \mathscr{B}\}$ 
	can shatter
	$K^d=\calO(\varepsilon^{-d/\alpha})$
	points in $[0,1]^d$.

% 	\begin{figure}[!htp]        
% 	\centering
% 	\begin{minipage}{0.985\textwidth}
% 		\centering
% 		\begin{subfigure}[b]{0.48\textwidth}
% 			\centering            
% 			\includegraphics[width=0.999\textwidth]{fchiEg1.pdf}
% % 			\subcaption{}
% 		\end{subfigure}
% %		\begin{minipage}{0.07\textwidth}
% %			\,
% %		\end{minipage}
% 		\begin{subfigure}[b]{0.48\textwidth}
% 			\centering            
% 			\includegraphics[width=0.999\textwidth]{fchiEg2.pdf}
% % 			\subcaption{}
% 		\end{subfigure}
% 	\end{minipage}
% 	\caption{Two examples of $f_\chi$ for $K=5$ and $d=1$.}
% 	\label{fig:fchiEgs}
% \end{figure}

	\mystep{2}{Construct $\{\phi_\chi:\chi\in\mathscr{B}\}$ that also scatters $\calO(\varepsilon^{-d/\alpha})$ points.}

	By Equation~\eqref{eq:distLipScrF}, for each $\chi\in\scrB$, there exists $\phi_\chi \in \scrF$ such that 
	\begin{equation*}
	\|\phi_\chi-f_\chi\|_{L^\infty([0,1]^d)}\le \varepsilon+\varepsilon/81.
	\end{equation*}
	Let $\mu(\cdot)$ denote the Lebesgue measure of a set. Then, for each $\chi\in \scrB$, there exists $\calH_\chi\subseteq [0,1]^d$ with $\mu(\calH_\chi)=0$ such that
	\begin{equation*}
	|\phi_\chi(\bmx)-f_\chi(\bmx)|\le \tfrac{82}{81}\varepsilon,\quad \tn{for any $\bmx\in [0,1]\backslash \calH_\chi$}.
	\end{equation*}
	Set $\calH=\cup_{\chi\in \scrB} \calH_\chi$, then we have $\mu(\calH)=0$ and 
	\begin{equation}
	\label{eq:fBetaPhiBeta}
	|\phi_\chi(\bmx)-f_\chi(\bmx)|\le \tfrac{82}{81}\varepsilon,\quad \tn{for any $\chi\in \scrB$ and $\bmx\in [0,1]\backslash \calH$}.
	\end{equation}
	
	Since $Q_{\bmbeta}$ has a sidelength $\tfrac{1}{K}=\tfrac{1}{\lfloor (9\varepsilon/2)^{-1/\alpha}\rfloor}$,  we have, for each ${\bmbeta}\in \{0,1,\cdots,K-1\}^d$ and any $\bm{x} \in \tfrac{1}{10}Q_{\bmbeta}$\footnote{$\tfrac1{10}Q_\bmbeta$ denotes the closed cube whose sidelength is $1/10$ of that of $Q_\bmbeta$ and which shares the same center of $Q_\bmbeta$.},
	\begin{equation}
	\label{eq:fBetaLowerBound}
	|f_\chi (\bm{x})|=|\zeta_{Q_{\bmbeta}}(\bm{x})|\ge \tfrac{9}{10}|\zeta_{Q_{\bmbeta}} (\bm{x}_{Q_{\bmbeta}})|=\tfrac{9}{10}(\tfrac1{2\lfloor (9 \varepsilon/2)^{-1/\alpha}\rfloor})^\alpha/2\ge \tfrac{81}{80}\varepsilon,
	\end{equation}
	where $\bm{x}_{Q_{\bmbeta}}$ is the center of $Q_{\bmbeta}$.
	
	Note that $(\tfrac{1}{10}Q_\bmbeta)\backslash\calH$ is not empty, since $\mu\big((\tfrac{1}{10}Q_\bmbeta)\backslash\calH\big)>0$ for each $\bmbeta\in\{0,1,\cdots,K-1\}^d$. Together with Equations~\eqref{eq:fBetaPhiBeta} and \eqref{eq:fBetaLowerBound},  there exists $\bmx_\bmbeta\in (\tfrac{1}{10} Q_\bmbeta) \backslash \calH$ such that, for each $\bmbeta\in \{0,1,\cdots,K-1\}^d$ and each $\chi\in \scrB$,	
	\begin{equation*}
	|f_\chi (\bmx_{\bmbeta})|\ge  \tfrac{81}{80} \varepsilon> \tfrac{82}{81}\varepsilon\ge  |f_\chi (\bmx_{\bmbeta})-\phi_\chi (\bmx_{\bmbeta})|.\label{inq3}
	\end{equation*}

	Hence,  $f_\chi (\bmx_{\bmbeta})$ and $\phi_\chi (\bmx_{\bmbeta})$ have the same sign for each $\chi\in\mathscr{B}$ and ${\bmbeta}\in \{0,1,\cdots,K-1\}^d$. Then  $ \{\phi_\chi:\chi\in\mathscr{B}\}$ shatters $\big\{\bmx_\bmbeta:\bmbeta\in \{0,1,\cdots,K-1\}^d\big\}$ since $\{f_\chi:\chi\in\mathscr{B}\}
	$ shatters $\big\{\bmx_\bmbeta:\bmbeta\in \{0,1,\cdots,K-1\}^d\big\}$. Therefore, 
	\begin{equation*}
% 	\label{eq:NLVcdimLowerBound}
	\tn{VCDim}(\scrF)\ge \tn{VCDim}\big( \{\phi_\chi:\chi\in\mathscr{B}\} \big)\geq K^d=\lfloor (9\varepsilon/2)^{-1/\alpha}\rfloor^d\ge (9\varepsilon)^{-d/\alpha},
	\end{equation*}
	where the last inequality comes from the fact $\lfloor x\rfloor \ge x/2\ge x/(2^{1/\alpha})$ for any $x\in [1,\infty)$ and $\alpha\in(0,1]$.
	So we finish the proof.  
\end{proof}

%%%%%%%%%%%%%%%%%%%%%%%%%%%%%%%%%%%%%%%%%%%%%%%%%%%%%%%%%55
\subsection{Approximation in irregular domain}
We extend our analysis to general continuous functions defined on any irregular bounded set in $\R^d$. The key idea is to extend the target function to a hypercube while preserving the modulus of continuity.
\black{The extension of continuous (smooth) functions has been widely studied, e.g., \cite{smooth:extension} for smooth functions and \cite{continuous:extension} for continuous functions. For simplicity, we use Lemma $4.2$ of \cite{shijun2}. The proof can be found therein. }
For a general set $E\subseteq\R^d$, the modulus of continuity of $f\in C(E)$ is defined via
\begin{equation*}
	\omega_f^E(r)\coloneqq \sup\big\{|f(\bm{x})-f(\bm{y})|:\bm{x},\bm{y}\in E,\ \|\bm{x}-\bm{y}\|_2\le r\big\},\quad \tn{for any $r\ge 0$}.
\end{equation*} 
In particular, $\omega_f(\cdot)$ is short of $\omega_f^E(\cdot)$ in the case of $E=[0,1]^d$.
Then, Theorem~\ref{thm:main} can be generalized to $f\in C(E)$ for any bounded set $E\subseteq [-R,R]^d$ with $R>0$, as shown in the following theorem.

\begin{theorem}\label{thm:main:E}
	Given any bounded continuous function $f\in C(E)$ with $E\subseteq [-R,R]^d$ and $R>0$, for any $N\in \N^+$, $L\in \N^+$, and $p\in [1,\infty]$, 
	there exists a function $\phi$ implemented by a ReLU network with width $C_1\max\big\{d\lfloor N^{1/d}\rfloor,\, N+2\big\}$
	and depth $11L+C_2$
	such that 
	\begin{equation*}
		\|f-\phi\|_{L^p(E)}\le 131(2R)^{d/p} \sqrt{d}\,\omega_f^E\Big(2R\big(N^2L^2 \log_3(N+2)\big)^{-1/d}\Big),
	\end{equation*}
	where $C_1=16$ and $C_2=18$ if $p\in [1,\infty)$;  $C_1=3^{d+3}$ and $C_2=18+2d$ if $p=\infty$.
\end{theorem}
\begin{proof}
		Given any bounded continuous function $f\in C(E) $, by Lemma $4.2$ of \cite{shijun2} via setting $S=[-R,R]^d$, there exists $g\in C([-R,R]^d)$ such that 
	\begin{itemize}
		\item $g(\bmx)=f(\bmx)$ for any $\bmx\in E\subseteq S= [-R,R]^d$;
		\item $\omega_g^{S}(r)=\omega_f^{E}(r)$ for any $r\ge 0$.
	\end{itemize} 
	
	Define 
	\[\tildeg(\bmx)\coloneqq  g(2R\bmx-R),\quad \tn{for any $\bmx\in [0,1]^d$.}\] 
	By applying Theorem~\ref{thm:main} to $\tildeg\in  C([0,1]^d)$, 
		there exists a function $\tildephi$ implemented by a ReLU network with width $C_1\max\big\{d\lfloor N^{1/d}\rfloor,\, N+2\big\}$
	and depth $11L+C_2$
	such that 
	\begin{equation*}
		\|\tildephi-\tildeg\|_{L^p([0,1]^d)}\le 131\sqrt{d}\,\omega_\tildeg\Big(\big(N^2L^2 \log_3(N+2)\big)^{-1/d}\Big),
	\end{equation*}
	where $C_1=16$ and $C_2=18$ if $p\in [1,\infty)$;  $C_1=3^{d+3}$ and $C_2=18+2d$ if $p=\infty$.
		
	Note that  $f(\bmx)=g(\bmx)=\tildeg(\tfrac{\bmx+R}{2R})$ for any $\bmx\in E\subseteq S=[-R,R]^d$ and
	\[\omega_\tildeg(r)=\omega_g^{S}(2Rr)=\omega_f^E(2Rr),\quad \tn{ for any $r\ge 0$}.\]
	Define $\phi(\bmx)\coloneqq\tildephi(\tfrac{\bmx+R}{2R})=\tildephi\circ \calL(\bmx)$ for any $\bmx\in\R^d$, where $\calL:\R^d\to\R^d$ is an affine linear map given by $\calL(\bmx)=\tfrac{\bmx+R}{2R}$. Clearly,
	$\phi$  can be implemented by a ReLU network with width $C_1\max\big\{d\lfloor N^{1/d}\rfloor,\, N+2\big\}$
	and depth $11L+C_2$,
	where $C_1=16$ and $C_2=18$ if $p\in [1,\infty)$;  $C_1=3^{d+3}$ and $C_2=18+2d$ if $p=\infty$.	
	Moreover, for any $\bmx\in E\subseteq S=[-R,R]^d$, we have $\tfrac{\bmx+R}{2R}\in [0,1]^d$, implying
	\begin{equation*}
		\begin{split}
			\|\phi-f\|_{L^p(E)}&=\|\phi-g\|_{L^p(E)}  
			= \|\tildephi\circ \calL-\tildeg\circ \calL\|_{L^p(E)}\\
			&\le \|\tildephi\circ \calL-\tildeg\circ \calL\|_{L^p([-R,R]^d)} 
					=(2R)^{d/p}\|\tildephi-\tildeg\|_{L^p([0,1]^d)}\\
			&\le 131(2R)^{d/p} \sqrt{d}\,\omega_\tildeg\Big(\big(N^2L^2 \log_3(N+2)\big)^{-1/d}\Big)\\
			&= 131(2R)^{d/p} \sqrt{d}\,\omega_f^E\Big(2R\big(N^2L^2 \log_3(N+2)\big)^{-1/d}\Big).
		\end{split}
	\end{equation*}
	With the discussion above, we have proved Theorem~\ref{thm:main:E}.
\end{proof}

%%%%%%%%%%%%%%%%%%%%%%%%%%%%%%%%%%%%%%%%%%%%%%%%%%%%%%%%%55
%%%%%%%%%%%%%%%%%%%%%%%%%%%%%%%%%%%%%%%%%%%%%%%%%%%%%%%%%55
\section{Proof of Theorem~\ref{thm:mainGap}}
\label{sec:proof:mainGap}

We will prove Theorem~\ref{thm:mainGap} in this section. We first present the key ideas in Section~\ref{sec:keyIdea}. The detailed proof is presented in Section~\ref{sec:detaied:proof:mainGap}, based on two propositions in Section~\ref{sec:keyIdea}, 
the proofs of which can be found in Section~\ref{sec:proof:props}.

\subsection{Key ideas of proving Theorem~\ref{thm:mainGap}}
\label{sec:keyIdea}

%We will show that an almost piecewise constant function $\phi$ implemented by a ReLU network is enough to achieve the desired approximation rate in Theorem~\ref{thm:main}. 
Given an arbitrary $f\in C([0,1]^d)$, our goal is to construct an almost piecewise constant function $\phi$ implemented by a ReLU network
to approximate $f$ well.
To this end, we introduce a piecewise constant function $f_p\approx f$ serving as an intermediate approximant in our construction in the sense that
\[
f\approx f_p \tn{ on $[0,1]^d$} \quad \tn{and}\quad f_p\approx \phi \tn{ on $[0,1]^d\backslash\Omega([0,1]^d,K,\delta)$} .
\]
The approximation in $f\approx f_p$ is a simple and standard technique in constructive approximation. 
The most technical part is to design  a  ReLU network with the desired width and depth to implement a function $\phi$ with $\phi\approx f_p$ outside $\Omega([0,1]^d,K,\delta)$. See Figure~\ref{fig:Qml} for an illustration.
The introduction of the trifling  region is to ease the construction of $\phi$, which is a continuous piecewise linear  function, to approximate the discontinuous function $f_p$ by removing the difficulty near discontinuous points, essentially smoothing $f_p$ by restricting the approximation domain in $[0,1]^d\backslash\Omega([0,1]^d,K,\delta)$.

\begin{figure}[!htp]
	\centering
	\includegraphics[width=0.872\textwidth]{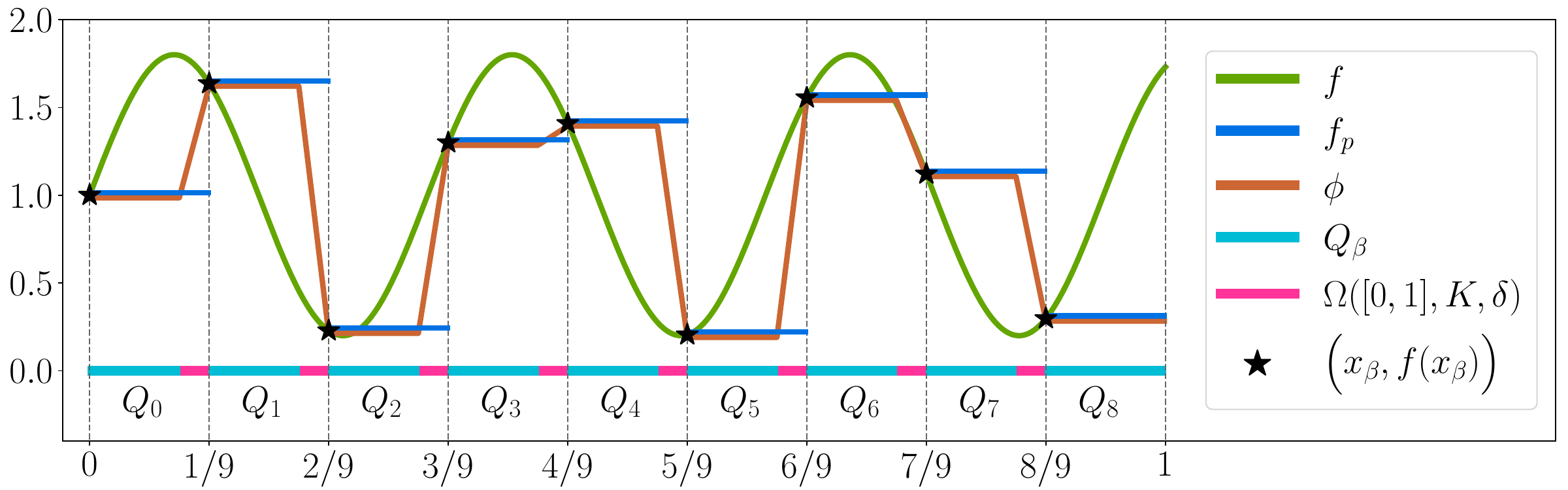}
	\caption{An illustration of $f$, $f_p$, $\phi$, $x_\beta$, $Q_{\beta}$, and the trifling  region $\Omega([0,1]^d,K,\delta)$ in the one-dimensional case for $\beta\in \{0,1,\cdots,K-1\}^d$, where $K=N^2L^2\lfloor \log_3(N+2)\rfloor$ and  $d=1$ with $N=1$ and $L=3$. $f$ is the target function; $f_p$ is the piecewise constant function approximating $f$; $\phi$ is a function, implemented by a ReLU network,  approximating $f$; and $x_\beta$ is a representative of $Q_\beta$. The measure of $\Omega([0,1]^d,K,\delta)$ can be arbitrarily small as we shall see in the proof of Theorem~\ref{thm:main}.}
	\label{fig:Qml}
\end{figure}

Now let us discuss the detailed steps of construction.
\begin{enumerate}[(i)]
	\item First, divide $[0,1]^d$  into a union of important regions $\{Q_\bmbeta\}_\bmbeta$ and the trifling  region $\Omega([0,1]^d,K,\delta)$, where each $Q_\bmbeta$ is associated with a representative $\bm{x}_\bmbeta\in Q_\bmbeta$ such that $f_p(\bm{x}_\bmbeta)=f(\bm{x}_\bmbeta)$ for each index vector $\bmbeta\in \{0,1,\cdots,K-1\}^d$, where $K=\calO((N^{2}L^{2}\ln N)^{1/d})$ is the partition number per dimension (see Figure~\ref{fig:Q+TR} for examples for $d=1$ and $d=2$). 
	
	\item Next, we design a vector function $\bmPhi_1(\bmx)$ constructed via \[\bmPhi_1(\bmx)=\big[\phi_1(x_1),\,\phi_1(x_2),\,\cdots,\,\phi_1(x_d)\big]^T\]
	 to project the whole cube $Q_\bmbeta$ to a $d$-dimensional index $\bmbeta$ for each $\bmbeta$, where each one-dimensional function $\phi_1$  is a step function implemented by a ReLU network. 
	 
	 \item The third step is to solve a point fitting problem. To be precise, we construct a function $\phi_2$ implemented by a ReLU network to map $\bmbeta\in \{0,1,\cdots,K-1\}^d$ approximately to $f_p(\bmx_\bmbeta)=f(\bmx_\bmbeta)$. Then $\phi_2\circ\bmPhi_1(\bmx)=\phi_2(\bmbeta)\approx f_p(\bmx_\bmbeta)=f(\bmx_\bmbeta)\approx f(\bmx)$ for any $\bmx\in Q_\bmbeta$ and each $\bmbeta$, implying 
	$\phi\coloneqq \phi_2\circ\bmPhi_1 \approx f_p\approx f $ on $ [0,1]^d\backslash\Omega([0,1]^d,K,\delta)$. We would like to point out that we only need to care about the values of $\phi_2$ at a set of points $\{0,1,\cdots,K-1\}^d$ in the construction of $\phi_2$ according to our design $\phi=\phi_2\circ\bmPhi_1$ as illustrated in Figure~\ref{fig:idea}. Therefore, it is not necessary to care about the values of $\phi_2$ sampled outside the set $\{0,1,\cdots,K-1\}^d$, which is a key point to ease the design of a ReLU network to implement $\phi_2$ as we shall see later. 
\end{enumerate}

\begin{figure}[!htp]        
	\centering
	\includegraphics[width=0.75\textwidth]{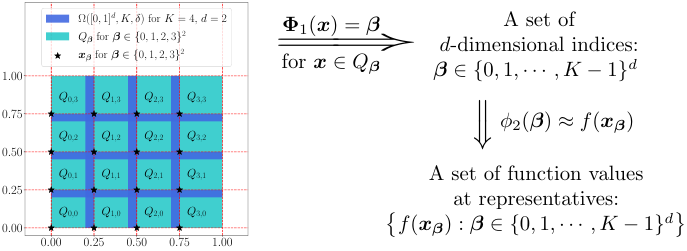}
	\caption{An illustration of the desired function $\phi=\phi_2\circ\bmPhi_1$. Note that $\phi\approx f$ on $[0,1]^d\backslash \Omega([0,1]^d,K,\delta)$, since $\phi(\bmx)=\phi_2\circ\bmPhi_1(\bmx)=\phi_2(\bmbeta)\approx f(\bmx_\bmbeta)\approx f(\bmx)$ for any $\bmx\in Q_\bmbeta$ and each $\bmbeta\in \{0,1,\cdots,K-1\}^d$.}
	\label{fig:idea}
\end{figure}

\black{
We remark that in Figure~\ref{fig:idea}, we have 
\begin{equation*}
    \phi(\bmx)=\phi_2\circ\bmPhi_1(\bmx)=\phi_2(\bmbeta)\mathop{\approx}\limits^{\scrE_1} f(\bmx_\bmbeta)\mathop{\approx}\limits^{\scrE_2} f(\bmx)
\end{equation*}
for any $\bmx\in Q_\bmbeta$ and each $\bmbeta\in \{0,1,\cdots,K-1\}^d$. Thus, $\phi-f$ is bounded by $\scrE_1+\scrE_2$ outside the trifling region. Observe that $\scrE_2$ is bounded by $\omega_f(\sqrt{d}/K)$. As we shall see later in Section~\ref{sec:detaied:proof:mainGap}, $\scrE_1$ can also be bounded by $\omega_f(\sqrt{d}/K)$ by applying Proposition~\ref{prop:pointFitting}. Hence, $\phi-f$ is controlled by $2\omega_f(\sqrt{d}/K)$ outside the trifling region, which deduces the desired approximation error. 
}%%%%%%%%%%%%%%  blue

Finally, we discuss how to implement $\bmPhi_1$ and $\phi_2$ by deep ReLU networks with width $\calO(N)$ and depth $\calO(L)$ using two propositions as we shall prove in Sections~\ref{sec:proofProp1} and \ref{sec:proofProp2} later. %The construction of $\bmPhi_1$ is based on Proposition~\ref{prop:stepFunc}. 
We first show how to  construct a ReLU network with the desired width and depth  by Proposition~\ref{prop:stepFunc} to implement a one-dimensional step function $\phi_1$. Then $\bmPhi_1$ can be attained via defining \[\bmPhi_1(\bmx)\coloneqq\big[\phi_1(x_1),\,\phi_1(x_2),\,\cdots,\,\phi_1(x_d)\big]^T,\quad \tn{for any $\bmx=[x_1,x_2,\cdots,x_d]^T\in \R^d$.}\]

\begin{proposition}
	\label{prop:stepFunc}
	For any $N,L,d\in \N^+$ and $\delta\in(0, \tfrac{1}{3K}]$ with  
	\begin{equation*}
		K=\lfloor N^{1/d}\rfloor^2 \lfloor L^{1/d}\rfloor^2 \lfloor n^{1/d}\rfloor,\quad \tn{where} \  n=\lfloor \log_3(N+2) \rfloor,
	\end{equation*}
	there exists a one-dimensional function $\phi$ implemented by a ReLU network  with width $8\lfloor N^{1/d}\rfloor +3$ and depth $2\lfloor L^{1/d}\rfloor+5$ such that
	\begin{equation*}
	\phi(x)=k,\quad \tn{if $x\in [\tfrac{k}{K},\tfrac{k+1}{K}-\delta\cdot \one_{\{k\le K-2\}}]$,\quad for $k=0,1,\cdots,K-1$.}
	\end{equation*} 
\end{proposition}

The setting $K=\lfloor N^{1/d}\rfloor^2 \lfloor L^{1/d}\rfloor^2 \lfloor n^{1/d}\rfloor=\calO(N^{2/d}L^{2/d}n^{1/d})$ is not neat here, but it is very convenient for later use.
The construction of $\phi_2$ is a direct result of Proposition~\ref{prop:pointFitting} below, the proof of which relies on 
%both the idea illustrated in Figure~\ref{fig:wideToDeepIdea} and
the bit extraction technique in   \cite{Bartlett98almostlinear}. 

\begin{proposition}
	\label{prop:pointFitting}
	Given any $\varepsilon>0$ and  arbitrary $N,L,J\in \N^+$ with $J\le N^2L^2\lfloor \log_3(N+2) \rfloor$, assume   
	${y}_j\ge 0$ for $j=0,1,\cdots,J-1$ are samples  with 
	\[|y_{j}-y_{j-1}|\le \varepsilon,\quad \tn{for $j=1,2,\cdots,J-1.$}\] 
	Then there exists $\phi\in \NN(\NNinput=1\NNspace\NNwidth\le 16N+30\NNspace\NNdepth\le 6L+10\NNspace\NNoutput=1)$ such that
	\begin{enumerate}[(i)]
		\item $|\phi(j)-{y}_j|\le \varepsilon$ for $j=0,1,\cdots,J-1$.
		\item $0\le \phi(x)\le  \max\{y_{j}:j=0,1,\cdots,J-1\}$ for any $x\in\R$.
	\end{enumerate}
\end{proposition}

% With the above propositions ready, let us prove Theorem~\ref{thm:mainGap} in Section~\ref{sec:detaied:proof:mainGap}. 
%We further assume that $\omega_f(r)>0$ for any $r>0$, excluding a simple case when $f$ is a constant function.

%%%%%%%%%%%%%%%%%%%%%%%%%%%%%%%%%%%%%%%%%%%
%%%%%%%%%%%%%%%%%%%%%%%%%%%%%%%%%%%%%%%%
\black{
\subsection{Construction of final network}
\label{sec:sketch:proof:mainGap} 
We will  discuss the construction of the final network approximating the target function with the same setting as in Section~\ref{sec:keyIdea}.  There are two main parts: 1) Construct the final network architecture based on Propositions~\ref{prop:stepFunc} and \ref{prop:pointFitting}; 2) Implement the network architectures in Propositions~\ref{prop:stepFunc} and \ref{prop:pointFitting}.

%%%%%%%%%%%%%%%%%%%%%%%%%%%%%%%%%%%%%%%%
\subsubsection*{Final network architecture based on Propositions~\ref{prop:stepFunc} and \ref{prop:pointFitting}}
By the idea mentioned in Figure~\ref{fig:idea}, the final network architecture can be implemented as shown in  Figure~\ref{fig:final:net}.

\begin{figure}[!htp]        
	\centering
	\includegraphics[width=0.8\textwidth]{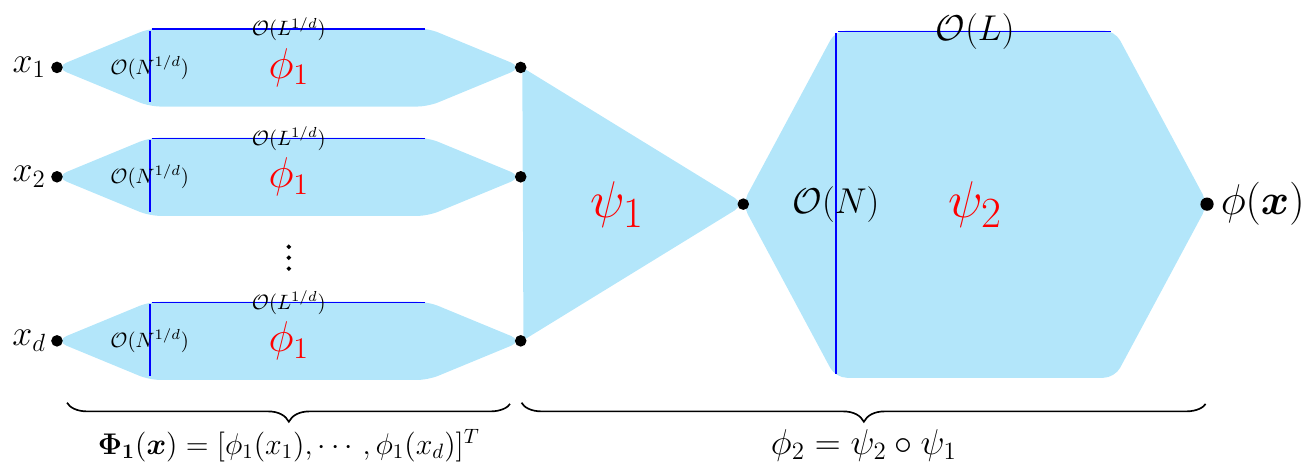}
	\caption{An illustration of the final network architecture with width $\max\{\calO(dN^{1/d}),\,\calO(N)\}$ and depth $\calO(L)$.  $\psi_1:\R^d\to \R$ is a linear function. $\phi_1$ and $\psi_2$ are implemented via Propositions~\ref{prop:stepFunc} and \ref{prop:pointFitting}, respectively.}
	\label{fig:final:net}
\end{figure}

Note that $\phi_1$ in Figure~\ref{fig:final:net} is a step function mapping $x\in [\tfrac{k}{K},\tfrac{k+1}{K}-\delta\cdot \one_{\{k\le K-2\}}]$  to $k$ for each $k\in \{0,1,\cdots,K-1\}$. It can be easily implemented via Proposition~\ref{prop:stepFunc}. Clearly, by defining $\bmPhi_1(\bmx)=\big[\phi_1(x_1),\phi_1(x_2),\cdots,\phi_1(x_d)\big]^T$, $\bmPhi_1$ maps $\bmx\in Q_\bmbeta$ to $\bmbeta$.

As shown in Figure~\ref{fig:idea}, we need to design a network to compute $\phi_2$ mapping $\bmbeta\in \{0,1,\cdots,K-1\}^d$ approximately to $f(\bmx_\bmbeta)$.  To this end, we first construct a \textbf{linear} function $\psi_1:\R^d\to \R$ mapping $\bmbeta\in \{0,1,\cdots,K-1\}^d$ to $\R$ for the purpose of  converting a $d$-dimensional point-fitting problem to a one-dimensional one, and then construct a network to compute $\psi_2$ with $\psi_2(\psi_1(\bmbeta))\approx f(\bmx_\bmbeta)$ via applying Proposition~\ref{prop:pointFitting}. Thus, we have $\phi_2(\bmbeta)\coloneqq \psi_2\circ\psi_1(\bmbeta)\approx f(\bmx_\bmbeta)$ as desired.

%%%%%%%%%%%%%%%%%%%%%%%%%%%%%%%%%%%%%%%%%%%%%%%
\subsubsection*{Network architectures in Propositions~\ref{prop:stepFunc} and \ref{prop:pointFitting}}
To prove Proposition~\ref{prop:stepFunc}, we need to construct a ReLU network with width $\calO(N^{1/d})$ and depth $\calO(L^{1/d})$  to compute a step function with $\calO\big((N^2L^2\ln N)^{1/d}\big)$ ``steps'' outside the trifling region. It is easy to construct a ReLU network with $\calO(W)$ parameters to compute a step function with $W$ ``steps'' outside a small region. As we shall see later in Section~\ref{sec:proofProp1}, the composition architecture of ReLU networks can help to implement step functions with much more ``steps''. Refer to Section~\ref{sec:proofProp1} for the detailed proof of Proposition~\ref{prop:stepFunc}.

Proposition~\ref{prop:pointFitting} essentially solves a point-fitting problem with $N^2L^2\lfloor \log_3(N+2)\rfloor$ points via a ReLU network with width $\calO(N)$ and depth $\calO(L)$. Set $M=N^2L$, $\hatL=L\lfloor \log_3(N+2)\rfloor$, and represent $j\in\{0,1,\cdots,M\hatL-1\}$ via $j=m\hatL+k$, where $m\in\{0,1,\cdots,M-1\}$ and $k\in\{0,1,\cdots,\hatL-1\}$. 

Define $a_{m,k}\coloneqq\lfloor y_{m,k}/\varepsilon\rfloor$ where $y_{m,k}=y_{m\hatL+k}$. Then 
\begin{equation*}
    |a_{m,k}\varepsilon -y_{m,k}|
=\big|\lfloor y_{m,k}/\varepsilon\rfloor \varepsilon- y_{m,k}\big|\le \varepsilon.
\end{equation*}
It suffices to prove $\phi(m,k)=a_{m,k}$. The assumption $|y_j-y_{j-1}|\le  \varepsilon$ implies that $b_{m,k}\coloneqq a_{m,k}-a_{m,k-1}\in \{-1,0,1\}$. Thus, there exist $c_{m,k}\in \{0,1\}$ and $d_{m,k}\in\{0,1\}$ such that $b_{m,k}=c_{m,k}-d_{m,k}$.

Note that 
\begin{equation*}
    a_{m,k}=a_{m,0}+\sum_{j=1}^k(a_{m,j}-a_{m,j-1})=a_{m,0}+\sum_{j=1}^k  b_{m,j}=a_{m,0}+\sum_{j=1}^k c_{m,j}-\sum_{j=1}^k d_{m,j}.
\end{equation*}

It is easy to construct a ReLU network with width $\calO(N)$ and depth $\calO(L)$ ($\calO(N^2L)$ parameters in total)  to compute $\phi_1$ such that $\phi_1(m)=a_{m,0}$ for each $m\in\{0,1,\cdots,M-1\}$ with $M=N^2L$. By the bit extraction technique in \cite{Bartlett98almostlinear}, one could construct $\phi_2,\phi_3\in \NN(\NNwidth\le \calO(N)\NNspace\NNdepth\le \calO(L))$
 such that $\phi_2(m,k)=\sum_{j=1}^k c_{m,j}$ and $\phi_3(m,k)=\sum_{j=1}^k d_{m,j}$. Thus, $\phi(m,k)\coloneqq \phi_1(m)+\phi_2(m,k)-\phi_3(m,k)=a_{m,k}$ as desired.
 
 In order to use the bit extraction technique (two types of bits $0$ or $1$) to solve the point-fitting problem, we essentially simplify the target as discussed above. That is, 
 \begin{equation*}
     \begin{split}
         \tn{non-negative number $y_{m,k}$} 
         &\longrightarrow \tn{integer $a_{m,k}=\lfloor y_{m,k}/\varepsilon\rfloor \mathop{\approx}^{\varepsilon} y_{m,k}$}\\
         &\longrightarrow 
         \tn{ $b_{m,k}=a_{m,k}-a_{m,k-1}\in\{-1,0,1\} $}\\
         &\longrightarrow \tn{ $b_{m,k}=c_{m,k}-d_{m,k}$ with $c_{m,k},d_{m,k}\in\{0,1\} $}.
     \end{split}
 \end{equation*}
 The detailed proof of Proposition~\ref{prop:pointFitting} can be found in Section~\ref{sec:proofProp2}.
 }%%%%%%%%% blue

%%%%%%%%%%%%%%%%%%%%%%%%%%%%%%%%%%%%%%%%%%%
%%%%%%%%%%%%%%%%%%%%%%%%%%%%%%%%%%%%%%%%
\subsection{Detailed proof}
\label{sec:detaied:proof:mainGap}

We essentially construct an almost piecewise constant function implemented by a ReLU network with width $\calO(N)$ and depth $\calO(L)$ to approximate $f$. \black{We may assume $f$ is not a constant function since it is a trivial case. Then $\omega_f(r)>0$ for any $r>0$.}
It is clear that $|f(\bm{x})-f(\bmzero)|\le \omega_f(\sqrt{d})$ for any ${\bm{x}}\in [0,1]^d$. Define $\tildef\coloneqq  f-f(\bmzero)+\omega_f(\sqrt{d})$, then $0\le \tildef (\bm{x}) \le 2\omega_f(\sqrt{d})$ for any ${\bm{x}}\in [0,1]^d$. 

Let $M=N^2L$,  $n=\lfloor \log_3(N+2) \rfloor$, $K=\lfloor N^{1/d}\rfloor^2 \lfloor L^{1/d}\rfloor^2 \lfloor n^{1/d} \rfloor$, and $\delta$ be an arbitrary number in $(0,\tfrac{1}{3K}]$.
The proof can be divided into four steps as follows:
\begin{enumerate}
	\item Normalize $f$ as $\tildef$, divide $[0,1]^d$ into a union of sub-cubes $\{Q_{\bm{\beta}}\}_{\bm{\beta}\in \{0,1,\cdots,K-1\}^d}$ and the trifling region $\Omega([0,1]^d,K,\delta)$, and denote $\bmx_\bmbeta$ as the vertex of $Q_\bmbeta$ with minimum $\|\cdot\|_1$ norm;
	
	\item Construct a sub-network to implement a vector function $\bmPhi_1$ projecting the whole cube $ Q_\bmbeta$ to the $d$-dimensional index $\bmbeta$ for each $\bmbeta$, i.e., $\bmPhi_1(\bmx)=\bmbeta$ for all $\bmx\in Q_\bmbeta$;
	
	\item Construct a sub-network to implement a  function $\phi_2$  mapping the index $\bmbeta$ approximately to $\tildef(\bmx_\bmbeta)$. This core step can be further divided into three sub-steps:	
	\begin{enumerate}
		\item[3.1.] Construct a sub-network to implement $\psi_1$ bijectively mapping the index set $\{0,1,\cdots,K-1\}^d$ to an auxiliary set $\mathcal{A}_1\subseteq \big\{\tfrac{j}{2K^d}:j=0,1,\cdots,2K^d\big\}$ defined later (see Figure~\ref{fig:g+A12} for an illustration);
		
		\item[3.2.] Determine a continuous piecewise linear function $g$ with  a set of breakpoints $\calA_1\cup\calA_2\cup\{1\}$  satisfying: 1) assign the values of $g$ at breakpoints in $\calA_1$ based on $\{\tildef(\bmx_\bmbeta)\}_\bmbeta$, i.e., $g\circ \psi_1(\bmbeta)=\tildef(\bmx_\bmbeta)$; 2) assign the values of $g$ at breakpoints in $\calA_2\cup\{1\}$ to reduce the variation of $g$  for applying Proposition~\ref{prop:pointFitting};
		%		such that $\tildef$ and $g\circ \phi_2 \circ \bmPhi_1$ have the same value at the elements of $\{\tfrac{k}{K}:k=0,1,\cdots,K-1\}^d$, i.e. $\tildef\approx g\circ \phi_2\circ \bmPhi_1$;
		
		\item[3.3.] Apply Proposition~\ref{prop:pointFitting} to construct a sub-network to implement a function $\psi_2$ approximating $g$ well on $\calA_1\cup\calA_2\cup\{1\}$. Then the desired function $\phi_2$ is given by  $\phi_2=\psi_2\circ\psi_1$ satisfying $\phi_2(\bmbeta)=\psi_2\circ\psi_1(\bmbeta)\approx g\circ\psi_1(\bmbeta)=\tildef(\bmx_\bmbeta)$;
		%		based on Proposition~\ref{prop:pointFitting} to implement a function  $\psi_3$ mapping $\mathcal{A}_1$ approximately to $\big\{\tildef(\bmx):\bmx\in \{\tfrac{k}{K}:k=0,1,\cdots,K-1\}^d\big\}$ such that $g\approx \phi_3$ on $\mathcal{A}_1$;
	\end{enumerate}
	
	\item Construct the final  network to implement the desired function $\phi$ such that $\phi(\bmx)= \phi_2 \circ \bmPhi_1(\bmx) +f(\bmzero)-\omega_f(\sqrt{d})\approx \tildef(\bmx_\bmbeta) +f(\bmzero)-\omega_f(\sqrt{d}) = f(\bmx_\bmbeta)\approx f(\bmx)$ for any $\bmx\in Q_\bmbeta$ and $\bmbeta\in \{0,1,\cdots,K-1\}^d$.
\end{enumerate}    

The details of these steps can be found below.
\mystep{1}{Divide $[0,1]^d$ into  $\{Q_{\bm{\beta}}\}_{\bm{\beta}\in \{0,1,\cdots,K-1\}^d}$ and  $\Omega([0,1]^d,K,\delta)$.}

Define  $\bmx_\bmbeta \coloneqq \bmbeta/K$ and 
\[
Q_{\bm{\beta}}\coloneqq\Big\{{\bm{x}}= [x_1,x_2,\cdots,x_d]^T\in [0,1]^d:x_i\in[\tfrac{\beta_i}{K},\tfrac{\beta_i+1}{K}-\delta\cdot \one_{\{\beta_i\le K-2\}}], \quad i=1,2,\cdots,d\Big\}
\]
for each $d$-dimensional index  ${\bm{\beta}}= [\beta_1,\beta_2,\cdots,\beta_d]^T\in \{0,1,\cdots,K-1\}^d$. Recall that $\Omega([0,1]^d,K,\delta)$ is the trifling region defined in Equation~\eqref{eq:triflingRegionDef}. Apparently, $\bmx_\bmbeta$ is the vertex of $Q_\bmbeta$ with minimum $\|\cdot\|_1$ norm and 
\[[0,1]^d= \big(\cup_{\bm{\beta}\in \{0,1,\cdots,K-1\}^d}Q_{\bm{\beta}}\big)\bigcup \Omega([0,1]^d,K,\delta).\]
See Figure~\ref{fig:Q+TR} for illustrations.

\begin{figure}[!htp]
	\centering
	\begin{minipage}{0.8\textwidth}
		\centering
		\begin{subfigure}[b]{0.35\textwidth}
			\centering
			\includegraphics[width=0.999\textwidth]{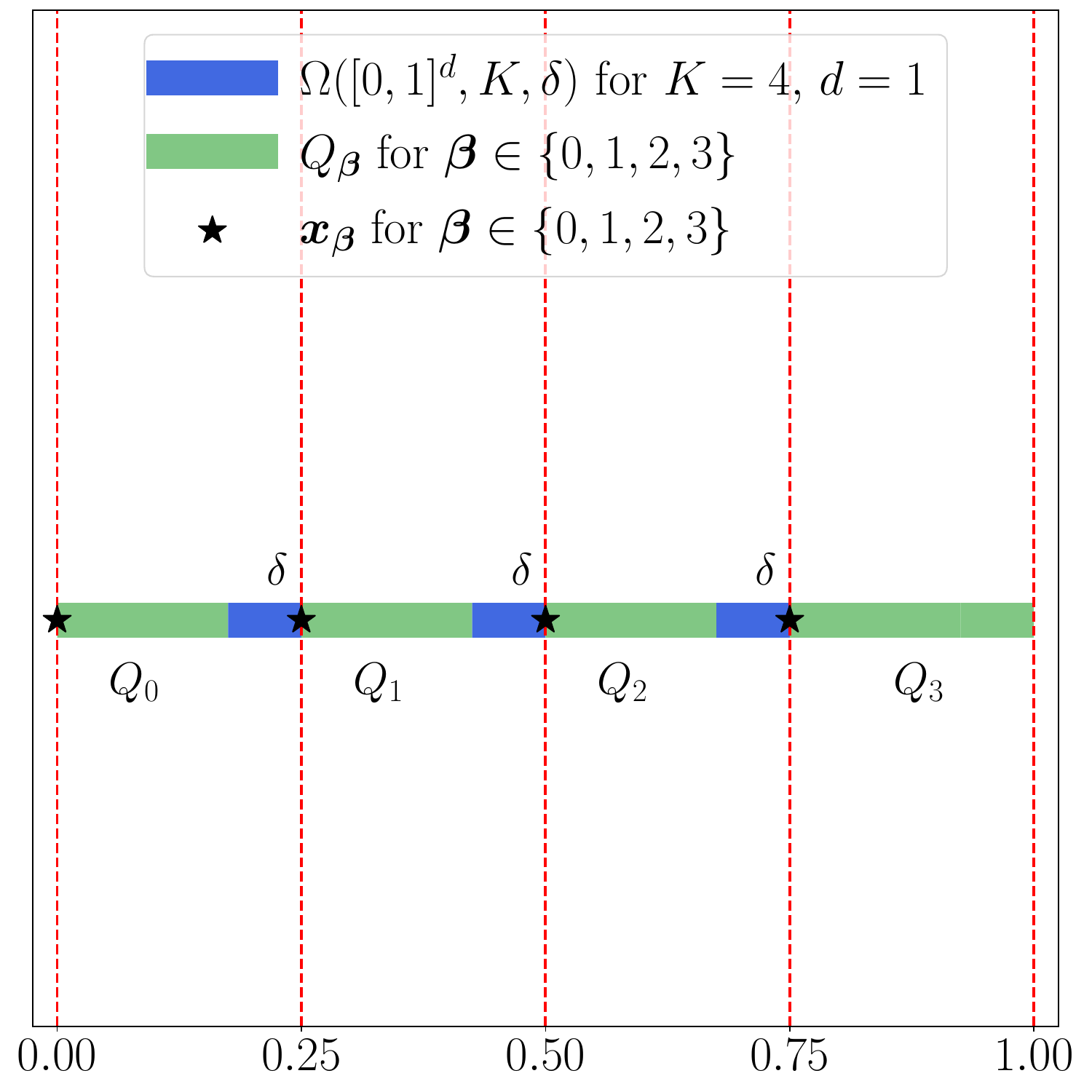}
			\subcaption{}
		\end{subfigure}
			\begin{minipage}{0.064\textwidth}
				\
			\end{minipage}
		\begin{subfigure}[b]{0.35\textwidth}
			\centering
			\includegraphics[width=0.999\textwidth]{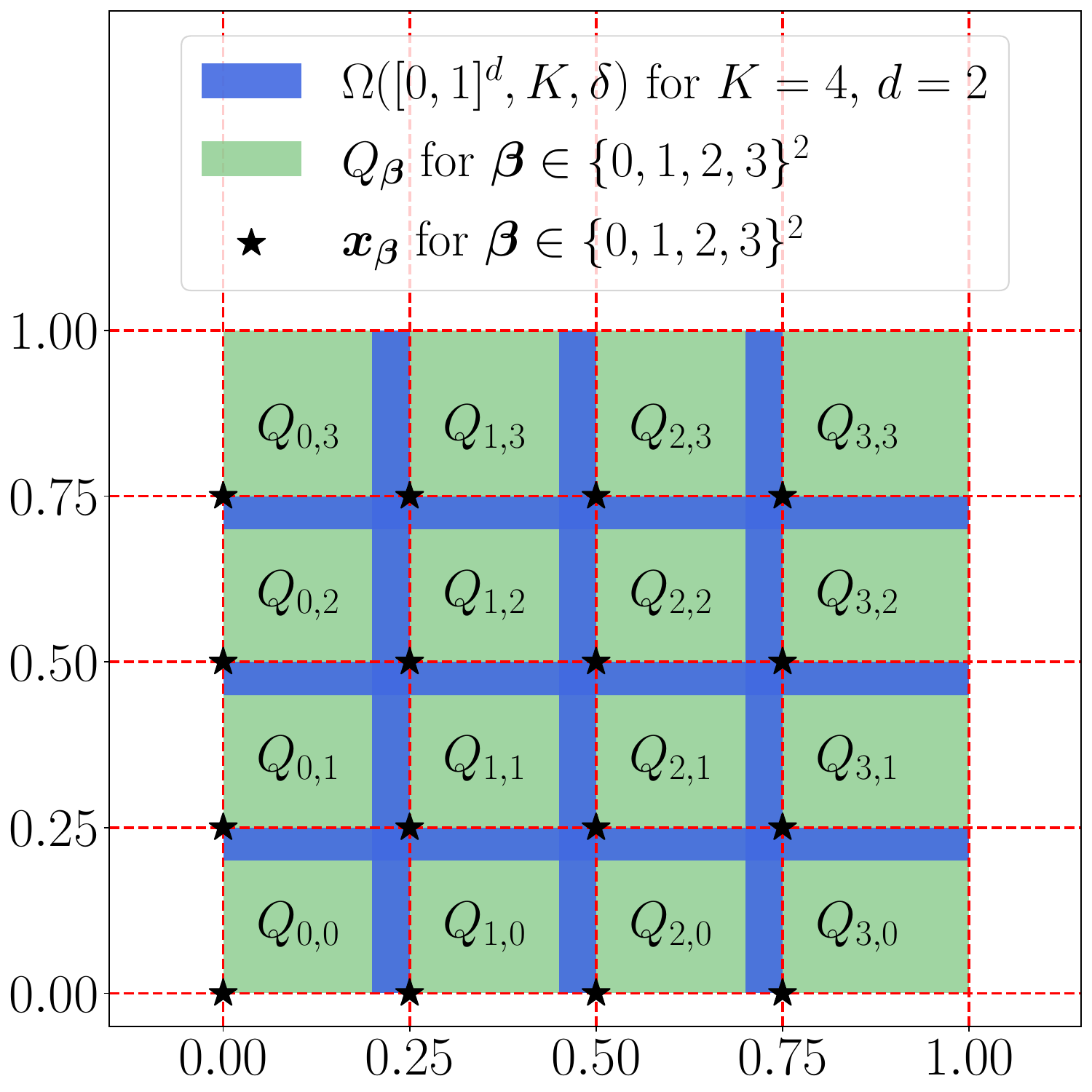}
			\subcaption{}
		\end{subfigure}
	\end{minipage}
	\caption{Illustrations of  $\Omega([0,1]^d,K,\delta)$,  $Q_\bmbeta$, and $\bmx_\bmbeta$ for $\bmbeta\in \{0,1,\cdots,K-1\}^d$. (a) $K=4$ and $d=1$. (b) $K=4$ and $d=2$. }
	\label{fig:Q+TR}
\end{figure}

\mystep{2}{Construct $\bmPhi_1$ mapping $\bmx\in Q_\bmbeta$ to  $\bmbeta$.}

By Proposition~\ref{prop:stepFunc}, there exists $\phi_1\in \NNF(
\NNwidth \le 8\lfloor N^{1/d}\rfloor +3\NNspace\NNdepth\le 2\lfloor L^{1/d}\rfloor+5)$ such that
\begin{equation*}
\phi_1(x)=\black{k},\quad \tn{if $x\in [\tfrac{k}{K},\tfrac{k+1}{K}-\delta\cdot \one_{\{k\le K-2\}}]$,\quad for $k=0,1,\cdots,K-1$.}
\end{equation*}    
It follows that $\phi_1(x_i)=\beta_i$ if $\bmx=[x_1,x_2,\cdots,x_d]^T\in Q_\bmbeta$ for each $\bmbeta=[\beta_1,\beta_2,\cdots,\beta_d]^T$.

By defining
\begin{equation*}
\bmPhi_1(\bmx)\coloneqq \big[\phi_1(x_1),\,\phi_1(x_2),\,\cdots,\,\phi_1(x_d)\big]^T,\quad \tn{for any } \bmx=[x_1,x_2,\cdots,x_d]^T\in \R^d,
\end{equation*}
we have $\bmPhi_1(\bmx)=\bmbeta$ if $\bmx\in Q_\bmbeta$ for each $\bmbeta\in \{0,1,\cdots,K-1\}^d$.

\mystep{3}{Construct $\phi_2$ mapping $\bmbeta$ approximately to $\tildef(\bmx_\bmbeta)$.}

The construction of the sub-network implementing $\phi_2$ is essentially based on Proposition~\ref{prop:pointFitting}. 
To meet the requirements of applying Proposition~\ref{prop:pointFitting}, we first define two auxiliary sets $\calA_1$ and $\calA_2$ as 
\begin{equation*}
\calA_1\coloneqq \big\{\tfrac{i}{K^{d-1}}+\tfrac{k}{2K^d}:i=0,1,\cdots,K^{d-1}\black{-1}\tn{\quad and \quad} k=0,1,\cdots,K-1\big\}
\end{equation*}
and 
\begin{equation*}
\calA_2\coloneqq \big\{\tfrac{i}{K^{d-1}}+\black{\tfrac{K+k}{2K^d}}:i=0,1,\cdots,K^{d-1}\black{-1}\tn{\quad and \quad}k=0,1,\cdots,K-1\big\}.
\end{equation*}
Clearly, $\calA_1\cup\calA_2\cup\{1\}=\{\tfrac{j}{2K^d}:j=0,1,\cdots,2K^d\}$ and $\calA_1\cap\calA_2=\emptyset$. See Figure~\ref{fig:Q+TR} for an illustration of $\calA_1$ and $\calA_2$. Next, we further divide this step into three sub-steps.

\mystep{3.1}{Construct $\psi_1$ bijectively mapping  $\{0,1,\cdots,K-1\}^d$ to $\mathcal{A}_1$.}

Inspired by the binary representation, we define
\begin{equation}
\psi_1(\bmx)\coloneqq \frac{x_d}{2K^d}+\sum_{i=1}^{d-1}\frac{x_i}{K^i},\quad \tn{for any $\bmx=[x_1,x_2,\cdots,x_d]^T\in \R^d$.}
\end{equation}
Then $\psi_1$ is a linear function bijectively mapping the index set $\{0,1,\cdots,K-1\}^d$ to
\begin{equation*}
\begin{split}
&\quad \Big\{\tfrac{\beta_d}{2K^d}+\sum_{i=1}^{d-1}\tfrac{\beta_i}{K^i}:\bm{\beta}\in \{0,1,\cdots,K-1\}^d\Big\}\\
&=\big\{\tfrac{i}{K^{d-1}}+\tfrac{k}{2K^d}:i=0,1,\cdots,K^{d-1}\black{-1}\tn{\quad and\quad } k=0,1,\cdots,K-1\big\}=\calA_1.
\end{split}
\end{equation*}

\mystep{3.2}{Construct $g$ to satisfy $g\circ\psi_1(\bmbeta)=\tildef(\bmx_\bmbeta)$ and to meet the requirements of applying Proposition~\ref{prop:pointFitting}.}

Let $g:[0,1]\to \R$ be a continuous piecewise linear function with a set of breakpoints $\left\{\tfrac{j}{2K^d}: j=0,1,\cdots,2K^d\right\}=\calA_1\cup\calA_2\cup\{1\}$ and the values of $g$ at these breakpoints satisfy the following properties:

\begin{itemize}
	\item The values of $g$ at the breakpoints in $\mathcal{A}_1=\big\{\psi_1(\bmbeta):\bm{\beta}\in \{0,1,\cdots,K-1\}^d\big\}$ are set as
	\begin{equation}
	\label{eq:ftog}
	g(\psi_1(\bmbeta))=\tildef(\bmx_\bmbeta),\quad \tn{for any  $\bm{\beta}\in \{0,1,\cdots,K-1\}^d$;}
	\end{equation} 
	
	\item At the breakpoint $1$, let $g(1)=\tildef(\bm{1})$, where $\bm{1}=[1,1,\cdots,1]^T\in \R^d$;
	
	\item 
	The values of $g$ at the breakpoints in $\mathcal{A}_2$ are assigned to reduce the variation of $g$, which is a requirement of applying Proposition~\ref{prop:pointFitting}. Note that 
	\begin{equation*}
	\big\{\tfrac{i}{K^{d-1}}-\tfrac{K+1}{2K^d},\ \tfrac{i}{K^{d-1}}\big\}\subseteq\calA_1\cup\{1\},\quad \tn{for $i=1,2,\cdots,K^{d-1}$,}
	\end{equation*}
	implying the values of $g$ at $\tfrac{i}{K^{d-1}}-\tfrac{K+1}{2K^d}$ and $\tfrac{i}{K^{d-1}}$ have been assigned for $i=1,2,\cdots,K^{d-1}$. Thus, the values of $g$ at the breakpoints in $\calA_2$ can be successfully assigned by letting $g$ linear on each interval 
	$[\tfrac{i}{K^{d-1}}-\tfrac{K+1}{2K^d},\, \tfrac{i}{K^{d-1}}]$ for $i=1,2,\cdots,K^{d-1}$, since $\calA_2\subseteq \bigcup_{i=1}^{K^{d-1}}[\tfrac{i}{K^{d-1}}-\tfrac{K+1}{2K^d},\, \tfrac{i}{K^{d-1}}]$. See Figure~\ref{fig:g+A12} for an illustration.
\end{itemize}
\begin{figure}
	\centering
	\includegraphics[width=0.8\textwidth]{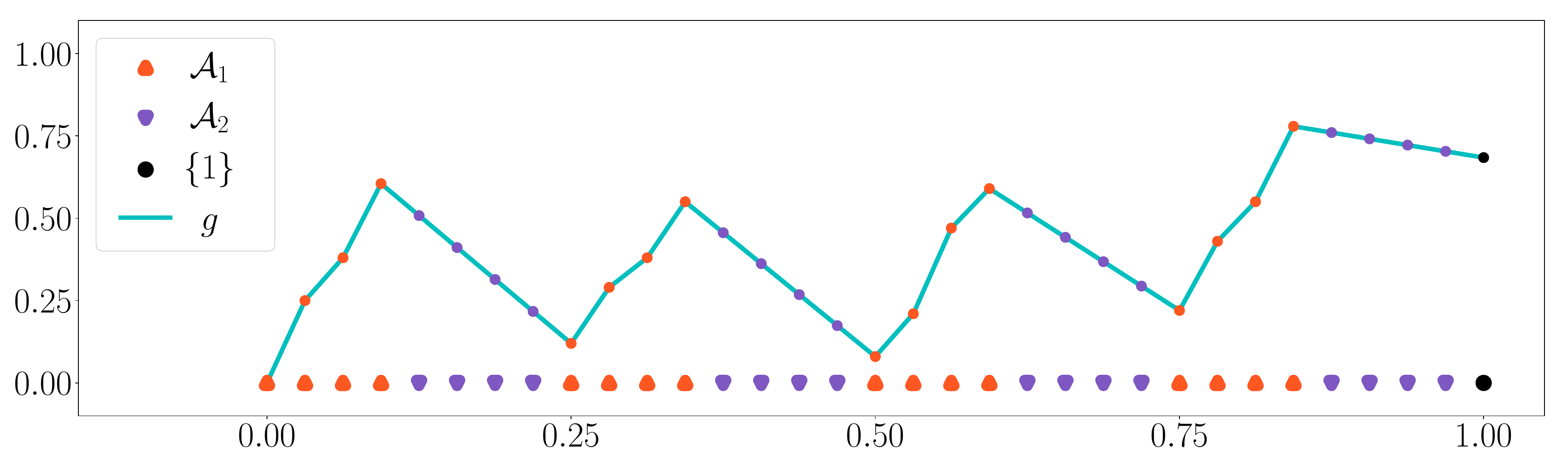}
	\caption[An illustration of $\mathcal{A}_1$, $\mathcal{A}_2$, $\{1\}$, and $g$ for $d=2$ and $K=4$]{An illustration of $\mathcal{A}_1$, $\mathcal{A}_2$, $\{1\}$, and $g$ for $d=2$ and $K=4$.}
	\label{fig:g+A12}
\end{figure}

Apparently, such a function $g$ exists (see Figure~\ref{fig:g+A12} for an example) and satisfies
\begin{equation*}
\label{eq:gErrorEstimation}
\left|g(\tfrac{j}{2K^d})-g(\tfrac{j-1}{2K^d})\right|\le \max\big\{\omega_f(\tfrac{1}{K}),\omega_f({\sqrt{d}})/K\big\}\le  \omega_f(\tfrac{\sqrt{d}}{K}),\quad \tn{for } j=1,2,\cdots,2K^d,
\end{equation*} 
and
\begin{equation*}
0\le g(\tfrac{j}{2K^d})\le 2\omega_f(\sqrt{d}), \quad \tn{for}\ j=0,1,\cdots,2K^d.
\end{equation*}

\mystep{3.3}{Construct $\psi_2$ approximating $g$ well on $\calA_1\cup\calA_2\cup\{1\}$.}

Note that 
\begin{equation*}
	2K^d=2 \big(\lfloor N^{1/d}\rfloor^2 \lfloor L^{1/d}\rfloor^2 \lfloor n^{1/d}\rfloor\big)^d\le 2\big(N^2L^2n\big)
	\le {N}^2\lceil \sqrt{2}L\rceil^2\lfloor \log_3(N+2)\rfloor.
\end{equation*}
By Proposition~\ref{prop:pointFitting} \black{(set $y_j=g(\tfrac{j}{2K^2})$ and $\varepsilon=\omega_f(\tfrac{\sqrt{d}}{K})>0$ therein)}, there exists 
\begin{equation*}
	\begin{split}
		\tildepsi_2  \in \NNF(\NNinput=1\NNspace\NNwidth\le 16N+30\NNspace\NNdepth\le 6\lceil \sqrt{2}L\rceil+10\NNspace\NNoutput=1)
	\end{split}
\end{equation*}
such that
\begin{equation*}
%\label{eq:phi1Minusg}
|\tildepsi_2(j)-g(\tfrac{j}{2K^d})|\le \omega_f(\tfrac{\sqrt{d}}{K}),\quad  \tn{for } j=0,1,\cdots,2K^d-1,
\end{equation*}
and 
\begin{equation*}
%\label{eq:phi3tUB}
\begin{split}
0\le \tildepsi_2(x) \le  \max\{g(\tfrac{j}{2K^d}):j=0,1,\cdots,2K^d-1\}\le 2\omega_f(\sqrt{d}), \quad \tn{for any $x\in\R$.}
\end{split}
\end{equation*}

By defining $\psi_2(x)\coloneqq \tildepsi_2(2K^dx)$ for any $x\in \R$, we have $\psi_2\in \NNF(\NNinput=1\NNspace\NNwidth\le 16N+30\NNspace\NNdepth\le 6\lceil \sqrt{2}L\rceil+10\NNspace\NNoutput=1)$,
\begin{equation}
\label{eq:phi3tUB}
\begin{split}
0\le \psi_2(x)=\tildepsi_2(2K^dx) \le 2\omega_f(\sqrt{d}), \quad \tn{for any } x\in\R,
\end{split}
\end{equation}
and 
\begin{equation}
\label{eq:phi1Minusg}
|\psi_2(\tfrac{j}{2K^d})-g(\tfrac{j}{2K^d})|=|\tildepsi_2(j)-g(\tfrac{j}{2K^d})|\le \omega_f(\tfrac{\sqrt{d}}{K}), \quad \tn{for $ j=0,1,\cdots,2K^d-1.$}
\end{equation}

%\mystep{3.4}{Construct $\phi_2$ mapping $\bmbeta$ approximately to $\tildef(\bmx_\bmbeta)$.}
Let us end Step $3$ by defining the desired function $\phi_2$ as 
$\phi_2\coloneqq \psi_2\circ \psi_1$. Note that $\psi_1:\R^d\to\R $ is a linear function and $\psi_2\in \NNF(\NNinput=1\NNspace\NNwidth\le 16N+30\NNspace\NNdepth\le 6\lceil \sqrt{2}L\rceil+10\NNspace\NNoutput=1)$. Thus, $\phi_2\in \NNF(\NNinput=1\NNspace\NNwidth\le 16N+30\NNspace\NNdepth\le 6\lceil \sqrt{2}L\rceil+10\NNspace\NNoutput=1)$.
By Equations~\eqref{eq:ftog} and \eqref{eq:phi1Minusg}, we have
\begin{equation}
\label{eq:phi2-tildef}
\begin{split}
|\phi_2(\bmbeta)-\tildef(\bmx_\bmbeta)|%&=\left|\psi_2\big(\sigma(\psi_1(\bmbeta))\big)-g(\psi_1(\bmbeta))\right|\\
=\left|\psi_2(\psi_1(\bmbeta))-g(\psi_1(\bmbeta))\right|\le \omega_f(\tfrac{\sqrt{d}}{K}),
\end{split}
\end{equation}
for any $\bmbeta\in\{0,1,\cdots,K-1\}^d$.
Equation~\eqref{eq:phi3tUB} and $\phi_2= \psi_2\circ \psi_1$ implies 
\begin{equation}
\label{eq:phi2tUB}
\begin{split}
0\le \phi_2(\bmx)\le 2\omega_f(\sqrt{d}), \quad \tn{for any } \bmx\in\R^d.
\end{split}
\end{equation}

\mystep{4}{Construct the final network to implement the desired function $\phi$.}

Define $\phi\coloneqq \phi_2\circ\bmPhi_1+f(\bmzero)-\omega_f(\sqrt{d})$. 
Since $\phi_1 \in\NNF(\NNwidth\le 8\lfloor N^{1/d}\rfloor+3\NNspace\NNdepth\le 2\lfloor L^{1/d}\rfloor+5])$, we have $\bmPhi_1\in \NNF(\NNinput=d\NNspace\NNwidth\le8d\lfloor N^{1/d}\rfloor+3d\NNspace\NNdepth\le 2L+5\NNspace\NNoutput=d)$. 
It follows from the fact $\lceil \sqrt{2}L\rceil\le \lceil \tfrac32L\rceil\le \tfrac32L+\tfrac12$ that  $6\lceil \sqrt{2}L\rceil+10\le9L+13$,  implying
\begin{equation*}
	\begin{split}
		\phi_2 &\in\NNF(\NNinput=1\NNspace\NNwidth\le 16N+30\NNspace\NNdepth\le 6\lceil \sqrt{2}L\rceil+10\NNspace\NNoutput=1)\\
		&\subseteq \NNF(\NNinput=1\NNspace\NNwidth\le 16N+30\NNspace\NNdepth\le 9L+13\NNspace\NNoutput=1).
	\end{split}
\end{equation*}
%Note that $ \NNF(\NNinput=d\NNspace\NNwidth\le 12N+8\NNspace\NNdepth\le 8L+9)$. 
Thus, $\phi=\phi_2 \circ \bmPhi_1+f(\bmzero)-\omega_f(\sqrt{d})$ is in 
\begin{equation*}
\begin{split}
\NN\big(\NNwidth\le\max\{8d\lfloor N^{1/d}\rfloor+3d, 16N+30\}\NNspace\NNdepth\le (2L+5)+(9L+13)= 11L+18\big).
\end{split}
\end{equation*}

Now let us estimate the approximation error.
Note that $f=\tildef+f(\bmzero)-\omega_f(\sqrt{d})$. By Equation~\eqref{eq:phi2-tildef}, for any $\bmx\in Q_\bmbeta$ and $\bmbeta\in \{0,1,\cdots,K-1\}^d$, we have
\begin{equation*}
\begin{split}
|f(\bmx)-\phi(\bmx)|&=|\tildef(\bmx)-\phi_2(\bmPhi_1(\bmx))|=|\tildef(\bmx)-\phi_2(\bmbeta)|\\
&\le |\tildef(\bmx)-\tildef(\bmx_\bmbeta)|+|\tildef(\bmx_\bmbeta)-\phi_2(\bmbeta)|\\
&\le \omega_f(\tfrac{\sqrt{d}}{K})+\omega_f(\tfrac{\sqrt{d}}{K})\le 2\omega_f\Big(64\sqrt{d}\big(N^2L^2\log_3(N+2)\big)^{-1/d}\Big),
\end{split}
\end{equation*}
where the last inequality comes from the fact 
\begin{equation*}
	\begin{split}
		K=\lfloor N^{1/d}\rfloor^2\lfloor L^{1/d}\rfloor^2 \lfloor n^{1/d}\rfloor
		&\ge \tfrac{N^{2/d}L^{2/d}n^{1/d}}{32}=\tfrac{N^{2/d}L^{2/d}\lfloor \log_3(N+2)\rfloor^{1/d}}{32}\ge \tfrac{(N^{2}L^{2}\log_3(N+2))^{1/d}}{64},
	\end{split}
\end{equation*}
for any $N,L\in \N^+$. Recall the fact $\omega_f(j\cdot r)\le j\cdot\omega_f(r)$ for any $j\in\N^+$ and $r\in [0,\infty)$. Therefore, for any $\bmx\in \bigcup_{\bm{\beta}\in \{0,1,\cdots,K-1\}^d} Q_\bmbeta\black{=} [0,1]^d\backslash \Omega([0,1]^d,K,\delta)$, we have
\begin{equation*}
\begin{split}
|f(\bmx)-\phi(\bmx)|
&\le 2\omega_f\Big(64\sqrt{d}\big(N^2L^2\log_3(N+2)\big)^{-1/d}\Big)\\
&\le 2\Big\lceil 64\sqrt{d}\Big\rceil \omega_f\Big(\big(N^2L^2\log_3(N+2)\big)^{-1/d}\Big)\\
&\le 130\sqrt{d}\,\omega_f\Big(\big(N^2L^2 \log_3(N+2)\big)^{-1/d}\Big). 
\end{split}
\end{equation*}

It remains to show the upper bound of $\phi$. By Equation~\eqref{eq:phi2tUB} and  $\phi= \phi_2\circ\bmPhi_1+f(\bmzero)-\omega_f(\sqrt{d})$, it holds that
$ \|\phi\|_{L^\infty(\R^d)}\le |f(\bmzero)|+ \omega_f(\sqrt{d})$. 
Thus, we finish the proof.

%%%%%%%%%%%%%%%%%%%%%%%%%%%%%%%%%%%%%%%%%%%%%%%%%%%%
%%%%%%%%%%%%%%%%%%%%%%%%%%%%%%%%%%%%%%%%%%%%%
\section{Proofs of propositions in Section~\ref{sec:keyIdea} }\label{sec:proof:props}

In this section, we will prove Propositions~\ref{prop:stepFunc} and \ref{prop:pointFitting}. We first introduce several basic results of ReLU networks. Next, we prove these two propositions  based on these basic results.
%%%%%%%%%%%%%%%%%%%%%%%%%%%%%%%%%%%%%%%%%%%
\subsection{Basic results of ReLU networks}\label{sec:basicReLU}

To simplify the proofs of two propositions in Section~\ref{sec:keyIdea}, we introduce three lemmas below, which are basic results of ReLU networks

\begin{lemma}
	\label{lem:widthPower}
	For any $N_1,N_2\in \N^+$, given $N_1(N_2+1)+1$ samples $(x_i,y_i)\in \R^2$ with $x_0<x_1<\cdots<x_{N_1(N_2+1)}$ and  $y_i\ge 0$ for $i=0,1,\cdots,N_1(N_2+1)$,
	there exists $\phi\in \NN(\NNinput=1\NNspace\NNwidthvec=[2N_1,2N_2+1]\NNspace\NNoutput=1)$ satisfying the following conditions.
	\begin{enumerate}[(i)]
		\item $\phi(x_i)=y_i$ for $i=0,1,\cdots,N_1(N_2+1)$.
		\item $\phi$ is linear on each interval $[x_{i-1},x_{i}]$ for $i\notin \{(N_2+1)j:j=1,2,\cdots,N_1\} $.
	\end{enumerate}
\end{lemma}

\begin{lemma}
	\label{lem:wideToDeep}
	Given any $N,L,d\in \N^+$, it holds that \begin{equation*}
		\begin{split}
			&\quad \, \NN(\NNinput=d\NNspace\NNwidthvec=[N,NL]\NNspace\NNoutput=1)\\
			&\subseteq \NN(\NNinput=d\NNspace\NNwidth\le 2N+2\NNspace \NNdepth\le L+1\NNspace\NNoutput=1).
		\end{split}
	\end{equation*}
\end{lemma}

\begin{lemma}\label{lem:n:cpl(n)}
	For any $n\in \N^+$, it holds that
	\begin{equation}\label{eq:cpl:subset:nn}
		\cpl\big(\R,n\big)\subseteq \NNF(\NNinput=1\NNspace \NNwidthvec=[n+1]\NNspace\NNoutput=1).
	\end{equation}
\end{lemma}

Lemma~\ref{lem:widthPower} is a part of Theorem~$3.2$ in \cite{shijun:thesis} or Lemma $2.2$ in \cite{shijun1}. 
Lemma~\ref{lem:widthPower} is Theorem~$3.1$ in \cite{shijun:thesis} or Lemma~$3.4$ in \cite{shijun1}.
It remains to prove Lemma~\ref{lem:n:cpl(n)}.
\begin{proof}[Proof of Lemma~\ref{lem:n:cpl(n)}]
We use the mathematical induction to prove Equation~\eqref{eq:cpl:subset:nn}. 
First, consider the case $n=1$.  Given any $f \in \cpl\big(\R,1\big)$, there exist $a_1,a_2,x_0\in\R$ such that
\begin{equation*}
	f(x)=\left\{
	\begin{array}{ll}
		a_1(x-x_0)+f(x_0), &\tn{if \  }  x\ge x_0,\\
		a_2(x_0-x)+f(x_0), &\tn{if \  }  x< x_0.\\
	\end{array}
\right.
\end{equation*}
Thus, $f(x)=a_1\sigma(x-x_0)+a_2\sigma(x_0-x)+f(x_0)$ for any $x\in \R$, implying 
\[f \in \NNF(\NNinput=1\NNspace \NNwidthvec=[2]\NNspace\NNoutput=1).\]
Thus, Equation~\eqref{eq:cpl:subset:nn} holds for $n=1$.

Now assume Equation~\eqref{eq:cpl:subset:nn} holds for $n=k\in \N^+$, we would like to show it is also  true for $n=k+1$.
Given any $f \in \cpl\big(\R,k+1\big)$, we may assume the biggest breakpoint of $f$ is $x_0$ since it is trivial for the case that $f$ has no breakpoint. Denote the slopes of the linear pieces left and right next to $x_0$ by $a_1$ and $a_2$, respectively.
Define 
\[\tildef(x)\coloneqq f(x)- (a_2-a_1)\sigma(x-x_0),\quad \tn{ for any $x\in\R$.}\]
 Then $\tildef$ 
 has at most $k$ breakpoints.
By the induction hypothesis, we have
\begin{equation*}
	\tildef \in \cpl\big(\R,k\big)\subseteq \NNF(\NNinput=1\NNspace \NNwidthvec=[k+1]\NNspace\NNoutput=1).
\end{equation*}
Thus, there exist   $w_{0,j},b_{0,j},w_{1,j},b_1$ for $j=1,2,\cdots,k+1$  such that
\begin{equation*}
	\tildef(x)=\sum_{j=1}^{k+1} w_{1,j}\sigma(w_{0,j}x+b_{0,j})+b_1,\quad \tn{for any $x\in\R$.}
\end{equation*}
 Therefore, for any $x\in\R$, we have
\begin{equation*}
		f(x)=(a_2-a_1)\sigma(x-x_0)+\tildef(x)
		= 	(a_2-a_1)\sigma(x-x_0)
				+\sum_{j=1}^{k+1} w_{1,j}\sigma(w_{0,j}x+b_{0,j})+b_1,
\end{equation*}
implying $f\in \NNF(\NNinput=1\NNspace \NNwidthvec=[k+2]\NNspace\NNoutput=1)$.
Thus, Equation~\eqref{eq:cpl:subset:nn} holds for $k+1$, which means we finish the induction process. So we complete the proof.
\end{proof}

%%%%%%%%%%%%%%%%%%%%%%%%%%%%%%%%%%%%%%%%%%%%%%%%%%%%
\subsection{Proof of Proposition~\ref{prop:stepFunc}}\label{sec:proofProp1}

 Now, let us present the detailed proof of Proposition~\ref{prop:stepFunc}. 
Denote $K=\tildeM\cdot\tildeL$, where $\tildeM=\lfloor N^{1/d}\rfloor ^2\lfloor L^{1/d}\rfloor$, $n=\lfloor \log_3(N+2)\rfloor$, and  $\tildeL=\lfloor L^{1/d}\rfloor \lfloor n^{1/d}\rfloor$.
Consider the sample set 
\[
\begin{split}
	\big\{(1,\tildeM-1),(2,0)\big\}
	&\bigcup
	\big\{(\tfrac{m}{\tildeM},m):m=0,1,\cdots,\tildeM-1\big\}\\
	&\bigcup \big\{(\tfrac{m+1}{\tildeM}-\delta,m):m=0,1,\cdots,\tildeM-2\big\}.
	\end{split}
\]
Its size is 
\[2\tildeM+1 = 2\lfloor N^{1/d}\rfloor^2 \lfloor L^{1/d}\rfloor +1 = \lfloor N^{1/d}\rfloor\cdot\Big(\big(2\lfloor N^{1/d}\rfloor \lfloor L^{1/d}\rfloor-1\big)+1\Big)+1.\] 
By Lemma~\ref{lem:widthPower} (set $N_1=\lfloor N^{1/d}\rfloor$ and $N_2=2\lfloor N^{1/d}\rfloor \lfloor L^{1/d}\rfloor-1$ therein), there exists \[\begin{split}
	\phi_1&\in \NNF\big(\NNwidthvec=\big[2\lfloor N^{1/d}\rfloor,2(2\lfloor N^{1/d}\rfloor \lfloor L^{1/d}\rfloor-1)+1\big]\big)\\
	&=\NNF\big(\NNwidthvec=\big[2\lfloor N^{1/d}\rfloor ,4\lfloor N^{1/d}\rfloor \lfloor L^{1/d}\rfloor-1\big]\big)
\end{split}\] such that
\begin{itemize}
	\item $\phi_1(\tfrac{\tildeM-1}{\tildeM})=\phi_1(1)=\tildeM-1$ and $\phi_1(\tfrac{m}{\tildeM})=\phi_1(\tfrac{m+1}{\tildeM}-\delta)=m$ for $m=0,1,\cdots,\tildeM-2$.
	\item $\phi_1$ is linear on $[\tfrac{\tildeM-1}{\tildeM},1]$ and each interval $[\tfrac{m}{\tildeM},\tfrac{m+1}{\tildeM}-\delta]$ for $m=0,1,\cdots,\tildeM-2$. 
\end{itemize}
Then, for $m=0,1,\cdots,\tildeM-1$, we have
\begin{equation}
	\label{eq:returnmStepFunc}
	\phi_1(x)=m, \quad \tn{for any} \ x\in [\tfrac{m}{\tildeM},\tfrac{m+1}{\tildeM}-\delta\cdot \one_{\{m\le \tildeM-2\}}].
\end{equation}

Now consider another sample set
\[\begin{split}
	\big\{(\tfrac{1}{\tildeM},\tildeL-1),(2,0)\big\}
	&\bigcup\big\{(\tfrac{\ell}{\tildeM\tildeL},\ell):\ell=0,1,\cdots,\tildeL-1\big\}\\
	&\bigcup \big\{(\tfrac{\ell+1}{\tildeM\tildeL}-\delta,\ell):\ell=0,1,\cdots,\tildeL-2\big\}.
\end{split}\] 
Its size is 
\[2\tildeL+1= 2 \lfloor L^{1/d}\rfloor  \lfloor n^{1/d}\rfloor + 1  = \lfloor n^{1/d}\rfloor\cdot\big((2\lfloor L^{1/d}\rfloor-1)+1\big)+1.\]
 By Lemma~\ref{lem:widthPower} (set $N_1=\lfloor n^{1/d}\rfloor$ and $N_2=2 \lfloor L^{1/d}\rfloor-1$ therein), there exists \[\begin{split}
	\phi_2 &\in \NNF\big(\NNwidthvec=\big[2 \lfloor n^{1/d}\rfloor,2(2\lfloor L^{1/d}\rfloor-1)+1\big]\big)\\
	&=\NNF\big(\NNwidthvec=\big[2\lfloor n^{1/d}\rfloor,4\lfloor L^{1/d}\rfloor-1\big]\big)
\end{split}\] such that
\begin{itemize}
	\item $\phi_2(\tfrac{\tildeL-1}{\tildeM\tildeL})=\phi_2(\tfrac{1}{\tildeM})=\tildeL-1$ and $\phi_2(\tfrac{\ell}{\tildeM\tildeL})=\phi_2(\tfrac{\ell+1}{\tildeM\tildeL}-\delta)=\ell$ for $\ell=0,1,\cdots,\tildeL-2$.
	\item $\phi_2$ is linear on $[\tfrac{\tildeL-1}{\tildeM\tildeL},\tfrac{1}{\tildeM}]$ and each interval $[\tfrac{\ell}{\tildeM\tildeL},\tfrac{\ell+1}{\tildeM\tildeL}-\delta]$ for $\ell=0,1,\cdots,\tildeL-2$. 
\end{itemize}
It follows that, for  $m=0,1,\cdots,\tildeM-1$ and $\ell=0,1,\cdots,\tildeL-1$,
\begin{equation}
	\label{eq:returnlStepFunc}
	\phi_2(x-\tfrac{m}{\tildeM})=\ell, \quad  \tn{for any}\ x\in [\tfrac{m\tildeL+\ell}{\tildeM\tildeL},\tfrac{m\tildeL+\ell+1}{\tildeM\tildeL}-\delta\cdot \one_{\{\ell\le \tildeL-2\}}].
\end{equation}
%	implying
%	\begin{equation}
%	\label{eq:returnlStepFunc}
%	\phi_2\big(x-\tfrac{1}{M}\phi_1(x)\big)=\phi_2(x-\tfrac{m}{M})=\ell, 
%	\end{equation}
%	for $\ x\in [\tfrac{mL+\ell}{ML},\tfrac{mL+\ell+1}{ML}-\delta\cdot \one_{\{m\le M-1\tn{ or }\ell<L-1\}}]=[\tfrac{mL+\ell}{ML},\tfrac{mL+\ell+1}{ML}-\delta\cdot \one_{\{mL+\ell \le ML-2\}}]$.

$K=\tildeM\cdot \tildeL$  implies any $k\in \{0,1,\cdots,K-1\}$ can be unique represented by $k=m\tildeL+\ell$ for $m\in \{0,1,\cdots,\tildeM-1\}$ and $\ell\in \{0,1,\cdots,\tildeL-1\}$. 
%	Thus, by Equations~\eqref{eq:returnmStepFunc} and \eqref{eq:returnlStepFunc}, we have
%	\begin{equation*}
%	\begin{split}
%	L\phi_1(x)+\phi_2\big(x-\tfrac{1}{M}\phi_1(x)\big)=mL+\ell=k,
%	\end{split}
%	\end{equation*}
%	for any $x\in [\tfrac{k}{K},\tfrac{k+1}{K}-\delta\cdot \one_{\{k\le K-2\}}]$.
Then the desired function $\phi$ can be implemented by a ReLU network shown in Figure~\ref{fig:stepFunc}. 

\begin{figure}[!htp]
	\centering
	\includegraphics[width=0.852\textwidth]{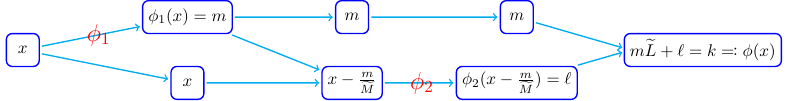}
	\caption{An illustration of the network architecture implementing $\phi$ based on Equations~\eqref{eq:returnmStepFunc} and \eqref{eq:returnlStepFunc} for $x\in [\tfrac{k}{K},\tfrac{k+1}{K}-\delta\cdot \one_{\{k\le K-2\}}]=[\tfrac{mL+\ell}{\tildeM\tildeL},\tfrac{mL+\ell+1}{\tildeM\tildeL}-\delta\cdot \one_{\{m\le \tildeM-2 \tn{ or }\ell\le \tildeL-2\}}]$, where $k=m\tildeL+\ell$ for $m=0,1,\cdots,\tildeM-1$ and $\ell=0,1,\cdots,\tildeL-1$. 
		%			``\textcolor{red}{$\phi_1$}'' and ``\textcolor{red}{$\phi_2$}'' near ``\textcolor{cyan}{$\longrightarrow$}'' represent the respective ReLU network implementing itself. We omit the activation function ReLU if the input of a neuron is non-negative.
	}
	\label{fig:stepFunc}
\end{figure}

Clearly, 
\begin{equation*}
	%	\label{eq:extractIndex1D}
	\phi(x)=k,\quad \tn{if $x\in [\tfrac{k}{K},\tfrac{k+1}{K}-\delta\cdot \one_{\{k\le K-2\}}]$,\quad for any $k\in\{0,1,\cdots,K-1\}.$}
\end{equation*}
By Lemma~\ref{lem:wideToDeep}, we have
\begin{equation*}
	\begin{split}
		\phi_1  &\in\NNF\big(\NNinput=1\NNspace\NNwidthvec=\big[2\lfloor N^{1/d}\rfloor ,4\lfloor N^{1/d}\rfloor \lfloor L^{1/d}\rfloor -1\big]\NNspace\NNoutput=1\big)\\
		&\subseteq\NNF\big(\NNinput=1\NNspace\NNwidth\le 8\lfloor N^{1/d}\rfloor +2\NNspace\NNdepth\le \lfloor L^{1/d}\rfloor +1\NNspace\NNoutput=1\big)
	\end{split}
\end{equation*}
and
\begin{equation*}
	\begin{split}
		\phi_2 &\in\NNF\big(\NNinput=1\NNspace\NNwidthvec=\big[2\lfloor n^{1/d}\rfloor, 4\lfloor L^{1/d}\rfloor -1\big]\NNspace\NNoutput=1\big)\\
		&\subseteq\NNF\big(\NNinput=1\NNspace\NNwidth\le 8\lfloor n^{1/d}\rfloor+2\NNspace\NNdepth\le \lfloor L^{1/d}\rfloor +1\NNspace\NNoutput=1\big).
	\end{split}
\end{equation*}
Recall that $n=\lfloor \log_3(N+2)\rfloor\le N$. It follows from Figure~\ref{fig:stepFunc} that $\phi$ can be implemented by a ReLU network with width
\[ \max\big\{8\lfloor N^{1/d}\rfloor+2+1,8\lfloor n^{1/d}\rfloor+2+1\big\}=8\lfloor N^{1/d}\rfloor+3\]
and depth
\[ (\lfloor L^{1/d}\rfloor+1)+2+(\lfloor L^{1/d}\rfloor+1)+1=2\lfloor L^{1/d}\rfloor+5.\]
So we finish the proof.

%%%%%%%%%%%%%%%%%%%%%%%%%%%%%%%%%%%%%%%%%%%%%%%%%%%%%%%%%%%%%
\subsection{Proof of Proposition~\ref{prop:pointFitting}}
\label{sec:proofProp2}
The proof of Proposition~\ref{prop:pointFitting} is based on the bit extraction technique in \cite{Bartlett98almostlinear,pmlr-v65-harvey17a}.
To simplify the proof, we first prove Lemmas~\ref{lem:BitsWidth}, \ref{lem:BitsWidthDepth},  \ref{lem:BitExtractionMulti}, and \ref{lem:pointFittingOld}, which serve as four important intermediate steps. Next, we will apply Lemma~\ref{lem:pointFittingOld} to prove Proposition~\ref{prop:pointFitting}.
In fact, we modify this technique to extract the sum of many bits rather than one bit and this modification can be summarized in Lemmas~\ref{lem:BitsWidth} and \ref{lem:BitsWidthDepth} below.

\begin{lemma}%[Bit Extraction]
	\label{lem:BitsWidth}
	For any $n\in \N^+$, there exists a function $\phi$ in
	\[\NNF\big(\NNinput=2\NNspace \NNwidth\le (n+1)2^{n+1}\NNspace\NNdepth\le 3  \NNspace\NNoutput=1\big)\]
	such that: Given any $\theta_j \in \{0,1\}$ for $j=1,2,\cdots,n$,  we have 
	\begin{equation*}
		\phi(\bin  0.\theta_1\theta_2\cdots\theta_{n} ,\, i)
		=\sum_{j=1}^{i} \theta_j,
		\quad \tn{for any $i\in \{0,1,2,\cdots,n\}$}.\,\footnote{By convention, $\sum_{j=n}^m a_j=0$ if $n>m$, no matter what $a_j$ is for each $j$.}
	\end{equation*}
%	where 
%	\begin{equation*}
%		z_{i,j}\coloneqq  \sigma\big(\theta_j+\sigma(i-j+1)-\sigma(i-j)-1\big),\quad \tn{for any $i\in\N^+$ and $j\in\{1,2,\cdots,n\}$}.
%	\end{equation*}
\end{lemma}
\begin{proof}
	Set $\theta=\bin 0.\theta_1\theta_2\cdots\theta_{n}$. Clearly, 
	\begin{equation*}
		\theta_j = {\lfloor 2^j \theta\rfloor}\big/{2} - \lfloor 2^{j-1} \theta\rfloor,\quad \tn{for any $j\in \{1,2,\cdots,n\}$.}
	\end{equation*}
We shall use a ReLU network to replace $\lfloor\cdot\rfloor$. 
Let $g\in \cpl(\R,2^{n+1}-2)$ be the function satisfying two conditions: 
\begin{itemize}
    \item $g$ matches set of samples
\begin{equation*}
	\bigcup_{k=0}^{2^n-1}\big\{(k,k),(k+1-\delta,k)\big\},\quad 
	\tn{where $\delta=2^{-(n+1)}$;}
\end{equation*} 
\item The breakpoint set of $g$ is 
\begin{equation*}
    \Big(\bigcup_{k=0}^{2^n-1}\big\{k,k+1-\delta \big\}\Big)\Big\backslash \Big(\{0\}\bigcup\{2^n-\delta\}\Big).
\end{equation*}
\end{itemize}
Then $g(x)=\lfloor x\rfloor$ for any $x\in \bigcup_{k=0}^{2^n-1}[k,k+1-\delta]$. Clearly, $\theta=\bin 0.\theta_1\theta_2\cdots\theta_{n}$ implies
\[2^j\theta\in \bigcup_{k=0}^{2^n-1}[k,k+1-\delta],\quad \tn{ for any $j\in \{0,1,2,\cdots,n\}$}.\]
Thus, 
\begin{equation}\label{eq:output:thetaj}
	\theta_j = {\lfloor 2^j \theta\rfloor }\big/{2} - \lfloor 2^{j-1} \theta\rfloor
	\,=\,
	{g(2^j \theta)}\big/{2} - g(2^{j-1} \theta),\quad \tn{for any $j\in \{1,2,\cdots,n\}$.}
\end{equation}

It is easy to design a ReLU network to  output $\theta_1,\theta_2,\cdots,\theta_n$ by Equation~\eqref{eq:output:thetaj} when using $\theta=\bin  0.\theta_1\theta_2\cdots\theta_n$ as the input. However, it is highly non-trivial to construct a ReLU network to output $\sum_{j=1}^{i}\theta_j$ with another input $i$, since many operations like multiplication and comparison are not allowed in designing ReLU networks.
Now let us establish a formula to represent $\sum_{j=1}^{i}\theta_j$ in a form of a ReLU network as follows.

Define $\calT(n)\coloneqq \sigma(n+1)-\sigma(n)=\left\{\begin{smallmatrix*}[l]
	1,& n\ge 0,\\ 0, & n<0
\end{smallmatrix*}\right.$ for any integer $n$. Then, by Equation~\eqref{eq:output:thetaj} and the fact $x_1x_2=\sigma(x_1+x_2-1)$ for any  $x_1,x_2\in \{0,1\}$, we have, for $i=0,1,2,\cdots,n$,
\begin{equation*}
	%	\label{eq:outputSumxFormula}
	\begin{split}
		\sum_{j=1}^{i}\theta_j=\sum_{j=1}^{n}\theta_j\cdot\calT(i-j)
		&=\sum_{j=1}^{n}\sigma\Big(\theta_j+\calT(i-j)-1\Big)\\
		&=\sum_{j=1}^{n}\sigma\Big(\theta_j+\sigma(i-j+1)-\sigma(i-j)-1\Big)\\
		&=\sum_{j=1}^{n}\sigma\Big({g(2^j \theta)}\big/{2} - g(2^{j-1} \theta)+\sigma(i-j+1)-\sigma(i-j)-1\Big).
	\end{split}
\end{equation*} 
Define
\begin{equation}\label{eq:zij:def}
	\begin{split}
			z_{i,j}&\coloneqq  \sigma\Big(g(2^j \theta)\big/2 - g(2^{j-1} \theta)+\sigma(i-j+1)-\sigma(i-j)-1\Big),
	\end{split}
\end{equation}
for any  $i,j\in\{0,1,2,\cdots,n\}$. Then the goal is to design $\phi$ satisfying 
\begin{equation}\label{eq:phi:theta:i}
	\phi(\theta,i)=\sum_{j=1}^{i}\theta_j=\sum_{j=1}^n z_{i,j},\quad  \tn{ for any $i\in\{0,1,2,\cdots,n\}$.}
\end{equation}
See Figure~\ref{fig:BitsWidth} for the network architecture implementing the desired function $\phi$.

\begin{figure}[!htp]
	\centering
	\includegraphics[width=0.84\textwidth]{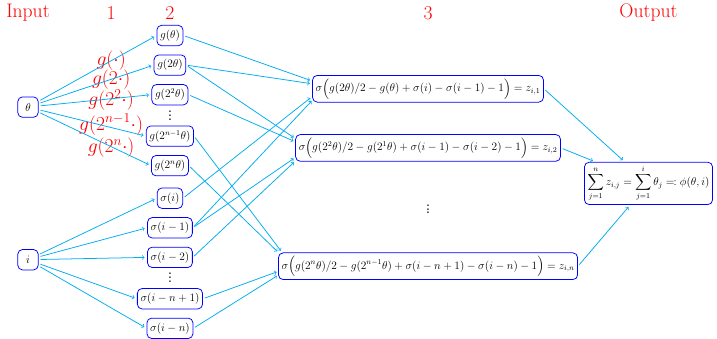}	
	\caption{An illustration of the  network implementing the desired function $\phi$ with the input $[\theta,i]^T=[\bin 0.\theta_1\theta_2\cdots\theta_{n},\, i]^T$ for any 
		$i\in \{0,1,2,\cdots,n\}$ and $\theta_1,\theta_2,\cdots,\theta_{n}\in\{0,1\}$. 
		$g(2^j\cdot)$ can be implemented by a one-hidden-layer network with width $2^{n+1}-1$ for each $j\in \{0,1,2,\cdots,n\}$.
		The red numbers above the architecture indicate the order of hidden layers. The network architecture is essentially determined by Equations~\eqref{eq:zij:def} and \eqref{eq:phi:theta:i}, which are valid no matter what $\theta_1,\theta_2,\cdots,\theta_{n}\in \{0,1\}$ are. Thus, the desired function $\phi$ is independent of $\theta_1,\theta_2,\cdots,\theta_{n}\in \{0,1\}$. We omit ReLU ($\sigma$) for a neuron if its output is non-negative without ReLU. Such a simplification is applied to similar figures in this paper.}
	\label{fig:BitsWidth}
\end{figure}

By Lemma~\ref{lem:n:cpl(n)}, we have
\begin{equation*}
	g\in \cpl(\R,2^{n+1}-2)\subseteq \NNF\big(\NNinput=1\NNspace \NNwidthvec=[2^{n+1}-1]\NNspace\NNoutput=1\big), 
\end{equation*}
implying 
\begin{equation*}
	g(2^j\cdot)\in \cpl(\R,2^{n+1}-2)\subseteq \NNF\big(\NNinput=1\NNspace \NNwidthvec=[2^{n+1}-1]\NNspace\NNoutput=1\big), 
\end{equation*}
for $j=0,1,2,\cdots,n$.
Clearly, the network in Figure~\ref{fig:BitsWidth} has width 
\[(n+1)(2^{n+1}-1)+(n+1)=(n+1)2^{n+1}\]
and depth $3$. So we finish the proof.
\end{proof}

\begin{lemma}%[Bit Extraction]
	\label{lem:BitsWidthDepth}
	For any $n,L\in \N^+$, there exists a function $\phi$ in
	\[\NNF\big(\NNinput=2\NNspace \NNwidth\le (n+3)2^{n+1}+4\NNspace\NNdepth\le 4L+2  \NNspace\NNoutput=1\big)\]
	such that: Given any $\theta_j \in \{0,1\}$ for $j=1,2,\cdots,Ln$,  we have 
	\begin{equation*}
		\phi(\bin  0.\theta_1\theta_2\cdots\theta_{Ln} ,\, k)
		=\sum_{j=1}^{k} \theta_j,
		\quad \tn{for any $k\in\{1,2,\cdots,Ln$\}}.
	\end{equation*}
\end{lemma}
\begin{proof}
%	content...$\xi_{1\to nL}$ $\bin 0.\theta_1\cdots\theta_{nL}$  $\xi_{(n-1)L+1\to nL}$ $\bin 0.\theta_{(n-1)L+1}\cdots\theta_{nL}$
Let $g_1\in \cpl(\R,2^{n+1}-2)$ be the function satisfying:
\begin{itemize}
    \item $g_1$ matches
the set of samples
\begin{equation*}
	\bigcup_{i=0}^{2^n-1}\big\{(i,i),(i+1-\delta,i)\big\},\quad 
	\tn{where $\delta=2^{-(Ln+1)}$.}
\end{equation*} 
\item The breakpoint set of $g_1$ is 
\begin{equation*}
	\Big(\bigcup_{i=0}^{2^n-1}\big\{(i,i),(i+1-\delta,i)\big\}\Big)\Big\backslash \Big(
	\{0\}\bigcup\{2^n-\delta\}\Big).
\end{equation*} 
\end{itemize}
Then $g_1(x)=\lfloor x\rfloor$ for any $x\in \bigcup_{i=0}^{2^n-1}[i,i+1-\delta]$. Note that 
\[2^n \cdot \bin 0. \theta_{\ell n+1}\cdots \theta_{Ln}\in \bigcup_{i=0}^{2^n-1}[i,i+1-\delta],\quad \tn{ for any $\ell\in \{0,1,\cdots,L-1\}$}.\]
Thus, for any $\ell\in \{0,1,\cdots,L-1\}$, we have
\begin{equation}\label{eq:theta:g1}
	\bin 0. \theta_{\ell n+1}\cdots\theta_{\ell n+n} 
	= \frac{\lfloor 2^n\cdot \bin 0. \theta_{\ell n+1}\cdots\theta_{Ln}\rfloor}{2^n} 
	=\frac{g_1(2^n\cdot \bin 0. \theta_{\ell n+1}\cdots\theta_{Ln})}{2^n}.
\end{equation}
Define $g_2(x)\coloneqq 2^n x-g_1(2^n x)$ for any $x\in\R$. Then $g_2 \in \cpl(\R,2^{n+1}-2)$ and 
\begin{equation}\label{eq:theta:g2}
\begin{split}
		&\quad 	\bin 0. \theta_{(\ell+1) n+1}\cdots\theta_{Ln} 
	= 2^n\Big(\bin 0. \theta_{\ell n+1}\cdots\theta_{Ln} - \bin 0. \theta_{\ell n+1}\cdots\theta_{\ell n+n}\Big)\\
	&= 2^n\Big(\bin 0. \theta_{\ell n+1}\cdots\theta_{Ln} - \frac{g_1(2^n\cdot \bin 0. \theta_{\ell n+1}\cdots\theta_{Ln})}{2^n}\Big)=g_2(\bin 0. \theta_{\ell n+1}\cdots\theta_{Ln}).
\end{split}
\end{equation}
By Lemma~\ref{lem:BitsWidth}, there exists 
\begin{equation*}
	\phi_1\in \NNF\big(\NNinput=2\NNspace \NNwidth\le (n+1)2^{n+1}\NNspace\NNdepth\le 3  \NNspace\NNoutput=1\big)
\end{equation*}
such that: For any $\xi_1,\xi_2,\cdots,\xi_n\in\{0,1\}$, we have
\begin{equation*}%\label{eq:theta:phi1}
	\phi_1(\bin  0.\xi_1\xi_2\cdots\xi_{n} ,\, i)
	=\sum_{j=1}^{i} \xi_j,
	\quad \tn{for $i=0,1,2,\cdots,n$}.
\end{equation*}
It follows that
\begin{equation}\label{eq:theta:phi1}
	\phi_1(\bin  0.\theta_{\ell n+1}\theta_{\ell n+2}\cdots\theta_{\ell n+n} ,\, i)
	=\sum_{j=1}^{i} \theta_{\ell n+j},
	\quad \tn{for $\ell=0,1,\cdots,L-1$ and $i=0,1,\cdots,n$}.
\end{equation}

Define $\phi_{2,\ell}(x)\coloneqq \min\{\sigma(x-\ell n),\,n\}$ for any $x\in\R$ and $\ell\in \{0,1,\cdots,L-1\}$. For any $k\in\{1,2,\cdots,Ln\}$, there exist $k_1\in \{0,1,\cdots,L-1\}$ and $k_2\in\{1,2,\cdots,n\}$ such that $k=k_1n+k_2$, implying 
\begin{equation}\label{eq:sum:theta:phi2}
	\begin{split}
		\sum_{i=1}^k\theta_i=\sum_{i=1}^{k_1n+k_2} \theta_i
		&=\sum_{\ell=0}^{k_1-1} \bigg(\sum_{j=1}^n \theta_{\ell n+j} \bigg)
		+ \sum_{\ell=k_1}^{k_1}\bigg(\sum_{j=1}^{k_2}\theta_{\ell n+j}\bigg)
		+ \sum_{\ell=k_1+1}^{L-1}\bigg(\sum_{j=1}^0 \theta_{\ell n+j}\bigg)\\
		&=\sum_{\ell=0}^{L-1}\bigg(\sum_{j=1}^{\min\{\sigma(k-\ell n),\,n\}} \theta_{\ell n+j}\bigg)
		=\sum_{\ell=0}^{L-1}\bigg(\sum_{j=1}^{\phi_{2,\ell}(k)} \theta_{\ell n+j}\bigg).
	\end{split}
\end{equation}
Then, the desired function $\phi$ can be implemented by the network architecture in Figure~\ref{fig:BitsWidthDepth}.

\begin{figure}[!htp]
	\centering
	\includegraphics[width=0.9999\textwidth]{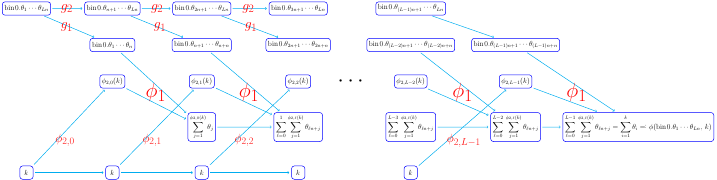}	
	\caption{An illustration of the  network implementing the desired function $\phi$ with the input $[\bin 0.\theta_1\theta_2\cdots\theta_{Ln},k]^T$ for any 
		$k\in \{1,2,\cdots,Ln\}$ and $\theta_1,\theta_2,\cdots,\theta_{Ln}\in\{0,1\}$.
%		$g_1$ and $g_2$ can be implemented by a one-hidden-layer network with width $2^{n+1}-1$.
	The network architecture is essentially determined by Equations~\eqref{eq:theta:g1}, \eqref{eq:theta:g2}, \eqref{eq:theta:phi1}, and \eqref{eq:sum:theta:phi2}, which are valid no matter what $\theta_1,\theta_2,\cdots,\theta_{Ln}\in \{0,1\}$ are. Thus, the desired function $\phi$ is independent of $\theta_1,\theta_2,\cdots,\theta_{Ln}\in \{0,1\}$. We omit ReLU ($\sigma$) for a neuron if its output is non-negative without ReLU. }
	\label{fig:BitsWidthDepth}
\end{figure}

By Lemma~\ref{lem:n:cpl(n)}, we have
\begin{equation*}
	g_1,g_2\in \cpl(\R,2^{n+1}-2)\subseteq \NNF\big(\NNinput=1\NNspace \NNwidthvec=[2^{n+1}-1]\NNspace\NNoutput=1\big). 
\end{equation*}

\begin{figure}[!htp]
	\centering
	\includegraphics[width=0.72\textwidth]{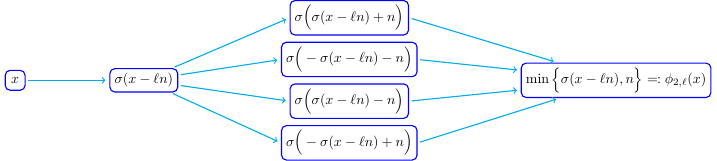}	
	\caption{An illustration of the  network implementing the desired function $\phi_{2,\ell}$ for each $\ell\in \{0,1,\cdots,L-1\}$, based on $\min\{y,n\}=\tfrac12\big(\sigma(y+n)-\sigma(-y-n)-\sigma(y-n)-\sigma(-y+n)\big)$. }
	\label{fig:phi2min}
\end{figure}

%It is clear that $\phi_2\in \NNF(NNwidthvec=[1,4])$ 
%implying 
%\begin{equation*}
%	g(2^j\cdot)\in \cpl(\R,2^{n+1}-2)\subseteq \NNF\big(\NNinput=1\NNspace \NNwidthvec=[2^{n+1}-1]\NNspace\NNoutput=1\big), 
%\end{equation*}
%for any $j=0,1,\cdots,n$.
Recall that $\phi_1\in \NNF\big(\NNwidth\le (n+1)2^{n+1}\NNspace\NNdepth\le 3 \big)$.  As shown in Figure~\ref{fig:phi2min},  $\phi_{2,\ell}(x)\in \NNF(\NNwidth\le 4\NNspace\NNdepth\le 2)$ for $\ell=0,1,\cdots,L-1$. 
Therefore, the network in Figure~\ref{fig:BitsWidthDepth} has width 
\begin{equation*}
(2^{n+1}-1)\ +\  (2^{n+1}-1) \ + \  (n+1)2^{n+1} \ +\  1  \ +\  4 \ +\   1= (n+3)2^{n+1}+4	
\end{equation*}
and depth 
\begin{equation*}
	2+L(1+3)=4L+2.
\end{equation*}
So we finish the proof.
\end{proof}

Next, we introduce Lemma~\ref{lem:BitExtractionMulti} to map indices to the partial sum of given bits.
\begin{lemma}
	\label{lem:BitExtractionMulti}
		Given any $N,L\in \N^+$ and arbitrary $\theta_{m, k }\in \{0,1\}$ for $m=0,1,\cdots,M-1$ and  $ k =0,1,\cdots,Ln-1$, where $M=N^2L$ and $n=\lfloor \log_3(N+2)\rfloor$, there exists 
		\[\phi\in\NNF\big(\NNinput=2\NNspace \NNwidth\le 6N+14\NNspace\NNdepth\le 5L+4  \NNspace\NNoutput=1\big)\]
		such that
	\[\phi(m, k )=\sum_{j=0}^{ k }\theta_{m,j}, \quad \tn{
	for $m=0,1,\cdots,M-1$\quad  and\quad   $ k =0,1,\cdots,Ln-1$.}\]
\end{lemma}
\begin{proof}
	Define \[y_m\coloneqq \bin    0.\theta_{m,0}\theta_{m,1}\cdots \theta_{m,Ln-1},\quad\tn{for $m=0,1,\cdots,M-1$.}\]  
	Consider the sample set $\{(m,y_m):m=0,1,\cdots,M\}$, whose cardinality is 
	\[M+1=N\big((NL-1)+1\big)+1.\]
	 By Lemma~\ref{lem:widthPower} (set $N_1=N$ and $N_2=NL-1$ therein), there exists 
	\[\begin{split}\phi_1&\in\NNF(\NNinput=1\NNspace\NNwidthvec=[2N,2(NL-1)+1]\NNspace\NNoutput=1)\\
	&= \NN(\NNinput=1\NNspace\NNwidthvec=[2N,2NL-1]\NNspace\NNoutput=1)\end{split}\]  such that 
	\[\phi_1(m)=y_m,\quad   \tn{for  $m=0,1,\cdots,M-1$.} \]
	
	By Lemma~\ref{lem:BitsWidthDepth}, there exists 
	\[\phi_2\in \NNF\big(\NNinput=2\NNspace \NNwidth\le (n+3)2^{n+1}+4\NNspace\NNdepth\le 4L+2  \NNspace\NNoutput=1\big)\] 
	such that, for any  $\xi_1,\xi_2,\cdots,\xi_{Ln}\in \{0,1\}$, we have
	\[\phi_2(\bin   0.\xi_1\xi_2\cdots \xi_{Ln},\, k )=\sum_{j=1}^{ k }\xi_j,\quad \tn{for  $ k =1,2,\cdots,Ln$.}\]
	It follows that,
	for any  $\xi_0,\xi_1,\cdots,\xi_{Ln-1}\in \{0,1\}$, we have
	\[\phi_2(\bin   0.\xi_0\xi_1\cdots \xi_{Ln-1},\, k +1)=\sum_{j=0}^{ k }\xi_j,\quad \tn{for  $ k =0,1,\cdots,Ln-1$.}\]
	
	Thus, for $m=0,1,\cdots,M-1$ and $ k =0,1,\cdots,Ln-1$, we have
	\[\phi_2(\phi_1(m), k +1)=\phi_2(y_m, k +1)=\phi_2(0.\theta_{m,0}\theta_{m,1}\cdots \theta_{m,Ln-1},\,  k +1)=\sum_{j=0}^{ k }\theta_{m,j}.\]
	
	Hence, the desired function $\phi$ can be implemented by the network shown in Figure~\ref{fig:BitsN^2L^2n}. By Lemma~\ref{lem:wideToDeep}, $\phi_1\in \NNF(\NNwidthvec=[2N,2NL-1])\subseteq \NNF(\NNwidth\le 4N+2\NNspace\NNdepth\le L+1)$. It holds that
	\begin{equation*}
		(n+3)2^{n+1}+4\le 6\cdot (3^{n})+2= 6\cdot( 3^{\lfloor \log_3(N+2)\rfloor})+2\le 6(N+2)+2=6N+14,
	\end{equation*} 
implying
\begin{equation*}
	\begin{split}
		\phi_2&\in \NNF\big(\NNinput=2\NNspace \NNwidth\le (n+3)2^{n+1}+4\NNspace\NNdepth\le 4L+2  \NNspace\NNoutput=1\big)\\
		&\subseteq \NNF\big(\NNinput=2\NNspace \NNwidth\le 6N+14\NNspace\NNdepth\le 4L+2  \NNspace\NNoutput=1\big).
	\end{split}
\end{equation*}
Therefore,
 the network in Figure~\ref{fig:BitsN^2L^2n} is with width $\max\{(4N+2)+1,6N+14\}=6N+14$ and depth $(4L+2)+1+(L+1)=5L+4$.
	\begin{figure}
		\centering
		\includegraphics[width=0.76\textwidth]{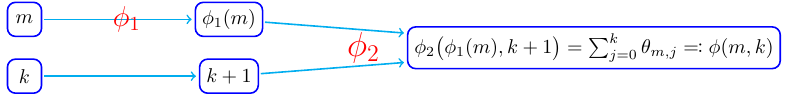}	
		\caption{An illustration of the  network implementing the desired function $\phi$ for $m=0,1,\cdots,M-1$ and $ k =0,1,\cdots,Ln-1$.  
%			``\textcolor{red}{$\phi_1$}'' and ``\textcolor{red}{$\phi_2$}'' near ``\textcolor{cyan}{$\longrightarrow$}'' represent the respective ReLU network implementing itself. We omit the activation function ReLU if the input of a neuron is non-negative.
		}
		\label{fig:BitsN^2L^2n}
	\end{figure}
	So we finish the proof.
\end{proof}

\vspace{5pt}
Next, we apply Lemma~\ref{lem:BitExtractionMulti} to prove Lemma~\ref{lem:pointFittingOld} below, which is a key intermediate conclusion to prove Proposition~\ref{prop:pointFitting}.

\vspace{5pt}
\begin{lemma}
	\label{lem:pointFittingOld}
	For any $\varepsilon>0$ and $N,L\in \N^+$, denote $M=N^2L$ and $n=\lfloor \log_3(N+2)\rfloor$. Assume   $y_{m, k }\ge 0$ for $m=0,1,\cdots,M-1 \tn{ and }  k =0,1,\cdots,Ln-1$ are samples with   
	\[|y_{m, k }-y_{m, k -1}|\le \varepsilon,\quad  \tn{for $m=0,1,\cdots,M-1\tn{\quad and\quad }  k =1,2,\cdots,Ln-1$.}\]
	Then there exists
	$\phi\in \NN(\NNinput=2\NNspace\NNwidth\le 16N+30\NNspace\NNdepth\le 5L+7\NNspace\NNoutput=1)$ such that 
	\begin{enumerate}[(i)]
		\item $|\phi(m, k )-y_{m, k }|\le \varepsilon$ for $m=0,1,\cdots,M-1$ and $ k =0,1,\cdots,Ln-1$;
		
		\item $0\le \phi(x_1,x_2)\le  \max\{y_{m, k }:m=0,1,\cdots,M-1 \tn{\quad and\quad } k =0,1,\cdots,Ln-1\}$  for any $x_1,x_2\in \R$.
	\end{enumerate}    
\end{lemma}
\begin{proof}
	Define 
\begin{equation*}
a_{m, k }\coloneqq \lfloor y_{m, k }/\varepsilon\rfloor 
,\quad \tn{for $m=0,1,\cdots,M-1 \tn{\quad and\quad }  k =0,1,\cdots,Ln-1$.}
\end{equation*}
We will construct a function implemented by a ReLU network to map the index $(m, k )$ to $a_{m, k }\varepsilon$ for $m=0,1,\cdots,M-1 \tn{ and }  k =0,1,\cdots,Ln-1$. 

Define $b_{m,0}\coloneqq 0$ and 
$b_{m, k }\coloneqq a_{m, k }-a_{m, k -1}$ for $m=0,1,\cdots,M-1 \tn{ and }  k =1,2,\cdots,Ln-1$. 
Since $|y_{m, k }-y_{m, k -1}|\le \varepsilon$ for all $m$ and $ k $, we have $b_{m, k }\in \{-1,0,1\}$. Hence, there exist $c_{m, k }\in\{0,1\}$ and $d_{m, k }\in\{0,1\}$ such that 
$b_{m, k }=c_{m, k }-d_{m, k }$, which implies
\begin{equation*}
\begin{split}
a_{m, k }=a_{m,0}+\sum_{i=1}^{ k }(a_{m,i}-a_{m,i-1})
&=a_{m,0}+\sum_{i=1}^{ k }b_{m,i}
=a_{m,0}+\sum_{i=0}^{ k }b_{m,i}\\
&=a_{m,0}+\sum_{i=0}^{ k }c_{m,i}-\sum_{i=0}^{ k }d_{m,i},
\end{split}
\end{equation*}
for $m=0,1,\cdots,M-1 \tn{ and }  k =0,1,\cdots,Ln-1$.

Consider the sample set 
\[\big\{(m,a_{m,0}):m=0,1,\cdots,M-1\big\}\bigcup \{(M,0)\}.\]
 Its size is $M+1=N\cdot\big((NL-1)+1\big)+1$, by Lemma~\ref{lem:widthPower} (set $N_1=N$ and $N_2=NL-1$ therein), there exists \[\psi_{1}\in \NNF(\NNwidthvec=[2N,2(NL-1)+1])=\NNF(\NNwidthvec=[2N,2NL-1])\]
  such that 
\[\psi_{1}(m)=a_{m,0},\quad \tn{for $m=0,1,\cdots,M-1$.}\]

By Lemma~\ref{lem:BitExtractionMulti}, there exist $\psi_{2}, \psi_{3}\in \NNF(\NNwidth\le 6N+14\NNspace\NNdepth\le 5L+4)$ such that \[\psi_{2}(m, k )=\sum\limits_{i=0}^{ k }c_{m,i}
\quad\tn{and}\quad  
\psi_{3}(m, k )=\sum\limits_{i=0}^{ k }d_{m,i},\]
\tn{for $m=0,1,\cdots,M-1 \tn{ and }  k =0,1,\cdots,Ln-1$.}
Hence, it holds that 
\begin{equation}
\label{eq:returnaml}
a_{m, k }=a_{m,0}+\sum_{i=0}^{ k }c_{m,i}-\sum_{i=0}^{ k }d_{m,i}=\psi_{1}(m)+\psi_{2}(m, k )-\psi_{3}(m, k ),
\end{equation}
for $m=0,1,\cdots,M-1 \tn{ and }  k =0,1,\cdots,Ln-1$.

Define 
\begin{equation*}
y_{\tn{max}}\coloneqq\max\{ y_{m, k }: m=0,1,\cdots,M-1 \tn{\quad and\quad }  k =0,1,\cdots,Ln-1\}.
\end{equation*}
Then the desired function can be implemented by two sub-networks shown in Figure~\ref{fig:contFuncPointFitting}.

\begin{figure}[!htp]
	\centering
	\begin{subfigure}[c]{0.41\textwidth}
		\centering
		\includegraphics[width=0.98\textwidth]{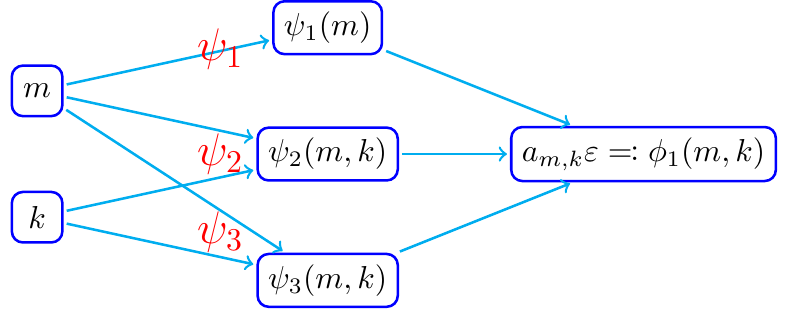}
		\subcaption{$\phi_1$}
	\end{subfigure}
	\begin{subfigure}[c]{0.5\textwidth}
		\centering
		\includegraphics[width=0.98\textwidth]{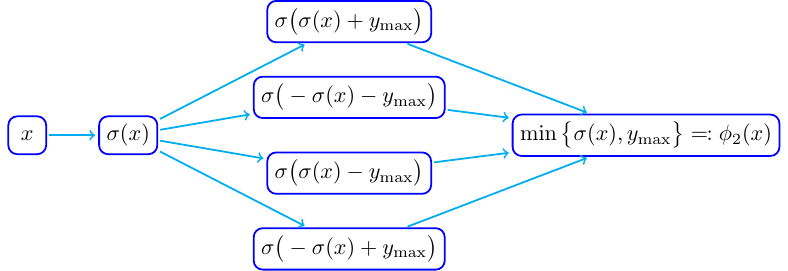}
		\subcaption{$\phi_2$}
	\end{subfigure}		
	\caption{Illustrations of two  sub-networks implementing the desired function $\phi=\phi_2\circ\phi_1$ for $m=0,1,\cdots,M-1$ and $ k =0,1,\cdots,Ln-1$, based on Equation~\eqref{eq:returnaml} and  the fact $\min\{x_1,x_2\}=\tfrac{x_1+x_2-|x_1-x_2|}{2}=\tfrac{\sigma(x_1+x_2)-\sigma(-x_1-x_2)-\sigma(x_1-x_2)-\sigma(-x_1+x_2)}{2}$. %The index $\beta\in \{0,1,\cdots,ML-1\}$ is unique represented by $\beta=mL+ k $ for $m=0,1,\cdots,M-1$ and $ k =0,1,\cdots,L-1$. 
%		$y_{\tn{max}}$ is given by $\max\{ y_{m, k }: m=0,1,\cdots,M-1 \tn{ and }  k =0,1,\cdots,L-1\}$. 
%		``\textcolor{red}{$\psi_1$}'',``\textcolor{red}{$\psi_2$}'', and ``\textcolor{red}{$\psi_3$}'' near ``\textcolor{cyan}{$\longrightarrow$}'' represent the respective ReLU network implementing itself. We omit the activation function ReLU if the input of a neuron is non-negative.
		}
	\label{fig:contFuncPointFitting}
\end{figure}

By Lemma~\ref{lem:wideToDeep},  
\begin{equation*}
	\begin{split}
		\psi_1 &\in \NNF(\NNinput=1\NNspace\NNwidthvec=[2N,2NL-1]\NNspace\NNoutput=1)\\
		&\subseteq \NNF(\NNinput=1\NNspace\NNwidth\le 4N+2\NNspace\NNdepth\le L+1\NNspace\NNoutput=1).
	\end{split}
\end{equation*}
Recall  that $\psi_{2}, \psi_{3}\in \NNF(\NNwidth\le 6N+14\NNspace\NNdepth\le 5L+4)$.	
Thus,   $\phi_1\in\NNF(\NNwidth\le (4N+2)+2(6N+14)= 16N+30\NNspace\NNdepth\le (5L+4)+1=5L+5)$ as shown in Figure~\ref{fig:contFuncPointFitting}.  And it is clear that $\phi_2\in \NNF(\NNwidth\le 4\NNspace\NNdepth\le 2)$, implying 
$\phi=\phi_2\circ\phi_1\in\NNF(\NNwidth\le 16N+30\NNspace\NNdepth\le (5L+5)+2=5L+7)$.

Clearly,  $0\le\phi(x_1,x_2)\le y_{\tn{max}}$ for any $x_1,x_2\in \R$, since $\phi(x_1,x_2)= \phi_2\circ\phi_1(x_1,x_2)=\max\{\sigma(\phi_1(x_1,x_2)),y_{\tn{max}}\}$. 

Note that $0\le  a_{m, k } \varepsilon= \lfloor y_{m, k }/\varepsilon\rfloor \varepsilon \le y_{\tn{max}}$. Then we have $\phi(m, k )=\phi_2\circ\phi_1(m, k )=\phi_2( a_{m, k }\varepsilon)=\max\{\sigma(a_{m, k }\varepsilon),y_{\tn{max}}\}= a_{m, k }\varepsilon$. Therefore,
\begin{equation*}
\begin{split}
|\phi(m, k )-y_{m, k }|=\left|a_{m, k }\varepsilon-y_{m, k }\right|=\big| \lfloor y_{m, k }/\varepsilon\rfloor\varepsilon-y_{m, k }\big|
\le \varepsilon,
\end{split}
\end{equation*}
for $m=0,1,\cdots,M-1$ and $ k =0,1,\cdots,Ln-1$. Hence, we finish the proof.
\end{proof}

Finally, we apply Lemma~\ref{lem:pointFittingOld} to prove Proposition~\ref{prop:pointFitting}.

\begin{proof}[Proof of Proposition~\ref{prop:pointFitting}]
	Denote $M=N^2L$, $n=\lfloor \log_3(N+2) \rfloor$, and $\hatL=Ln$. We may assume $J=MLn=M\hatL$ since we can set $y_{J-1}=y_{J}=y_{J+1}=\cdots=y_{M\hatL-1}$ if $J<M\hatL$. 

Consider the sample set
\[\big\{(m\hatL,m):m=0,1,\cdots,M\big\}\bigcup \big\{(m\hatL+\hatL-1,m):m=0,1,\cdots,M-1\big\}.\]
Its size is $2M+1=N\cdot\big((2NL-1)+1\big)+1$.
By Lemma~\ref{lem:widthPower} (set $N_1=N$ and $N_2=NL-1$ therein), there exists
\[\phi_1\in \NNF(\NNwidthvec=[2N,2(2NL-1)+1])=\NNF(\NNwidthvec=[2N,4NL-1])\] 
such that
\begin{itemize}
	\item $\phi_1(M\hatL)=M$ and $\phi_1(m\hatL)=\phi_1(m\hatL+\hatL-1)=m$ for $m=0,1,\cdots,M-1$.
	\item $\phi_1$ is linear on each interval $[m\hatL,m\hatL+\hatL-1]$ for $m=0,1,\cdots,M-1$.
\end{itemize}
It follows that 
\begin{equation}
\label{eq:phi1returnml}
\phi_1(j)=m,\quad \tn{and}\quad  j-\hatL\phi_1(j)= k ,\quad \tn{where $j=m\hatL+ k $},
\end{equation}
for $m=0,1,\cdots,M-1$ and $ k =0,1,\cdots,\hatL-1$.

Since $J=M\hatL$, any  $j\in \{0,1,\cdots,J-1\}$ can be uniquely indexed as $j=m\hatL+ k $ for $m\in \{0,1,\cdots,M-1\}$ and  $ k \in \{0,1,\cdots,\hatL-1\}$.
So we can denote $y_j=y_{m\hatL+ k }$ as  $y_{m, k }$.
Then by Lemma~\ref{lem:pointFittingOld}, there exists $\phi_2\in \NNF(\NNwidth\le 16N+30\NNspace\NNdepth\le 5L+7)$ such that 
\begin{equation}
\label{eq:phi2Mniusyml}
|\phi_2(m, k )-y_{m, k }|\le \varepsilon,\quad \tn{for $m=0,1,\cdots,M-1 \tn{\quad and\quad } k =0,1,\cdots,\hatL-1$},
\end{equation}
and 
\begin{equation}
\label{eq:phi2UB}
0\le \phi_2(x_1,x_2)\le  y_{\tn{max}},\quad \tn{for any $x_1,x_2\in \R$,}
\end{equation}
where $y_{\tn{max}}\coloneqq\max\{y_{m, k }:m=0,1,\cdots,M-1 \tn{ and }  k =0,1,\cdots,\hatL-1\}=\max\{y_{j}:j=0,1,\cdots,J-1\}$.

\begin{figure}[!htp]
	\centering
	\includegraphics[width=0.8\textwidth]{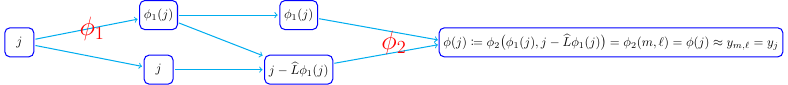}	
	\caption{An illustration of the ReLU network implementing the desired function $\phi$ based Equation~\eqref{eq:phi1returnml}. The index $j\in \{0,1,\cdots,M\hatL-1\}$ is unique represented by $j=mL+ k $ for $m\in\{0,1,\cdots,M-1\}$ and $ k \in\{0,1,\cdots,\hatL-1\}$. 
%		``\textcolor{red}{$\phi_1$}'' and ``\textcolor{red}{$\phi_2$}'' near ``\textcolor{cyan}{$\longrightarrow$}'' represent the respective ReLU network implementing itself. We omit the activation function ReLU if the input of a neuron is non-negative.
	}
	\label{fig:pointFittingThree}
\end{figure}

By Lemma~\ref{lem:wideToDeep},
\begin{equation*}
	\begin{split}
		 	\phi_1 &\in \NNF(\NNinput=1\NNspace\NNwidthvec=[2N,4NL-1]\NNspace\NNoutput=1)\\
		&\subseteq \NNF(\NNinput=1\NNspace\NNwidth\le 8N+2;\NNdepth\le L+1\NNspace\NNoutput=1).
	\end{split}
\end{equation*} 
Recall that $\phi_2\in \NNF(\NNwidth\le 16N+30\NNspace\NNdepth\le 5L+7)$. 
 So $\phi\in \NNF(\NNwidth\le 16N+30\NNspace\NNdepth\le (L+1)+2+(5L+7)= 6L+10)$ as shown in Figure~\ref{fig:pointFittingThree}.

Equation~\eqref{eq:phi2UB} implies 
\begin{equation*}
0\le \phi(x)\le  y_{\tn{max}},\quad \tn{for any $x\in \R$,}
\end{equation*}
since $\phi$ is given by $\phi(x)=\phi_2\big(\phi_1(x),x-\hatL\phi_1(x)\big)$.

Represent $j\in \{0,1,\cdots,M\hatL-1\}$ via $j=m\hatL+ k $ for $m=0,1,\cdots,M-1$ and $ k =0,1,\cdots,\hatL-1$.  Then, by Equation~\eqref{eq:phi2Mniusyml}, we have
\begin{equation*}
\begin{split}
|\phi(j)-y_j|=|\phi_2\big(\phi_1(j),j-\hatL\phi_1(j)\big)-y_j|=|\phi_2(m, k )-y_{m, k }|\le \varepsilon,
\end{split}
\end{equation*}
for any $j\in \{0,1,\cdots,M\hatL-1\}=\{0,1,\cdots,J-1\}$.
 So we finish the proof.
\end{proof}

We would like to remark that the key idea in the proof of Proposition~\ref{prop:pointFitting} is the bit extraction technique in Lemma~\ref{lem:BitsWidthDepth}, which allows us to store  $Ln$ bits in a binary number $\bin  0.\theta_1\theta_2\cdots \theta_{Ln}$ and extract each bit $\theta_i$. The extraction operator can be efficiently carried out via a deep ReLU neural network demonstrating the power of depth.

\section{Conclusion and future work}
\label{sec:conclusion}
This paper aims at a quantitative and optimal approximation rate for ReLU networks in terms of the width and depth  to approximate continuous functions. It is shown by construction that ReLU networks with width $\calO(N)$ and depth $\calO(L)$ can approximate an arbitrary continuous function $f$ on $[0,1]^d$ with an approximation rate $\calO\big(\,\omega_f\big((N^2L^2\ln N)^{-1/d}\big)\,\big)$. By connecting the approximation property to VC-dimension,  we prove that such a rate is optimal for H\"older continuous functions on $[0,1]^d$ in terms of the width and depth separately, \black{and hence this rate is also optimal for the whole continuous function class}. 
We also extend our analysis to general continuous functions on any bounded subset of $\R^d$. We would like to remark that our analysis was based on the fully connected feed-forward neural networks and the ReLU activation function. It would be very interesting to extend our conclusions to neural networks with other types of architectures (e.g., convolutional neural networks) and activation functions (e.g., tanh and sigmoid functions). %Another important direction is to extend our results to the $L^\infty$-norm on the whole domain without the trifling region. 
%Besides, if identity maps are allowed in the construction of neural networks as in the residual networks \cite{7780459}, the size of networks in our construction can be further optimized. Finally, the proposed analysis could be generalized to other function spaces with explicit formulas to characterize the approximation error. These will be left as future work. 

\section*{Acknowledgments}
Z.~Shen is supported by Tan Chin Tuan Centennial Professorship.  H.~Yang was partially supported by the US National Science Foundation under award DMS-1945029. S.~Zhang is supported by a Postdoctoral Fellowship under NUS ENDOWMENT FUND (EXP WBS) (01 651).
%%%%%%%%%%%%%%%%%%%%%%%%%%%%%%%%%%%%%%%%%%%%%%%%%%%%%%%%%%%%%%%
%\clearpage
\bibliographystyle{siam}
\bibliography{references}
%%%%%%%%%%%%%%%%%%%%%%%%%%%%%%%%%%%
\end{document}

%% file: preamble.tex
\usepackage{mathrsfs,amsmath,amsfonts,amssymb,comment}
\usepackage{mathtools,xcolor,enumerate}
\usepackage{dsfont}
\usepackage{bm}% MnSymbol after bm
\usepackage{MnSymbol}
\usepackage{indentfirst}
\usepackage[compress,sort]{cite}
\usepackage[bookmarks,colorlinks]{hyperref}
\hypersetup{citecolor=cyan,linkcolor=cyan}
\usepackage{tabularx,multirow,array,booktabs}
\usepackage{colortbl}%%% color in tables
\usepackage{graphicx}%% 
\graphicspath{{figures/}}%
\usepackage{float}%% 
\usepackage{subcaption} 
\makeatletter\@addtoreset{equation}{section}\makeatother
\usepackage{amsthm}
\theoremstyle{plain}
\newtheorem{theorem}{Theorem}[section]%  
\newtheorem{corollary}[theorem]{Corollary} %
\newtheorem{lemma}[theorem]{Lemma}
\newtheorem*{lemma*}{Lemma}

\newtheorem{proposition}[theorem]{Proposition}
\theoremstyle{definition}
\newtheorem{definition}[theorem]{Definition}
\theoremstyle{remark}

\newcommand{\R}{\ifmmode\mathbb{R}\else$\mathbb{R}$\fi}
\newcommand{\N}{\ifmmode\mathbb{N}\else$\mathbb{N}$\fi}
\newcommand{\Z}{\ifmmode\mathbb{Z}\else$\mathbb{Z}$\fi}
\newcommand{\Q}{\ifmmode\mathbb{Q}\else$\mathbb{Q}$\fi}
\newcommand{\tn}[1]{\textnormal{#1}}
\newcommand{\NN}{\mathcal{N\hspace{-2.5pt}N}}
\let\NNF\NN

\newcommand{\NNspace}{;\ }
\newcommand{\NNinput}{\tn{\#input}}
\newcommand{\NNwidthvec}{\tn{width\hspace{0.5pt}vec}}

\newcommand{\NNdepth}{\tn{depth}}

\newcommand{\NNparameter}{\tn{\#parameter}}
\newcommand{\NNwidth}{\tn{width}}

\newcommand{\NNoutput}{\tn{\#output}}
\newcommand{\bmx}{{\bm{x}}}

\newcommand{\bmbeta}{{\bm{\beta}}}

\newcommand{\bmA}{\bm{A}}

\newcommand{\bmPhi}{{\bm{\Phi}}}
\newcommand{\bmzero}{{\bm{0}}}
\newcommand{\calO}{{\mathcal{O}}}

\newcommand{\calH}{{\mathcal{H}}}

\newcommand{\calE}{\mathcal{E}}

\newcommand{\calA}{\mathcal{A}}
\newcommand{\calL}{\mathcal{L}}
\newcommand{\calT}{\mathcal{T}}
\newcommand{\tildephi}{{\widetilde{\phi}}}
\newcommand{\tildeg}{{\widetilde{g}}}

\newcommand{\tildeL}{{\widetilde{L}}}
\newcommand{\tildeN}{{\widetilde{N}}}
\newcommand{\tildef}{{\widetilde{f}}}
\newcommand{\tildeM}{{\widetilde{M}}}
\newcommand{\tildepsi}{{\widetilde{\psi}}}

\newcommand{\scrE}{{\mathscr{E}}}

\newcommand{\scrF}{{\mathscr{F}}}
\newcommand{\scrB}{{\mathscr{B}}}
\newcommand{\hatL}{{\widehat{L}}}
\newcommand{\Lip}{\tn{H\"older}}
\newcommand{\holder}[2]{{ \tn{H\"older}([0,1]^d,#2,#1) }}

\newcommand{\bin}{\tn{bin}\hspace{1.2pt}}
\newcommand{\vcd}{\tn{VCDim}}
\newcommand{\cpl}{\tn{CPwL}}

\newcommand{\mystep}[2]{\par \vspace{0.25cm}\noindent\textbf{\hspace{8pt}Step }$#1\colon$ #2 \vspace{0.18cm} \par }

\usepackage{tikz}
\newcommand{\myto}[2][1]{\mathop{
		\vcenter{\hbox{\scalebox{1}[#1]{\tikz{\draw[->,line width=0.72pt] (0,0.5) to (0.69*#2,0.5);}}}}
}}

\def\one{{\ensuremath{\mathds{1}}}}
\newcommand{\setMathResizeRate}[1]{\def\mathResizeRate{#1}}
\setMathResizeRate{0.8}
\newcommand{\mathResize}[2][\mathResizeRate]{
	\scalebox{#1}[#1]{\(\displaystyle #2\)}
}

\newenvironment{keywords}{\par \noindent\textbf{Key words}.}{\par}
\DeclareMathOperator*{\argmin}{arg\,min}

\usepackage{makecell}
\usepackage{amsopn}

\usepackage{url}
\newcommand*{\email}[1]{\href{mailto:#1}{\nolinkurl{#1}}}
\usepackage[left,mathlines]{lineno}
\usepackage{refcount}
\definecolor{mylinenumbercolor}{HTML}{BEBEBE}

\makeatletter
\newcommand*\patchAmsMathEnvironmentForLineno[1]{%
	\expandafter\let\csname old#1\expandafter\endcsname\csname #1\endcsname
	\expandafter\let\csname oldend#1\expandafter\endcsname\csname end#1\endcsname
	\renewenvironment{#1}%
	{\linenomath\csname old#1\endcsname}%
	{\csname oldend#1\endcsname\endlinenomath}}% 
\newcommand*\patchBothAmsMathEnvironmentsForLineno[1]{%
	\patchAmsMathEnvironmentForLineno{#1}%
	\patchAmsMathEnvironmentForLineno{#1*}}%
\patchBothAmsMathEnvironmentsForLineno{equation}%
\patchBothAmsMathEnvironmentsForLineno{align}%
\patchBothAmsMathEnvironmentsForLineno{flalign}%
\patchBothAmsMathEnvironmentsForLineno{alignat}%
\patchBothAmsMathEnvironmentsForLineno{gather}%
\patchBothAmsMathEnvironmentsForLineno{multline}%
\makeatother
\title{Optimal Approximation Rate of ReLU Networks in  terms of Width and Depth\thanks{Submitted to the editors DATE.}}
\author{Zuowei Shen\thanks{Department of Mathematics,  National University of Singapore
(\email{matzuows@nus.edu.sg}).}
  \and Haizhao Yang\thanks{Department of Mathematics,  Purdue University
  (\email{haizhao@purdue.edu}).}
\and Shijun Zhang\thanks{Department of Mathematics,  National University of Singapore
  (\email{zhangshijun@u.nus.edu}).}}
  \date{}